%% file: icml_main.tex
\documentclass[nohyperref]{article}

\usepackage{txie_icml}
\usepackage{natbib}

\usepackage{microtype}
\usepackage{graphicx}
\usepackage{subcaption}
\usepackage{booktabs} 

\usepackage{hyperref}

\usepackage[accepted]{icml2022}
\usepackage[scaled=.91]{helvet}

\usepackage{amsmath}
\usepackage{amssymb}
\usepackage{mathtools}
\usepackage{amsthm}
\usepackage[inline]{enumitem}
\usepackage{adjustbox}

\usepackage[capitalize,noabbrev]{cleveref}

\theoremstyle{plain}
\newtheorem{theorem}{Theorem}
\newtheorem{proposition}[theorem]{Proposition}
\newtheorem{lemma}[theorem]{Lemma}

\theoremstyle{definition}
\newtheorem{definition}[theorem]{Definition}
\newtheorem{assumption}[theorem]{Assumption}
\theoremstyle{remark}

\newcommand{\mypar}[1]{\noindent \textbf{#1} \hspace{1mm}}

\usepackage[textsize=tiny]{todonotes}

\usepackage{xspace}
\def\algo{ATAC\xspace}
\def\algofull{Adversarially Trained Actor Critic\xspace}
\newcommand{\tocite}[1]{\textcolor{red}{CITE\xspace#1}\xspace}

\icmltitlerunning{Adversarially Trained Actor Critic for Offline Reinforcement Learning}

\begin{document}

\twocolumn[

\icmltitle{Adversarially Trained Actor Critic for Offline Reinforcement Learning} 



\icmlsetsymbol{equal}{*}

\begin{icmlauthorlist}
\icmlauthor{Ching-An Cheng}{equal,MSR}
\icmlauthor{Tengyang Xie}{equal,UIUC}
\icmlauthor{Nan Jiang}{UIUC}
\icmlauthor{Alekh Agarwal}{Google}
\end{icmlauthorlist}

\icmlaffiliation{MSR}{Microsoft Research}
\icmlaffiliation{UIUC}{University of Illinois at Urbana-Champaign}
\icmlaffiliation{Google}{Google Research}

\icmlcorrespondingauthor{Ching-An Cheng}{chinganc@microsoft.com}

\icmlkeywords{Machine Learning, ICML}

\vskip 0.3in
]



\printAffiliationsAndNotice{\icmlEqualContribution} 

\begin{abstract}

\addtolength{\abovedisplayskip}{-1cm}
\addtolength{\belowdisplayskip}{-1cm}


%
We propose \algofull (\algo), a new model-free algorithm for offline reinforcement learning (RL) under insufficient data coverage, based on the concept of {\em relative pessimism}.
\algo is designed as a two-player Stackelberg game: A policy actor competes against an adversarially trained value critic, who finds data-consistent scenarios where the actor is inferior to the data-collection behavior policy.
We prove that, when the actor attains no regret in the two-player game, running \algo produces a policy that provably \emph{1)} outperforms the behavior policy over a wide range of hyperparameters that control the degree of pessimism, and \emph{2)} competes with the best policy covered by data with appropriately chosen hyperparameters.
Compared with existing works, notably our framework offers both theoretical guarantees for general 
function approximation and a  deep RL implementation scalable to complex environments and large datasets. 
In the D4RL benchmark,   \algo consistently outperforms state-of-the-art offline RL algorithms on a range of continuous control tasks.

\end{abstract}

\section{Introduction}
\label{sec:introduction}

Online reinforcement learning (RL) has been successfully applied in many simulation domains~\citep{mnih2015human, silver2016mastering}, demonstrating the promise of solving sequential decision making problems by direct exploratory interactions. However, collecting diverse interaction data is prohibitively expensive or infeasible in many real-world applications such as robotics, healthcare, and conversational agents. Due to these problems' risk-sensitive nature, data can only be collected by behavior policies that satisfy certain baseline performance or safety requirements.

The restriction on real-world data collection calls for \emph{offline} RL algorithms that can reliably learn with historical experiences that potentially have limited coverage over the state-action space.
Ideally, an offline RL algorithm should 
\begin{enumerate*}[label=\emph{\arabic*)}]
    \item always improve upon the behavior policies that collected the data, and
    \item learn from large datasets to outperform any other policy whose state-action distribution is well covered by the data.
\end{enumerate*}
The first condition is known as {safe policy improvement}~\citep{fujimoto2019off,laroche2019safe}, 
and the second is a form of learning consistency, that the algorithm makes the best use of the available data.

In particular, it is desirable that the algorithm can maintain safe policy improvement across \emph{large} and \emph{anchored} hyperparameter choices, a property we call \emph{robust policy improvement}.
Since offline hyperparameter selection is a difficult open question~\citep{paine2020hyperparameter, zhang2021towards}, robust policy improvement ensures the learned policy is always no worse than the baseline behavior policies and therefore can be reliably deployed in risk-sensitive decision making applications. For example, in healthcare, it is only ethical to deploy new treatment policies when we confidently know they are no worse than existing ones.
In addition, robust policy improvement makes tuning hyperparameters using additional online interactions possible. While online interactions are expensive, they are not completely prohibited in many application scenarios, especially when the tested policies are no worse than the behavior policy that collected the data in the first place. Therefore, if the algorithm has robust policy improvement, then its performance can potentially be more directly tuned. 

\begin{figure*}[t]
	\centering
	\begin{subfigure}{0.24\textwidth}
		\includegraphics[width=\textwidth]{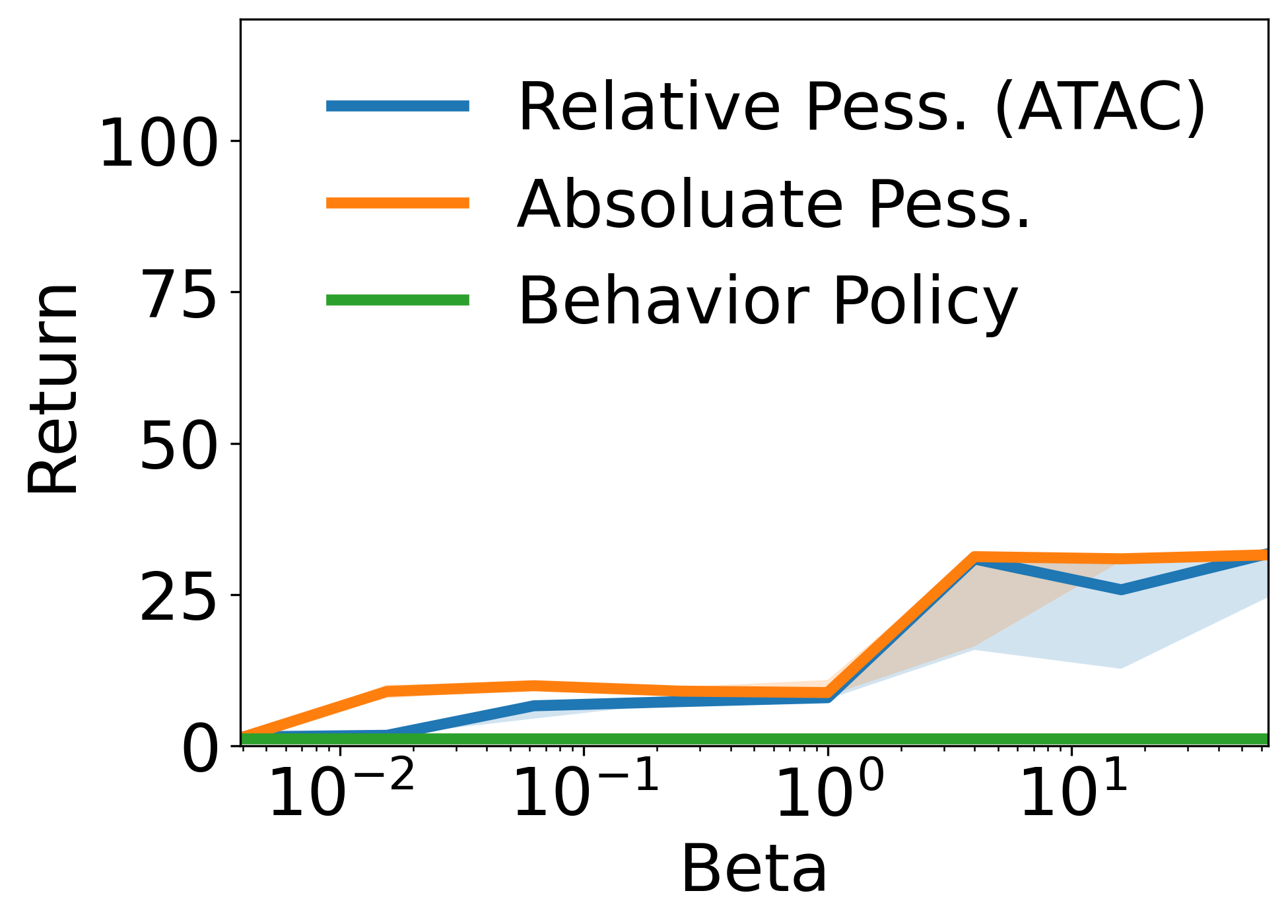}
		\caption{hopper-random}
		\label{fig:hopper-random robust PI}
	\end{subfigure}
	\begin{subfigure}{0.24\textwidth}
		\includegraphics[width=\textwidth]{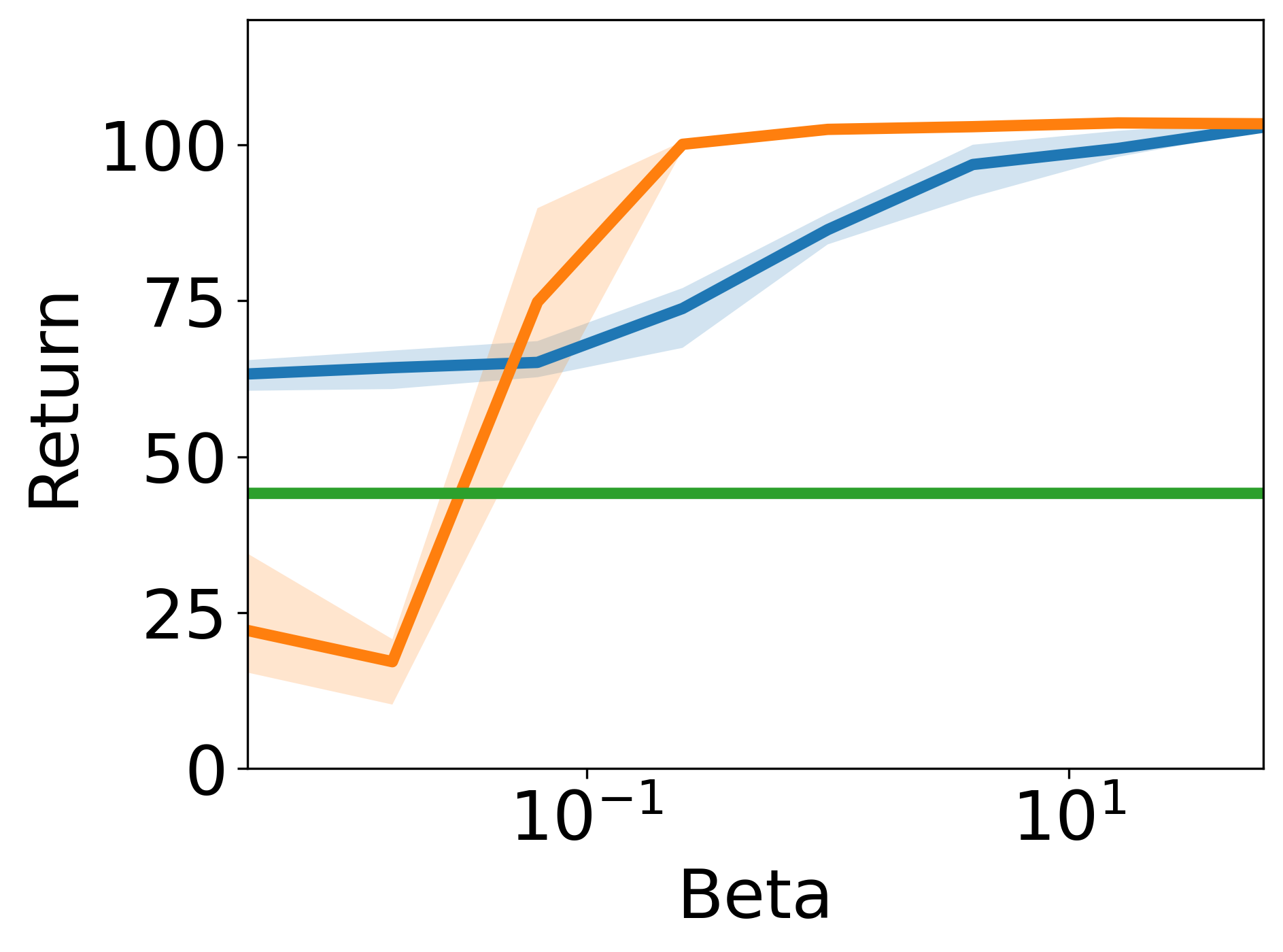}
		\caption{hopper-medium}
		\label{fig:hopper-medium robust PI}
	\end{subfigure}
	\begin{subfigure}{0.24\textwidth}
		\includegraphics[width=\textwidth]{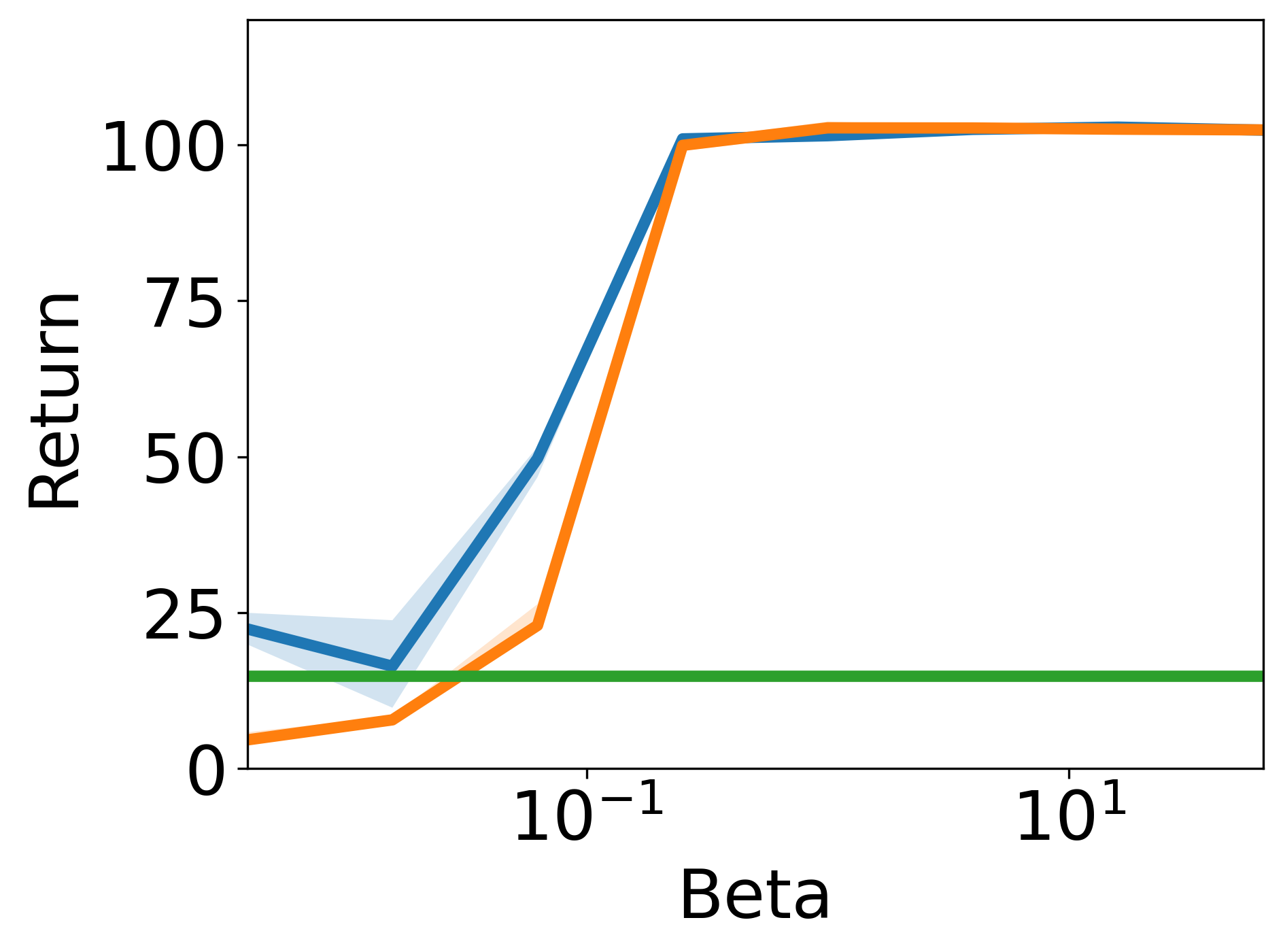}
		\caption{hopper-medium-replay}
		\label{fig:hopper-medium-replay robust PI}
	\end{subfigure}
	\begin{subfigure}{0.24\textwidth}
		\includegraphics[width=\textwidth]{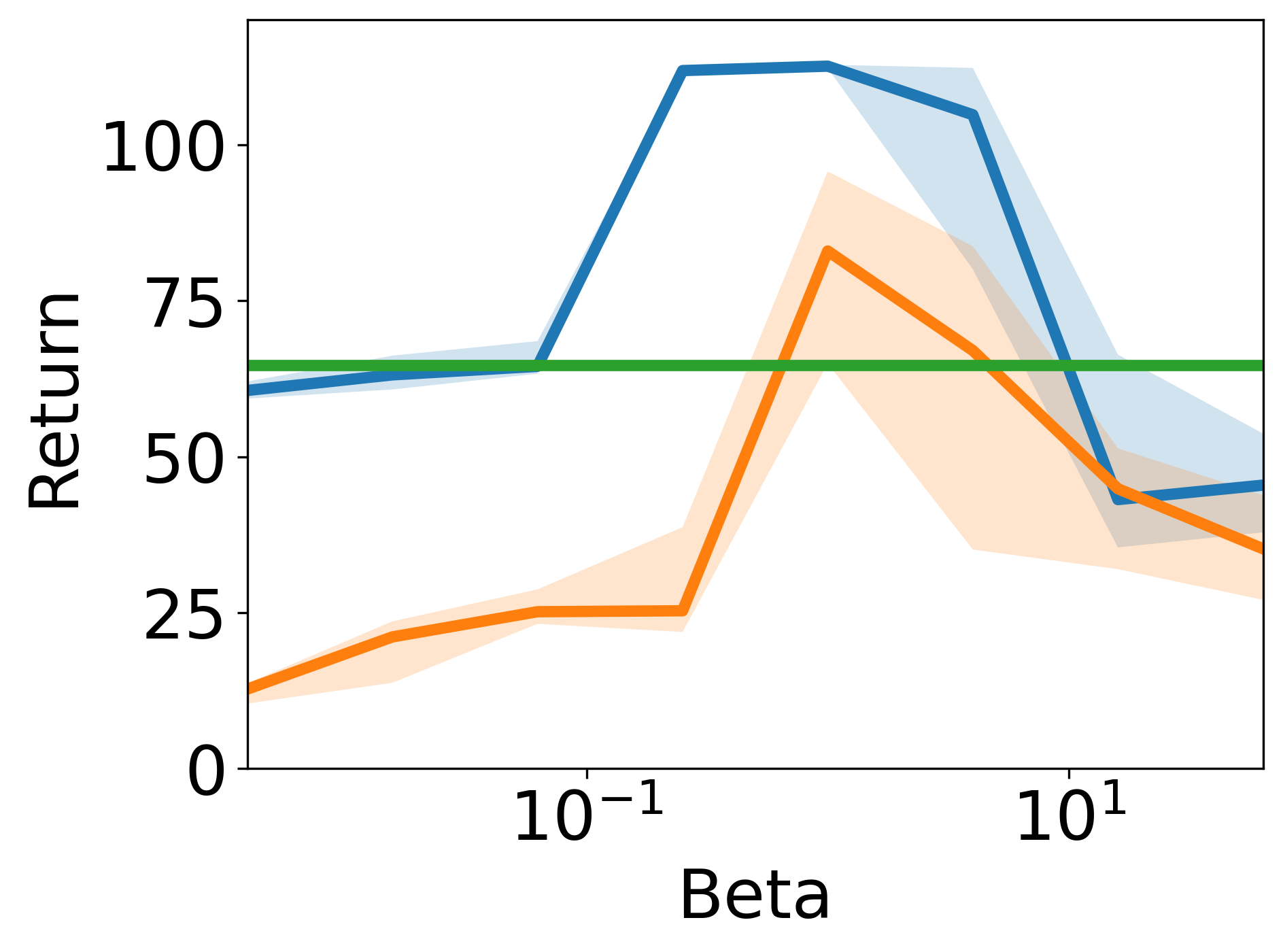}
		\caption{hopper-medium-expert}
		\label{fig:hopper-medium-expert robust PI}
	\end{subfigure}

\caption{\small{Robust Policy Improvement. \algo based on relative pessimism improves from behavior policies over a wide range of hyperparameters ($\beta$) that controls the degree of pessimism, and has a known safe policy improvement anchor point at $\beta=0$. Thus, we can gradually increase $\beta$ from zero to online tune \algo, while not violating the performance baseline of the behavior policy.
By contrast, offline RL based on absolute pessimism~\citep[e.g.,][]{xie2021bellman} has safe policy improvement only for well-tuned hyperparameters.
The differences are most stark in panel (d) where \algo outperforms behavior for $\beta$ ranging over $3$ orders of magnitude (0.01 to 10), compared with the narrow band of choices for absolute pessimism.
The plots show the $25^{th}$, $50^{th}$, $75^{th}$ percentiles over 10 random seeds.
}}
\label{fig:robust PI}
\vspace{-2mm}
\end{figure*}


However, few existing works 
possess all the desiderata above. Regarding consistency guarantees,
deep offline RL algorithms~\citep[e.g.][]{kumar2020conservative,kostrikov2021offline}  show strong empirical performance, but are analyzed theoretically in highly simplified tabular cases. 
Theoretical works~\citep{liu2020provably,jin2020pessimism,xie2021bellman,uehara2021representation} provide systematic analyses of learning correctness and consistency, but most of them have little empirical evaluation~\citep{liu2020provably} or consider only the linear case~\citep{jin2020pessimism,zanette2021provable}.

Turning to the robust policy improvement property, this is relatively rare in state-of-the-art offline RL literature.
Behavior regularization approaches~\citep{fujimoto2019off,kumar2019stabilizing,wu2019behavior,laroche2019safe,fujimoto2021minimalist} are scalable and show robustness for a broad range of hyperparameters that controls the pessimism degree; however, they are often more conservative, which ultimately limits the policy performance, as their robustness is achieved by a proximity regularization/constraint that ignores the reward information.
Some pessimistic algorithms~\citep{liu2020provably,xie2021bellman} have safe policy improvement guarantees but only for carefully selected hyperparameters. 
For a more detailed discussion of related works, see \cref{sec:related}.


In this paper, we propose a new model-free offline RL algorithm, \algofull (\algo). Compared with existing works, \algo
\begin{enumerate*}[label=\emph{\arabic*)}]
\item enjoys strong theoretical guarantees on robust policy improvement over hyperparameter that controls the pessimism degree
and learning consistency for
nonlinear function approximators,
and \item has a scalable implementation that can learn with deep neural networks and large datasets.
\end{enumerate*}

\algo is designed based on the concept of relative pessimism, leading to  a two-player Stackelberg game formulation of offline RL.
We treat the actor policy as the leader that aims to perform well under a follower critic, and adversarially train the critic to find Bellman-consistent~\citep{xie2021bellman} scenarios where the actor is inferior to the 
behavior policy.
Under standard function-approximation assumptions, we prove that, when the actor attains no regret in the two-player game, \algo produces a policy that provably outperforms the behavior policies for a large anchored range of hyperparameter choices and is optimal when the offline data covers scenarios visited by an optimal policy.
%

We also provide a practical implementation of \algo based on stochastic first-order two-timescale optimization.
In particular, we propose a new Bellman error surrogate, called double Q residual algorithm (DQRA) loss, which is inspired by a related idea of~\citet{wang2021convergent} and combines the double Q heuristic~\citep{fujimoto2018addressing} and the residual algorithm~\citep{baird1995residual} to improve the optimization stability of offline RL. 
We test \algo on the D4RL 
benchmark~\citep{fu2020d4rl}, and 
\algo consistently outperforms state-of-the-art baselines 
across multiple continuous-control problems. These empirical results also validate the robust policy improvement property of \algo (Fig.~\ref{fig:robust PI}), which makes \algo suitable for risk sensitive applications.
The code is available at \url{https://github.com/microsoft/ATAC}.

\section{Preliminaries}

\mypar{Markov Decision Process}
We consider RL in a Markov Decision Process (MDP) $\Mcal$, defined by $(\Scal, \Acal, \Pcal, R, \gamma)$. $\Scal$ is the state space, and $\Acal$ is the action space. $\Pcal: \Scal \times \Acal \to \Delta(\Scal)$ is the transition function, where $\Delta(\cdot)$ denotes the probability simplex, $R: \Scal \times \Acal \to [0, \Rmax]$ is the reward function, and $\gamma \in [0,1)$ is the discount factor. Without loss of generality, we assume that the initial state of the MDP, $s_0$, is deterministic.
We use $\pi: \Scal \to \Delta(\Acal)$ to denote the learner's decision-making policy, and $J(\pi) \coloneqq \E[\sum_{t = 0}^{\infty} \gamma^t r_t| a_t \sim \pi(\cdot|s_t)]$ to denote the expected discounted return of $\pi$, with $r_t = R(s_t, a_t)$. The goal of RL is to find a policy that maximizes  $J(\cdot)$. 
For any policy $\pi$, we define  the $Q$-value function as $Q^\pi(s,a) \coloneqq \E[\sum_{t = 0}^{\infty} \gamma^t r_t| (s_0, a_0) = (s,a), a_t \sim \pi(\cdot|s_t)]$. By the boundedness of rewards, we have $0 \leq Q^\pi \leq \tfrac{\Rmax}{1-\gamma} =: \Vmax$. For a policy $\pi$, the Bellman operator $\Tcal^\pi$  
is defined as
$\left(\Tcal^\pi f\right) (s,a) \coloneqq R(s,a) + \gamma \E_{s'|s,a} \left[ f(s', \pi)\right]$,
where $f(s', \pi) \coloneqq \sum_a \pi(a'|s') f(s',a')$.
In addition, we use $d^\pi$ to denote the normalized and discounted state-action occupancy measure of the policy $\pi$. That is, $d^\pi(s,a) \coloneqq (1 - \gamma) \E[\sum_{t = 0}^{\infty} \gamma^t \1(s_t = s, a_t = a)| a_t \sim \pi(\cdot|s_t)]$. We also use $\E_\pi$ to denote expectations with respect to $d^\pi$.

\mypar{Offline RL}
The goal of offline RL is to compute good policies based pre-collected offline data without environment interaction.
We assume the offline data $\Dcal$ consists of $N$ i.i.d.~$(s,a,r,s')$ tuples, where $(s,a) \sim \mu$, $r = R(s,a)$, $s' \sim \Pcal(\cdot|s,a)$. We also assume 
$\mu$ is the discounted state-action occupancy of some \emph{behavior policy}, which we also denote as $\mu$ with abuse of notation (i.e., $\mu = d^\mu$). We will use $a\sim\mu(\cdot | s)$ to denote actions drawn from this policy, and also $(s,a,s')\sim \mu$ to denote $(s,a)\sim \mu$ and $s'\sim P(\cdot | s,a)$.

\mypar{Function Approximation} 
We assume access to a value-function class $\Fcal \subseteq (\Scal \times \Acal \to [0, \Vmax])$ to model the $Q$-functions of policies, and we search for good policies from a policy class $\Pi \subseteq (\Scal \to \Delta(\Acal))$. The combination of $\Fcal$ and $\Pi$ is commonly used in the literature of actor-critic or policy-based approaches~\citep[see, e.g.,][]{bertsekas1995neuro, konda2000actor,haarnoja2018soft}. We now recall some standard assumptions on the expressivity of the value function class $\Fcal$ which are needed for actor-critic methods, particularly in an offline setting.

\begin{assumption}[Approximate Realizability]
\label{asm:relz2}
For any policy $\pi\in\Pi$, $\min_{f \in \Fcal}\max_{\text{admissible }\nu} \left\|f - \Tcal^{\pi} f\right\|_{2,\nu}^2 \leq \varepsilon_\Fcal$,
\label{ass:realizable}
where admissibilty $\nu$ means $ \nu \in \{ d^{\pi'} : \forall \pi \in \Pi\}$.
\end{assumption}
Assumption~\ref{ass:realizable} is a weaker form of stating $Q^\pi\in\Fcal$, $\forall \pi'\in\Pi$.
This realizability assumption is the same as the one made by~\citet{xie2021bellman} and
is weaker than assuming a small error in $\ell_\infty$ norm~\citep{antos2008learning}.

\begin{assumption}[Approximate Completeness]
\label{asm:completeness}
For any $\pi\in\Pi$ and $f\in\Fcal$, we have $\min_{g\in\Fcal}\|g - \Tcal^\pi f\|_{2,\mu}^2 \leq \varepsilon_{\Fcal,\Fcal}$.
\label{ass:complete}
\end{assumption}
Here $\| \cdot \|_{2,\mu} \coloneqq \sqrt{\E_\mu[(\cdot)^2]}$ denotes the $\mu$-weighted 2-norm.\footnote{We will use the notation $\|f\|_{2, \Dcal}$ for an empirical distribution $d$ of a dataset $\Dcal$, where $\| f \|_{2,\Dcal} = \sqrt{\frac{1}{|\Dcal|} \sum_{(s,a) \in \Dcal} f(s,a)^2}$.}
Again Assumption~\ref{ass:complete} weakens the typical completeness assumption of $\Tcal^\pi f \in \Fcal$ for all $f\in\Fcal$, which is commonly assumed in the analyses of policy optimization methods with TD-style value function learning. We require the approximation to be good \emph{only} on the data distribution.

\section{A Game Theoretic Formulation of Offline RL with Robust Policy Improvement} \label{sec:game}


In this section, we introduce the idea of relative pessimism and propose a new Stackelberg game~\citep{von2010market} formulation of offline RL, which is the foundation of our algorithm \algo.
For clarity, in this section we 
discuss solution concepts at the population level instead of using samples.
This simplification is for highlighting the uncertainty in decision making due to missing coverage in the data distribution $\mu$ that offline RL faces. We will consider the effects of finite sample approximation when we introduce \algo in \cref{sec:ATAC}.

\subsection{A Stackelberg Game Formulation of Offline RL}

\mypar{Stackelberg game} A Stackelberg game is a sequential game between a leader and a follower. It can be stated as a bilevel optimization problem,
    $\min_x g(x,y_x),  \textrm{ s.t. } y_x \in \argmin_y h(x,y)$
where the leader and the follower are the variables $x$ and $y$, respectively, and $g,h$ are their objectives.
The concept of Stackelberg game has its origins in the economics literature and has been recently applied to design \emph{online} model-based RL~\citep{rajeswaran2020game} and \emph{online} actor critic algorithms~\citep{zheng2021stackelberg}. The use of this formalism in an offline setting here is novel to our knowledge. Stackelberg games also generalize previous minimax formulations~\citep{xie2021bellman}, which correspond to a two-player zero-sum game with $h = -g$.

\addtolength{\abovedisplayskip}{-0.2cm}
\addtolength{\belowdisplayskip}{-0.2cm}

\mypar{Offline RL as a Stackelberg game} Inspired by the minimax offline RL concept by~\citet{xie2021bellman} and the pessimistic policy evaluation procedure by~\citet{kumar2020conservative}, we formulate the Stackelberg game for offline RL as a bilevel optimization problem, with the learner policy $\pi\in\Pi$ as the leader and a critic $f\in\Fcal$ as the follower: \vspace*{.1cm}
\begin{align} \label{eq:game_formuation}
  \widehat{\pi}^*  & \in\argmax_{\pi\in\Pi}  \Lcal_\mu(\pi, f^\pi) \\
 \textstyle    \textrm{s.t.} \quad f^\pi  & \in\argmin_{f\in\Fcal}  \Lcal_\mu(\pi, f) + \beta \Ecal_\mu(\pi, f) \nonumber
\end{align}
where $\beta\geq 0 $ is a hyperparamter, and 
\begin{align} 
     \Lcal_\mu(\pi, f) &\coloneqq \E_{\mu}[ f(s,\pi) - f(s,a) ]  \label{eq:difference}  \\
     \Ecal_\mu(\pi, f) &\coloneqq  \E_{\mu}[ ((f - \Tcal^\pi f) (s,a))^2 ]. \label{eq:bellman error}  
\end{align}
Intuitively, $\widehat\pi^*$ attempts to maximize the value predicted by $f^\pi$, and
$f^\pi$ performs a \emph{relatively} pessimistic policy evaluation of a candidate $\pi$ with respect to the behavior policy $\mu$ (we will show $\Lcal_\mu(\pi, f)$ aims to estimate $(1-\gamma)(J(\pi)- J(\mu))$). In the definition of $f^\pi$, $\Ecal_\mu(\pi, f)$ ensures $f^\pi$'s (approximate) Bellman-consistency on data and $\Lcal_\mu(\pi, f)$ promotes pessimism with $\beta$ being the hyperparameter that controls their relative contributions.
In the rest of this section, we discuss how the relative pessimistic formulation in \Eqref{eq:game_formuation} leads to the desired property of \textit{robust} policy improvement, and compare it to the solution concepts used in the previous offline RL works. 

\subsection{Relative Pessimism and Robust Policy Improvement}
\label{sec:safe_PI}

Our design of the Stackelberg game in \Eqref{eq:game_formuation} is motivated by the benefits of robust policy improvement in  $\beta$ given by relative pessimism. 
As discussed in the introduction, such property is particularly valuable to applying offline RL in risk-sensitive applications, because it guarantees the learned policy is no worse than the behavior policy regardless of the hyperparameter choice and allows potentially direct online performance tuning.
Note that prior works~\citep[e.g.,][]{liu2020provably,xie2021bellman} have relied on well-chosen hyperparameters to show improvement upon the behavior policy (i.e., safe policy improvement). We adopt the term \emph{robust} policy improvement here to distinguish from those weaker guarantees. While there are works~\citep{laroche2019safe,fujimoto2019off,kumar2019stabilizing} that provide robust policy improvement in tabular problems, but their heuristic extensions to function approximation lose this guarantee.

Below we show that the solution $\widehat{\pi}^*$ in \Eqref{eq:game_formuation} is no worse than the behavior policy for \emph{any} $\beta\geq0$ under Assumption~\ref{ass:realizable}. This property is because in \Eqref{eq:game_formuation} the actor is trying to optimize a lower bound on the relative performance $(1-\gamma) (J(\pi)-J(\mu))$ for $\pi$, and this lower bound is tight (exactly zero) at the behavior policy $\mu$, for any $\beta\geq0$.

\begin{proposition}\label{prop:RPI}
If Assumption~\ref{ass:realizable} holds with $\varepsilon_{\Fcal} = 0$ and $\mu\in\Pi$, then $\Lcal_\mu(\pi, f^\pi) \leq (1-\gamma) (J(\pi) - J(\mu))$ $\forall\pi\in\Pi$, for \emph{any} $\beta\geq 0$. This implies $J(\widehat{\pi}^*) \geq J(\mu)$.
\end{proposition}
\vspace{-4mm}
\begin{proof}
By performance difference lemma~\citep{kakade2002approximately}, $J(\pi) - J(\mu) = \frac{1}{1-\gamma}\E_{\mu}[ Q^\pi(s,\pi) - Q^\pi(s,a) ] $. Therefore, if $Q^\pi \in \Fcal$ on states of $\mu$, then $(1-\gamma)(J(\pi) - J(\mu))= \Lcal_\mu(\pi, Q^\pi) = \Lcal_\mu(\pi, Q^\pi) + \beta\Ecal(Q^\pi,\pi) \geq \Lcal_\mu(\pi, f^\pi) + \beta\Ecal(f^\pi,\pi) \geq \Lcal_\mu(\pi, f^\pi)$, where we use $\Ecal(\pi,Q^\pi) = 0$ by definition of $Q^\pi$ and $\Ecal(\pi,f) \geq 0$ for any $f\in\Fcal$.
Robust policy improvement follows, as $J(\widehat{\pi}^*) - J(\mu) \geq \Lcal_\mu(\widehat{\pi}^*, f^{\widehat{\pi}^*})\geq \Lcal_\mu(\mu, f^\mu) = 0$.
\end{proof}
\vspace{-2mm}

\mypar{Relative vs.~absolute pessimism} Our formulation is inspired by the maximin objective of~\citet{xie2021bellman}, which optimizes a pessimistic estimate of $J(\pi)$ (which we call \textit{absolute pessimism}) and learns a good policy with a well-chosen value of $\beta$. In contrast, our relative pessimism formulation optimizes the performance of $\pi$ \textit{relative}  to $J(\mu)$, i.e., $J(\pi) - J(\mu)$. As \cref{sec:theoretical_algo} will show, algorithms based on both formulations enjoy similar optimality guarantees with well-chosen hyperparameters. But absolute pessimism gives policy improvement only for certain hyperparameters, while the relative approach enjoys robust policy improvement for all $\beta \geq 0$, which is practically significant.

\mypar{Improvement beyond behavior policy} It is clear from \cref{prop:RPI}  that the objective in~\Eqref{eq:game_formuation} results in the optimization of a lower bound on the value gap $(1-\gamma)(J(\pi) - J(\mu))$. On the other hand, for appropriate settings of $\beta$, it turns out that this lower bound is not too loose for any $\pi\in\Pi$ such that $d^\pi$ is covered by the data distribution $\mu$, as implicitly shown in our detailed theoretical analysis presented in the next section. Consequently, maximizing the objective~\Eqref{eq:game_formuation} generally results in policies that significantly outperform $\mu$ for appropriate choices of $\beta$, as long as the data has support for at least one such policy. 

\mypar{Imitation learning perspective}
An alternative interpretation of \cref{prop:RPI} follows from examining the special case of $\beta = 0$ (i.e. not using any information of rewards and transitions). In this case, the objective~\Eqref{eq:game_formuation} reduces to the maximin problem:
$
\max_{\pi\in\Pi} \min_{f\in\Fcal} \Lcal_\mu (\pi, f), $ 
%
which always yields robust policy improvement under Assumption~\ref{ass:realizable}. More generally, if the function class $\Fcal$ is rich enough to approximate all bounded, Lipschitz functions, then the above objective with $\beta=0$ resembles behavior cloning to match the occupancy measures of $\pi$ and $\mu$ using an integral probability metric~\citep[IPM;][]{muller1997integral}~\citep[or equivalently, Wasserstein GAN;][]{arjovsky2017wasserstein}. With $\beta>0$, the algorithm gets more information and thus intuitively can perform better.
This perspective shows how our formulation unifies the previously disparate literature on behavior regularization and pessimism.

\section{Adversarially Trained Actor Critic} \label{sec:ATAC}

We design our new model-free offline RL algorithm, \algofull (\algo), based on the Stackelberg game of relative pessimism in \cref{sec:game}.\footnote{One can also use the formulation to design model-based algorithms, by constructing $f$ as a $Q$-function $\widehat{Q}_\theta^\pi$  computed from a model parameterized by $\theta$, and using  $\Ecal_\mu(\pi, \widehat{Q}_\theta^\pi)$ to capture the reward and transition errors of the model $\theta$.}
In the following, we first describe a theoretical version of \algo (\cref{alg:atac (theory)}) in \cref{sec:theoretical_algo}, which is based on a no-regret policy optimization oracle and a pessimistic policy evaluation oracle. We discuss its working principles and give theoretical performance guarantees. This theoretical algorithm further serves as a template that provides design principles for implementing \algo.
To show its effectiveness, in \cref{sec:practical algo}, we design  \cref{alg:atac (practice)}, a practical deep-learning implementation of \algo.
\cref{alg:atac (practice)} is a two-timescale first-order algorithm based on stochastic approximation, and uses a novel Bellman error surrogate (called double-Q residual algorithm loss) for off-policy optimization stability.
Later in \cref{sec:experiments}, we empirically demonstrate that the principally designed \cref{alg:atac (practice)} outperforms many state-of-the-art offline RL algorithms.

\subsection{Theory of \algo with Optimization Oracles}
\label{sec:theoretical_algo}

This section instantiates a version of the \algo algorithm with abstract optimization oracles for the leader and follower, using the concepts introduced in \cref{sec:game}. We first define the empirical estimates of $\Lcal_\mu$ and $\Ecal_\mu$ as follows.
Given a dataset $\Dcal$, we define
\begin{align}
\label{eq:pd_loss}
\Lcal_\Dcal(f, \pi) \coloneqq \E_\Dcal \left[ f(s,\pi) - f(s,a) \right],
\end{align}
and the estimated Bellman error~\citep{antos2008learning}
\begin{equation}
\label{eq:defmsbope}
\begin{aligned}
&~ \Ecal_\Dcal(f,\pi) \coloneqq \E_\Dcal \left[ \left(f(s,a) - r - \gamma f(s',\pi) \right)^2 \right]
\\
&~ \qquad\quad - \min_{f' \in \Fcal} \E_\Dcal \left[ \left(f'(s,a) - r - \gamma f(s',\pi) \right)^2 \right].
\end{aligned}
\end{equation}

\subsubsection{Algorithm}
Using these definitions, \cref{alg:atac (theory)} instantiates a version of the \algo approach. At a high-level, the $k^{th}$ iteration of the algorithm first finds a critic $f_k$ that is maximally pessimistic for the current actor $\pi_{k}$ along with a regularization based on the estimated Bellman error of $\pi_k$ (line~\ref{lin:reg_pes_pe}), with a hyperparameter $\beta$ trading off the two terms. The actor $\pi_{k+1}$ then invokes a no-regret policy optimization oracle to update its policy, given $f_k$ (line~\ref{step:po}). We now discuss some of the key aspects of the algorithm.

\begin{algorithm}[t]
\caption{ATAC (Theoretical Version)}
\label{alg:atac (theory)}
{\bfseries Input:} Batch data $\Dcal$. coefficient $\beta$.
\begin{algorithmic}[1]
\State Initialize policy $\pi_1$ as the uniform policy.
\For{$k = 1,2,\dotsc,K$}
\State Obtain the pessimistic estimation of $\pi_k$ as $f_k$, 
\label{lin:reg_pes_pe}
\Statex\qquad$ f_k \leftarrow \argmin_{f \in \Fcal_k} \Lcal_\Dcal(f, \pi_k) + \beta \Ecal_\Dcal(f,\pi_k)$.
%
\State Compute $\pi_{k + 1}$ by 
\label{step:po}
\Statex \qquad
\qquad $\pi_{k + 1} \leftarrow {\sf PO}(\pi_k, f_k, \Dcal)$,
\Statex \hspace{4.5mm} where ${\sf PO}$ denotes a no-regret oracle (Def.~\ref{def:no-regret}).
\EndFor
\State Output $\pibar \coloneqq \unif(\pi_{[1:K]})$. \algocmt{uniformly mix $\pi_1, \ldots, \pi_K$ at the trajectory level}
\end{algorithmic}
\end{algorithm}

\mypar{Policy optimization with a no-regret oracle}
In \cref{alg:atac (theory)}, the policy optimization step (\cref{step:po}) is conducted by calling a no-regret policy optimization oracle ({\sf PO}). We now define the property we expect from this oracle.

\begin{definition}[No-regret policy optimization oracle] \label{def:no-regret}
An algorithm {\sf PO} is called a \emph{no-regret policy optimization oracle} if for any sequence of functions\footnote{$\{f_k\}_{k=1}^K$ can be generated by an adaptive adversary.} $f_1,\ldots,f_K$ with $f_k~:~\Scal\times\Acal\to[0,\Vmax]$, the policies $\pi_1,\ldots,\pi_K$ produced by {\sf PO} satisfy, for any comparator $\pi \in \Pi$:
\begin{align*}
\textstyle
\varepsilon_\opt^\pi \coloneqq \frac{1}{1 - \gamma}\sum_{k = 1}^{K} \E_{\pi} \left[ f_k(s, \pi) - f_k(s,\pi_k) \right] = o(K).
\end{align*}
\end{definition}

The notion of regret used in Definition~\ref{def:no-regret} nearly corresponds to the standard regret definition in online learning~\citep{cesa2006prediction}, except that we take an expectation over states as per the occupancy measure of the comparator. Algorithmically, a natural oracle might perform online learning with states and actions sampled from $\mu$ in the offline RL setting. This mismatch of measures between the optimization objective and regret definition is typical in policy optimization literature~\citep[see e.g.][]{kakade2002approximately,agarwal2021theory}. One scenario in which we indeed have such an oracle is when {\sf PO} corresponds to running a no-regret algorithm \emph{separately in each state}\footnote{The computational complexity of doing so does \textit{not} depend on the size of the state space, since we only need to run the algorithm on states observed in the data. See~\citep{xie2021bellman}.} and the policy class is sufficiently rich to approximate the resulting iterates. There is a rich literature on such approaches using mirror-descent style methods~\citep[e.g.,][]{neu2017unified,geist2019theory}, of which a particularly popular instance is soft policy iteration or natural policy gradient~\citep{kakade2001natural} based on multiplicative weight updates~\citep[e.g.][]{even2009online,agarwal2021theory}: $\pi_{k + 1}(a|s) \propto \pi_{k}(a|s) \exp \left( \eta f_k(s,a) \right)$ 
with $\eta = \sqrt{\frac{\log|\Acal|}{2 \Vmax^2 K}}$.
This oracle is used by~\citet{xie2021bellman}, which leads to the regret bound $\varepsilon_\opt^\picomp \leq \Ocal\left(\frac{\Vmax}{1 - \gamma}\sqrt{K \log|\Acal|}\right)$.

\subsubsection{Theoretical Guarantees}
We now provide the theoretical analysis of \cref{alg:atac (theory)}. Recall that with missing coverage, we can only hope to compete with policies whose distributions are well-covered by data, and we need a quantitative measurement of such coverage. 
Following~\citet{xie2021bellman}, we use $\Cscr(\nu;\mu,\Fcal,\pi)\coloneqq \max_{f \in \Fcal}\frac{\|f - \Tcal^\pi f\|_{2,\nu}^2}{\|f - \Tcal^\pi f\|_{2,\mu}^2}$ to measure how well a distribution of interest $\nu$ (e.g., $d^\pi$) is covered by the data distribution $\mu$ w.r.t.~policy $\pi$ and function class $\Fcal$, which is a sharper measure than the more typical concentrability coefficients~\citep{munos2008finite} (e.g., $\Cscr(\nu;\mu,\Fcal,\pi) \leq \max_{s,a} \nu(s,a) / \mu(s,a)$).

We also use $d_{\Fcal,\Pi}$ to denote the joint statistical complexity of the policy class $\Pi$ and $\Fcal$. For example, when $\Fcal$ and $\Pi$ are finite, we have $d_{\Fcal,\Pi} = \Ocal(\log\nicefrac{|\Fcal| |\Pi|}{\delta})$, where $\delta$ is a failure probability. Our formal proofs utilize the covering number to address infinite function classes; see  \cref{app:theo_proofs} for details. In addition, we also omit the approximation error terms $\varepsilon_\Fcal$ and $\varepsilon_{\Fcal,\Fcal}$ in the results presented in this section for the purpose of clarity. The detailed results incorporating these terms are provided in \cref{app:theo_proofs}.

\begin{theorem}[Informal]
\label{thm:maintext_thm}
Let $|\Dcal| = N$, $C \geq 1$ be any constant, $\nu\in\Delta(\Scal\times\Acal)$ be an arbitrarily distribution that satisfies $\max_{k \in [K]}\Cscr(\nu;\mu,\Fcal,\pi_k) \leq C$, and $\pi\in\Pi$ be an arbitrary competitor policy. Then, when $\varepsilon_\Fcal=\varepsilon_{\Fcal,\Fcal}=0$, choosing $\beta = {\Theta}\left(\sqrt[3]{\frac{\Vmax N^2}{d_{\Fcal,\Pi}^2}}\right)$, with high probability:
\begin{align*}
\textstyle
J(\pi) - J(\pibar) &\leq \textstyle~ \varepsilon_\opt^\pi + \Ocal\left( \frac{\Vmax \sqrt{C} (d_{\Fcal,\Pi})^{\nicefrac{1}{3}}}{(1 - \gamma) N^{\nicefrac{1}{3}}}\right)
\\
&\quad \textstyle~ + \frac{1}{K(1-\gamma)}\sum_{k = 1}^{K} {\left\langle d^{\pi}\setminus\nu, ~ f_k - \Tcal^{\pi_k} f_k \right\rangle},
\end{align*}
where $(d^{\pi}\setminus\nu)(s,a) \coloneqq \max(d^\pi(s,a) - \nu(s,a),0)$, and $\langle d, f \rangle \coloneqq \sum_{(s,a) \in \Scal \times \Acal}d(s,a) f(s,a)$ for any $d$ and $f$.
\end{theorem}
At a high-level, our result shows that we can compete with any policy $\pi$ using a sufficiently large dataset, as long as our optimization regret is small and the data distribution $\mu$ has a good coverage for $d^\pi$. In particular, choosing $\nu = d^\pi$ removes the off-support term, so that we always have a guarantee scaling with $\max_k \Cscr(d^\pi, \mu,\Fcal,\pi_k)$, but can benefit if other distributions $\nu$ are better covered with a small off-support mass $\|d^\pi\setminus\nu\|_1$. The off-support term can also be small if a small Bellman error under $\mu$ generalizes to a small error out of support, due to properties of $\Fcal$.

\mypar{Comparison with prior theoretical results} To compare our result with prior works, we focus on the two statistical error terms in our bound, ignoring the optimization regret. Relative to the information-theoretic bound of~\citet{xie2021bellman}, we observe a similar decomposition into a finite sample deviation term and an off-support bias. Their finite sample error decays as $N^{-1/2}$ as opposed to our $N^{-1/3}$ scaling, which arises from the use of regularization here. Indeed, we can get a $N^{-1/2}$ bound for a constrained version, but such a version is not friendly to practical implementation. Prior linear methods~\citep{jin2020pessimism,zanette2021provable,uehara2021representation} have roughly similar guarantees to~\citet{xie2021bellman}, 
so a similar comparison holds.

Most related to Theorem~\ref{thm:maintext_thm} is the $N^{-1/5}$ bound of~\citet[Corollary 5]{xie2021bellman} for their regularized algorithm PSPI, which is supposed to be computationally tractable though no practical implementation is offered.\footnote{Incidentally, we are able to use our empirical insights to provide a scalable implementation of PSPI; see Section~\ref{sec:experiments}.} While our bound is better, we use a bounded complexity $\Pi$ while their result uses an unrestricted policy class. 
If we were to use the same policy class as theirs, the complexity of $\Pi$ would grow with optimization iterates, requiring us to carefully balance the regret and deviation terms and yielding identical guarantees to theirs. 
To summarize, our result is comparable to~\citet[Corollary 5]{xie2021bellman} and stated in a more general form, and we enjoy a crucial advantage of robust policy improvement as detailed below.



\mypar{Robust policy improvement}
We now formalize the robust policy improvement of \cref{alg:atac (theory)}, which can be viewed as the finite-sample version of Proposition~\ref{prop:RPI}.
\begin{proposition}
\label{prop:safe_pi}
Let 
$\pibar$ be the output of \cref{alg:atac (theory)}.
If Assumption~\ref{ass:realizable} holds with $\varepsilon_{\Fcal} = 0$ and $\mu\in\Pi$, with high probability,
\begin{small}
\begin{align*}
 J(\mu) - J(\pibar) \leq \Ocal\left(\frac{\Vmax}{1-\gamma}\sqrt{\frac{ d_{\Fcal,\Pi}}{N}} + \frac{\beta \Vmax^2 d_{\Fcal,\Pi}}{(1-\gamma)N} \right) + \varepsilon_\opt^\mu.
\end{align*}
\end{small}
\end{proposition}
\cref{prop:safe_pi} provides the robust policy improvement guarantee in the finite-sample regime, under a \emph{weaker} assumption on $\Fcal$ than that in \cref{thm:maintext_thm}.
In contrast to the regular safe policy improvement results in offline RL~\citep[e.g.,][Corollary 3]{xie2021bellman} where the pessimistic hyperparamter is required to choose properly, the robust policy improvement from \cref{prop:safe_pi} could adapt to a wide range of $\beta$. As long as $\beta = o(N)$, the learned policy $\pibar$ from \cref{alg:atac (theory)} is guaranteed improve the behavior policy $\mu$ consistently.
In fact, for such a range of $\beta$, robust policy improvement holds regardless of the quality of the learned critic. For example, when $\beta=0$,
\cref{prop:safe_pi} still guarantees a policy no worse than the behavior policy $\mu$, though the critic loss does not contain the Bellman error term anymore. (In this case, \algo performs IL). In contrast, prior works  based on absolute pessimism~\citep[e.g.,][]{xie2021bellman} immediately output degenerate solutions when the Bellman error term is removed. 

It is also notable that, compared with \cref{thm:maintext_thm}, \cref{prop:safe_pi} enjoys a better statistical rate with a proper $\beta$, i.e., $\beta \leq O(N^{1/2})$, due to the decomposition of performance difference shown in the following proof sketch.

\mypar{Proof sketch}
\cref{thm:maintext_thm} is established based on the following decomposition of performance difference: $\forall \pi$, 
\begin{small}
\begin{align} \label{th:performance_difference_main}
    & (1-\gamma)(J(\pi) - J(\pi_k)) \leq
    \E_{\mu} \left[ f_k - \Tcal^{\pi_k} f_k \right] - \E_{\pi} \left[ f_k-\Tcal^{\pi_k} f_k  \right] \nonumber \\
     & +  \E_{\pi} \left[ f_k(s,\pi) - f_k(s,\pi_k) \right] + \widetilde\Ocal \Big( \sqrt{\frac{\Vmax^2}{N}} + \frac{\beta \Vmax^2}{N}\Big).
\end{align}
\end{small}\par
Details of this decomposition can be found in \cref{sec:decomp_pd}, and the proof relies on the fact that $f_k$ is obtained by our pessimistic policy evaluation procedure.
In \Eqref{th:performance_difference_main}, the first two terms are controlled by the Bellman error (both on-support and off-support), and the third is controlled by the optimization error.
Notably, when the comparator $\pi$ is the behavior policy $\mu$, the first two terms in \Eqref{th:performance_difference_main} cancel out, giving the faster rate of Proposition~\ref{prop:safe_pi}. This provides insight for why robust policy improvement does not depend on the quality of the learned critic.
%



\begin{algorithm}[t]
\caption{ATAC (Practical Version)}
\label{alg:atac (practice)}
{\bfseries Input:} Batch data $\Dcal$, policy $\pi$, critics $f_1, f_2$, constants $\beta\geq 0$, $\tau\in[0,1]$, $w\in[0,1]$
\begin{algorithmic}[1]
\State Initialize target networks $\bar{f}_1\gets f_1$, $\bar{f}_2\gets f_2$
\For{$k = 1,2,\dotsc,K$}
\State Sample minibatch $\Dcal_{\textrm{mini}}$ from dataset $\Dcal$.

\State For {$f\in\{f_1, f_2\}$}, update critic networks \label{ln:update critic}
\Statex\qquad$l_{\textrm{critic}}(f) \coloneqq
    \Lcal_{\Dcal_{\textrm{mini}}}(f,\pi) + \beta \Ecal_{\Dcal_{\textrm{mini}}}^{w}(f, \pi)$
\Statex \qquad$f \gets \text{Proj}_\Fcal(f - \eta_{\textrm{fast}}\nabla l_{\textrm{critic}})$

\State Update actor network  \label{ln:update actor}

\Statex \qquad$l_{\textrm{actor}}(\pi)
    \coloneqq - \Lcal_{\Dcal_{\textrm{mini}}}(f_1,\pi)$
\Statex \qquad$\pi \gets \text{Proj}_{\Pi}(\pi - \eta_{\textrm{slow}}\nabla l_{\textrm{actor}})$
\State For $(f,\bar{f})\in\{(f_i,\bar{f}_i)\}_{i=1,2}$, update target \label{ln:update target}
\Statex \qquad\qquad$\bar{f} \gets (1-\tau) \bar{f} + \tau{f}  $.
\EndFor
\end{algorithmic}
\end{algorithm}

\subsection{A Practical Implementation of \algo} \label{sec:practical algo}

We present a scalable deep RL version of \algo in  \cref{alg:atac (practice)}, following the principles of \cref{alg:atac (theory)}. With abuse of notation, we use $\nabla l_{\textrm{actor}}$, $\nabla l_{\textrm{critic}}$ to denote taking gradients with respect to the parameters of the actor and the critic, respectively; similarly \cref{ln:update target} in \cref{alg:atac (practice)} refers to a moving average in the parameter space. In addition, every term involving $\pi$ in \cref{alg:atac (practice)} means a stochastic approximation based on sampling an action from $\pi$ when queried. In implementation, we use adaptive gradient descent algorithm ADAM~\citep{kingma2014adam} for updates in \cref{alg:atac (practice)} (i.e. $f - \eta_{\textrm{fast}}\nabla l_{\textrm{critic}}$ and $\pi - \eta_{\textrm{slow}}\nabla l_{\textrm{actor}}$).

\cref{alg:atac (practice)} is a two-timescale first-order algorithm~\citep{borkar1997stochastic,maei2009convergent}, where the critic is updated with a much faster rate $\eta_{\textrm{fast}}$ than the actor with $ \eta_{\textrm{slow}}$.
This two-timescale update is designed to mimic the oracle updates in \cref{alg:atac (theory)}. Using $\eta_{\textrm{fast}} \gg \eta_{\textrm{slow}}$ allows us to approximately treat the critic in \cref{alg:atac (practice)} as the solution to the pessimistic policy evaluation step in \cref{alg:atac (theory)} for a given actor~\citep{maei2009convergent}; on the other hand, the actor's gradient update rule is reminiscent of the incremental nature of no-regret optimization oracles.

\subsubsection{Critic Update}

The update in \cref{ln:update critic} of \cref{alg:atac (practice)} is a first-order approximation of \cref{lin:reg_pes_pe} in \cref{alg:atac (theory)}. We discuss the important design decisions of this practical critic update below.

\mypar{Projection}
Each critic update performs a projected mini-batch gradient step, where the projection to $\Fcal$ ensures bounded complexity for the critic. We parameterize $\Fcal$ as neural networks with $\ell_2$ \emph{bounded} weights.\footnote{We impose no constraint on the bias term. 
}
The projection is crucial to ensure stable learning across all $\beta$ values. 
The use of projection can be traced back to the training Wasserstein GAN~\citep{arjovsky2017wasserstein} or IPM-based IL~\citep{swamy2021moments}. 
We found alternatives such as weight decay penalty to be less reliable. 

\mypar{Double Q residual algorithm loss}
Off-policy optimization with function approximators and bootstrapping faces the notorious issue of deadly triad~\citep{sutton2018reinforcement}. Commonly this is mitigated through the use of double Q heuristic~\citep{fujimoto2018addressing,haarnoja2018soft}; however, we found that this technique alone is insufficient to enable numerically stable policy evaluation when the policy $\pi$ takes very different actions\footnote{Divergence often happens, e.g., when $\pi$ is uniform.} from the behavior policy $\mu$.
To this end, we design a new surrogate for the Bellman error $\Ecal_{D}(f,\pi)$ for \cref{alg:atac (practice)}, by combining the double Q heuristic and the objective of the Residual Algorithm (RA)~\citep{baird1995residual}, both of which are previous attempts to combat the deadly triad.
Specifically, we design the surrogate loss as the convex combination of the temporal difference (TD) losses of the critic and its delayed targets:
\begin{align} \label{eq:DQRA loss}
\hspace{-3mm}
    \Ecal_{\Dcal}^{w}(f, \pi) \coloneqq (1-w) \Ecal_{\Dcal}^{\textrm{td}}(f,f,\pi) +  w \Ecal_{\Dcal}^{\textrm{td}}(f, \bar{f}_{\min},\pi)
\end{align}
where $w\in[0,1]$,
$
    \Ecal_{\Dcal}^{\textrm{td}}(f,f',\pi) \coloneqq \E_{\Dcal}[(f(s,a) - r - \gamma f'(s',\pi))^2]
$, and $\bar{f}_{\min}(s,a) \coloneqq \min_{i=1,2} \bar{f}_i(s,a)$.
We call the objective in \Eqref{eq:DQRA loss}, the \emph{DQRA loss}.
We found that using the DQRA loss significantly improves the optimization stability compared with just the double Q heuristic alone; see \cref{fig:dqra ablation}. As a result, \algo can perform stable optimization with higher $\beta$ values and make the learner less pessimistic.
%
This added stability of DQRA comes from that the residual error $\Ecal_{\Dcal}^{\textrm{td}}(f,f,\pi)$ is a fixed rather than a changing objective. 
%
This stabilization overcomes potential biases due to the challenges (related to double sampling) in unbiased gradient estimation of the RA objective. Similar observations were made by~\citet{wang2021convergent} for online RL.
In practice, we found that $w=0.5$ works stably; using $w\approx 0$ ensures numerical stability, but has a worst-case exponentially slow convergence speed and often deteriorates neural network learning~\citep{schoknecht2003td,wang2021convergent}.
In \cref{sec:experiments}, we show an ablation to study the effects of $w$.

\begin{figure}[t]
	\centering
	\begin{subfigure}{0.23\textwidth}
		\includegraphics[width=\textwidth]{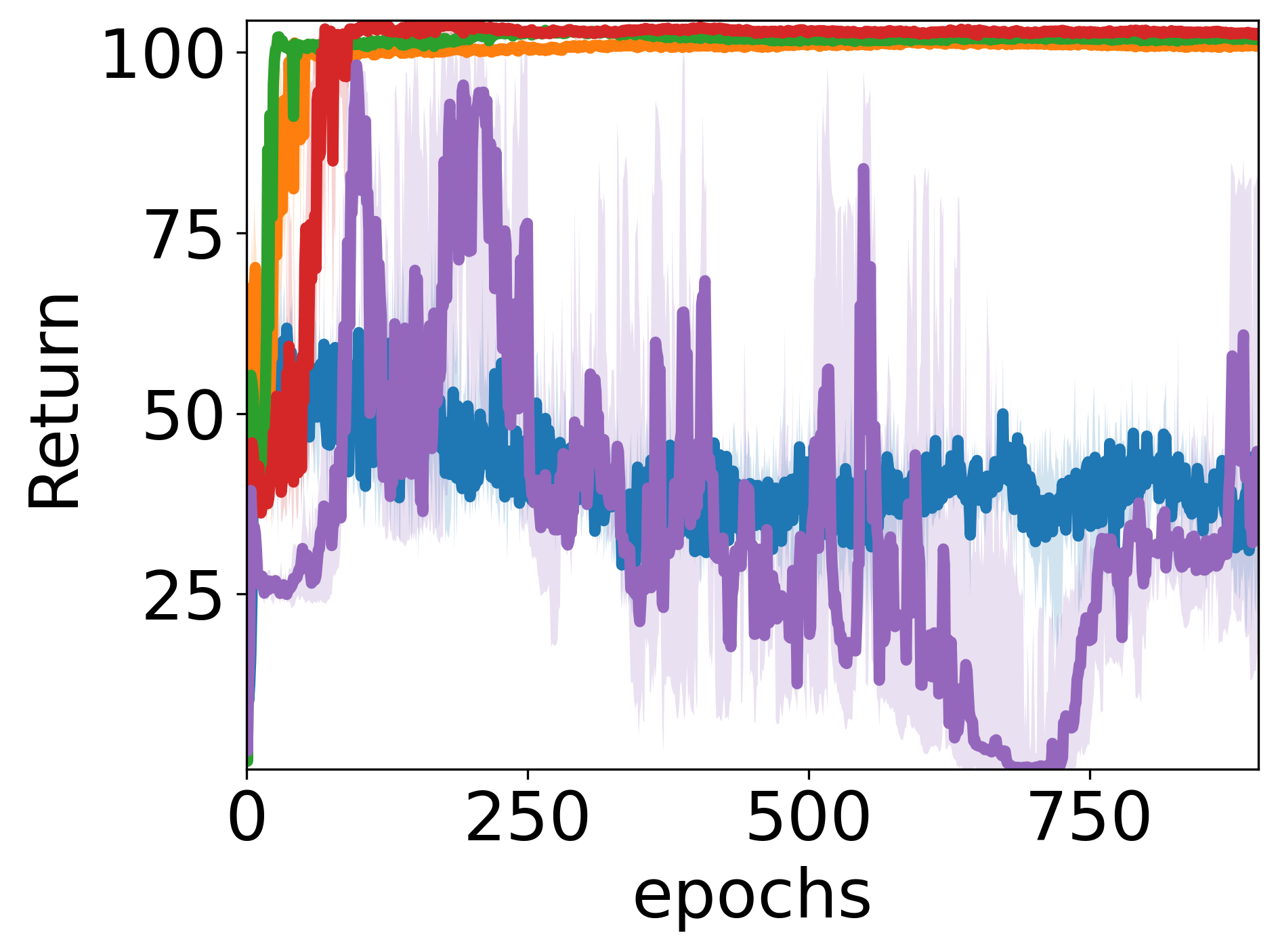}
		\label{fig:rg ablation performance}
	\end{subfigure}
	~~
	\begin{subfigure}{0.23\textwidth}
		\includegraphics[width=\textwidth]{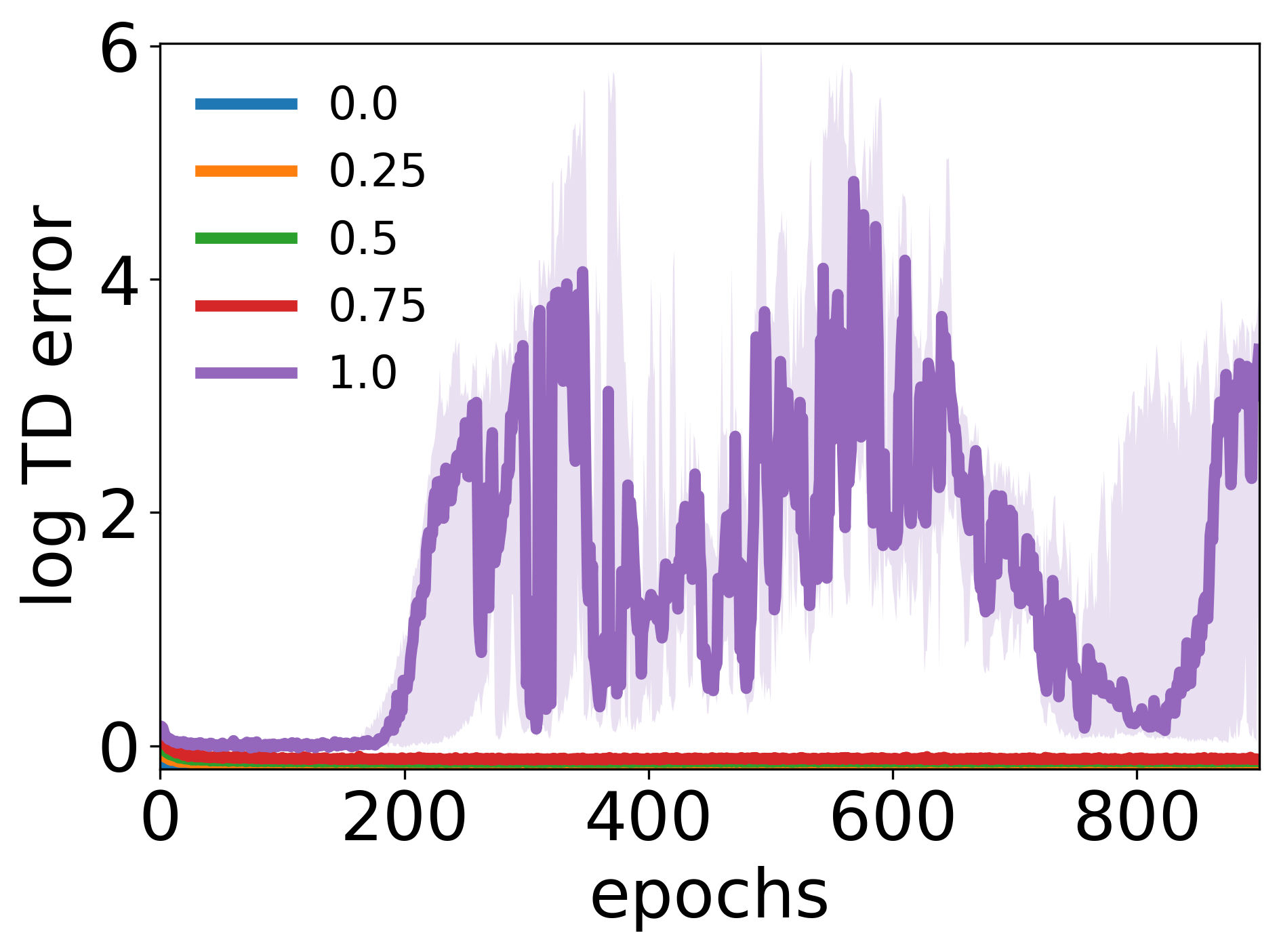}
		\label{fig:rg ablation td error}
	\end{subfigure}

\caption{\small{Ablation of the DQRA loss with different mixing weights $w$ in \Eqref{eq:DQRA loss}. The plots show the policy performance and TD error across optimization epochs of \algo with the hopper-medium-replay dataset. The stability and performance are greatly improved when $w\in(0,1)$.
For each $w$, the plot shows the $25^{th}$, $50^{th}$, $75^{th}$ percentiles over 10 random seeds.}}
\label{fig:dqra ablation}
\end{figure}

\begin{table*}[ht!]
\footnotesize
    \centering
        \begin{tabular}{|c |  c| c c c c c c c c c |}
         \hline
 & Behavior & ATAC$^*$ & ATAC & ATAC$_0^*$ & ATAC$_0$ & CQL & COMBO & TD3+BC & IQL & BC\\
\hline
halfcheetah-rand & -0.1 & 4.8 & 3.9 & 2.3 & 2.3 & 35.4 & \textbf{38.8} & 10.2 & -  & 2.1\\
walker2d-rand & 0.0 & \textbf{8.0} & 6.8 & 7.6 & 5.7 & 7.0 & 7.0 & 1.4 & -  & 1.6\\
hopper-rand & 1.2 & \textbf{31.8} & 17.5 & 31.6 & 18.2 & 10.8 & 17.9 & 11.0 & -  & 9.8\\
halfcheetah-med & 40.6 & \textbf{54.3} & \textbf{53.3} & 43.9 & 36.8 & 44.4 & \textbf{54.2} & 42.8 & 47.4 & 36.1\\
walker2d-med & 62.0 & \textbf{91.0} & \textbf{89.6} & \textbf{90.5} & \textbf{89.6} & 74.5 & 75.5 & 79.7 & 78.3 & 6.6\\
hopper-med & 44.2 & \textbf{102.8} & 85.6 & \textbf{103.5} & 94.8 & 86.6 & 94.9 & \textbf{99.5} & 66.3 & 29.0\\
halfcheetah-med-replay & 27.1 & 49.5 & 48.0 & 49.2 & 47.2 & 46.2 & \textbf{55.1} & 43.3 & 44.2 & 38.4\\
walker2d-med-replay & 14.8 & \textbf{94.1} & 92.5 & \textbf{94.2} & 89.8 & 32.6 & 56.0 & 25.2 & 73.9 & 11.3\\
hopper-med-replay & 14.9 & \textbf{102.8} & \textbf{102.5} & \textbf{102.7} & \textbf{102.1} & 48.6 & 73.1 & 31.4 & 94.7 & 11.8\\
halfcheetah-med-exp & 64.3 & \textbf{95.5} & \textbf{94.8} & 41.6 & 39.7 & 62.4 & 90.0 & \textbf{97.9} & 86.7 & 35.8\\
walker2d-med-exp & 82.6 & \textbf{116.3} & \textbf{114.2} & \textbf{114.5} & 104.9 & 98.7 & 96.1 & 101.1 & \textbf{109.6} & 6.4\\
hopper-med-exp & 64.7 & \textbf{112.6} & \textbf{111.9} & 83.0 & 46.5 & \textbf{111.0} & \textbf{111.1} & \textbf{112.2} & 91.5 & \textbf{111.9}\\
\hline
pen-human & 207.8 & 79.3 & 53.1 & \textbf{106.1} & 61.7 & 37.5 & -  & -  & 71.5 & 34.4\\
hammer-human & 25.4 & \textbf{6.7} & 1.5 & 3.8 & 1.2 & \textbf{4.4} & -  & -  & 1.4 & 1.5\\
door-human & 28.6 & 8.7 & 2.5 & \textbf{12.2} & 7.4 & \textbf{9.9} & -  & -  & 4.3 & 0.5\\
relocate-human & 86.1 & \textbf{0.3} & \textbf{0.1} & \textbf{0.5} & \textbf{0.1} & \textbf{0.2} & -  & -  & \textbf{0.1} & \textbf{0.0}\\
pen-cloned & 107.7 & 73.9 & 43.7 & \textbf{104.9} & 68.9 & 39.2 & -  & -  & 37.3 & 56.9\\
hammer-cloned & 8.1 & 2.3 & 1.1 & \textbf{3.2} & 0.4 & 2.1 & -  & -  & 2.1 & 0.8\\
door-cloned & 12.1 & \textbf{8.2} & 3.7 & 6.0 & 0.0 & 0.4 & -  & -  & 1.6 & -0.1\\
relocate-cloned & 28.7 & \textbf{0.8} & \textbf{0.2} & \textbf{0.3} & \textbf{0.0} & \textbf{-0.1} & -  & -  & \textbf{-0.2} & \textbf{-0.1}\\
pen-exp & 105.7 & \textbf{159.5} & 136.2 & \textbf{154.4} & 97.7 & 107.0 & -  & -  & -  & 85.1\\
hammer-exp & 96.3 & \textbf{128.4} & \textbf{126.9} & 118.3 & 99.2 & 86.7 & -  & -  & -  & \textbf{125.6}\\
door-exp & 100.5 & \textbf{105.5} & \textbf{99.3} & \textbf{103.6} & 48.3 & \textbf{101.5} & -  & -  & -  & 34.9\\
relocate-exp & 101.6 & \textbf{106.5} & \textbf{99.4} & \textbf{104.0} & 74.3 & 95.0 & -  & -  & -  & \textbf{101.3}\\
         \hline
        \end{tabular}
    \caption{\small Evaluation on the D4RL dataset. Algorithms with score within $\epsilon$ from the best on each domain are marked in bold, where $\epsilon = 0.1 |J(\mu)|$. Baseline results are from the respective papers. For \algo variants, we take the median score over $10$ seeds.}
    \label{tab:main exp results}
\end{table*}

\subsubsection{Actor Update}

The actor update aims to achieve no-regret with respect to the adversarially chosen critics. In \cref{alg:atac (practice)}, we adopt a gradient based update (implemented as ADAM) mimicking  the proximal nature of theoretical no-regret algorithms.
Although ADAM has no formal no-regret guarantees for neural network learning, it works quite well in practice for RL and IL algorithms based on no-regret learning~\citep{sun2017deeply,cheng2019predictor,cheng2021heuristic}.

\mypar{Projection}
We set $\Pi$ to be a class of policies with a minimal entropy constraint, so the projection in \cref{ln:update actor} ensures that the updated policy has a non-zero entropy. Soft policy iteration style theoretical algorithms naturally keep a reasonable entropy, and practically this avoids getting trapped in poor local optima.
We implement the constraint by a Lagrange relaxation similar to SAC~\citep{haarnoja2018soft}. 

\mypar{Actor loss with a single critic}
While the critic optimization uses the double Q heuristic for numerical stability, the actor loss only uses one of the critics (we select $f_1$). This actor loss is similar to TD3~\citep{fujimoto2018addressing}, but different from SAC~\citep{haarnoja2018soft} which takes $\min_{i=1,2} f_i(s,a)$ as the objective. This design choice is critical to enable \algo's IL behavior when $\beta$ is low. On the contrary,
using the SAC-style loss produces instability for small $\beta$, with the actor loss oscillating in a limit cycle between the two critics.

%



%

\section{Experiments} \label{sec:experiments}

We test the effectiveness of \algo (\cref{alg:atac (practice)}) in terms of performance and robust policy improvement using the D4RL offline RL benchmark's continuous control domains~\citep{fu2020d4rl}. More details are given in \cref{sec:exp details}.

\mypar{Setup and hyperaparameter selection}
We compare \algo (\cref{alg:atac (practice)}) with recent offline RL algorithms CQL~\citep{kumar2020conservative}, COMBO~\citep{yu2021combo}, TD3+BC~\citep{fujimoto2021minimalist}, IQL~\citep{kostrikov2021offline}, as well as the offline IL baseline, behavior cloning (BC). 
We also introduce an absolute pessimism version of \algo (denoted ATAC$_0$), where we replace $\Lcal_{\Dcal_{\textrm{mini}}}(f, \pi)$ in $l_{\textrm{critic}}$ of \cref{alg:atac (practice)} with $f(s_0, \pi)$. ATAC$_0$ can be viewed as a deep learning implementation of the theoretical algorithm PSPI from~\citet{xie2021bellman} with the template of \cref{alg:atac (practice)}.

In \cref{alg:atac (practice)}, we use $\eta_{\textrm{fast}} = 0.0005$ and $\eta_{\textrm{slow}} = 10^{-3} \eta_{\textrm{fast}}$ based on an offline tuning heuristic,
$\tau = 0.005$ from the work of~\citet{haarnoja2018soft}, and $w = 0.5$, \emph{across all domains}. We include an ablation for $w$ later and further details of our setup are given in~\cref{sec:exp details}. 
The regularization coefficient $\beta$ is our only hyperparameter which varies across datasets, based on an online selection.
Specifically, we run 100 epochs of BC for warm start; followed by 900 epochs of \algo, where 1 epoch denotes 2K gradient updates.
For each dataset, we report the median results over 10 random seeds. Since \algo does not have guarantees on last-iterate convergence, we report also the results of both the last iterate (denoted as \algo and ATAC$_0$) and the best checkpoint (denoted as ATAC$^*$ and ATAC$_0^*$) selected among $9$ checkpoints (each was made every 100 epochs).
The hyperparameter $\beta$ is picked separately for \algo, ATAC$_0$,  ATAC$^*$ and ATAC$_0^*$.


\mypar{Comparison with offline RL baselines}
Overall the experimental results in \cref{tab:main exp results} show that ATAC and ATAC$^*$ outperform other model-free offline RL baselines consistently and model-based method COMBO mostly.
Especially significant improvement is seen in \emph{walker2d-medium, walker2d-medium-replay, hopper-medium-replay and pen-expert}, although the performance is worse than COMBO and CQL in the \emph{halfhcheetah-rand}. It turns out that our fixed learning rate parameter does not result in sufficient convergence of \algo on this domain.
Our adaptation of PSPI (i.e. ATAC$_0$ and ATAC$_0^*$) is remarkably competitive with state-of-the-art baselines. This is the first empirical evaluation of PSPI, which further demonstrates the effectiveness of our design choices in \cref{alg:atac (practice)}. However, ATAC$_0$ and ATAC$_0^*$ perform worse than ATAC and ATAC$^*$, except for \emph{pen-human, door-human, and pen-cloned}.
In \cref{sec:exp details} we show \algo and ATAC$_0$'s variability of performance across seeds by adding $25\%$ and $75\%$ quantiles of scores across 10 random seeds. (For baselines we only have scalar performance from the published results.)

\mypar{Robust policy improvement}
We study whether the practical version of \algo also enjoys robust policy improvement as \cref{prop:safe_pi} proves for the theoretical version. We show how ATAC$^*$ performs with various $\beta$ values in \cref{fig:robust PI} on hopper. The results are consistent with the theoretical prediction in \cref{prop:safe_pi}: \algo robustly improves upon the behavior policy almost for all $\beta$ except very large ones.
For large $\beta$, \cref{prop:safe_pi} shows that the finite-sample statistical error 
dominates the bound.
\algo does, however, not improve from the behavior policy on  \emph{*-human} and \emph{*-cloned} even for well-tuned $\beta$; in fact, none of the offline RL algorithms does.
We suspect that this is due to the failure of the realizability assumption $\mu\in\Pi$, as these datasets contain human demonstrations which can be non-Markovian. We include the variation of results across $\beta$ for all datasets as well as statistics of robust policy improvement across $\beta$ and iterates in
\cref{sec:exp details}.
This robust policy improvement property of \algo means that practitioners can tune the performance of \algo by starting with $\beta=0$ and gradually increasing $\beta$ until the performance drops, without ever deploying a policy significantly worse than the previous behavior policy.

\mypar{Ablation of DQRA loss}
We show that the optimization stability from the DQRA loss is a key contributor to \algo's performance by an ablation. We run \algo with various $w$ on \emph{hopper-medium-replay}. When $w=1$ (i.e. using conventional bootstrapping with double Q), the Bellman minimization part becomes unstable and the TD error $\Ecal_{\Dcal}^{\textrm{td}}(f,f,\pi)$ diverges. Using just the residual gradient ($w=0$), while being numerical stable, leads to bad policy performance as also observed in the literature~\citep{schoknecht2003td,wang2021convergent}.
For $w\in(0,1)$, the stability and performance are usually significantly better than $w\in\{0,1\}$. For simplicity, we use $w=0.5$ in our experiments.

\section{Discussion and Conclusion}
We propose the concept of relative pessimism for offline RL and use it to design a new algorithm \algo based on a Stackelberg game formulation. \algo enjoys strong guarantees comparable to prior theoretical works, with an additional advantage of robust policy improvement due to relative pessimism. Empirical evaluation confirms the theoretical predictions and demonstrates \algo's state-of-the-art performance on D4RL offline RL benchmarks.

\algo shows a natural bridge between IL and offline RL. From its perspective, IL is an offline RL problem with the largest uncertainty on the value function (since IL does not have reward information), as captured by setting $\beta=0$ in \algo. In this case, the best policy under relative pessimism is to mimic the behavior policy exactly; otherwise, there is always a scenario within the uncertainty where the agent performs worse than the behavior policy.
Only by considering the reduced uncertainty due to labeled rewards, it becomes possible for offline RL to learn a policy that strictly improves over the behavior policy.
Conversely, we can view IL as the most pessimistic offline RL algorithm, which ignores the information in the data reward labels. Indeed IL does not make assumption on the data coverage, which is the core issue offline RL attempts to solve.
We hope that this insightful connection can encourage future research on advancing IL and offline RL.

Finally, we remark on some limitations of \algo. While \algo has strong theoretical guarantees with general function approximators, it comes with a computational cost that its adversarial optimization problem (like that of~\citet{xie2021bellman}) is potentially harder to solve than alternative offline RL approaches based on dynamic programming in a fixed pessimism MDP~\citep{jin2020pessimism,liu2020provably,fujimoto2021minimalist,kostrikov2021offline}. For example, in our theoretical algorithm (\cref{alg:atac (theory)}), we require having a no-regret policy optimization oracle (\cref{def:no-regret}), which we only know is provably time and memory efficient for linear function approximators and softmax policies~\citep{xie2021bellman}.\footnote{When using nonlinear function approximators, the scheme there require memory linear in $K$.}
This extra computational difficulty also manifests in the IL special case of \algo (i.e. $\beta=0$): \algo reduces to IPM-minimization or Wasserstein-GAN for IL which requires harder optimization than BC based on maximum likelihood estimation, though the adversarial training version can produce a policy of higher quality. How to strike a better balance between the quality of the objective function and its computational characteristics is an open question. 

\section*{Acknowledgment}
NJ acknowledges funding support from ARL Cooperative Agreement W911NF-17-2-0196, NSF IIS-2112471, NSF CAREER IIS-2141781, and Adobe Data Science Research Award.

\bibliography{ref}
\bibliographystyle{icml2022}

\clearpage
\appendix
\onecolumn
\allowdisplaybreaks
\section{Related Works} \label{sec:related}

There is a rich literature on offline RL with function approximation when the data distribution $\mu$ is sufficiently rich to cover the state-action distribution $d^\pi$ for any $\pi\in\Pi$~\citep{antos2008learning,munos2003error,munos2008finite,farahmand2010error,chen2019information,xie2020q}. However, this is a prohibitive assumption in practice where the data distribution is typically constrained by the quality of available policies, safety considerations and existing system constraints which can lead it to have a significantly narrower coverage. Based on this observation, there has been a line of recent works in both the theoretical and empirical literature that systematically consider datasets with inadequate coverage.

The methods designed for learning without coverage broadly fall in one of two categories. Many works adopt the \emph{behavior regularization} approach, where the learned policy is regularized to be close to the behavior policy in states where adequate data is not observed. On the theoretical side, some works~\citep{laroche2019safe,kumar2019stabilizing,fujimoto2018addressing} provide \emph{safe policy improvement} guarantees, meaning that the algorithms always do at least as well as the behavior policy, while improving upon it when possible. These and other works~\citep{wu2019behavior,fujimoto2021minimalist} also demonstrate the benefits of this principle in comprehensive empirical evaluations.

A second class of methods follow the principle of \emph{pessimism in the face of uncertainty}, and search for a policy with the best value under all possible scenarios consistent with the data. Some papers perform this reasoning in a model-based manner~\citep{kidambi2020morel,yu2020mopo}. In the model-free setting, \citet{liu2020provably} define pessimism by truncating Bellman backups from states with limited support in the data and provide theoretical guarantees for the function approximation setting when the behavior distribution $\mu$ is known or can be easily estimated from samples, along with proof-of-concept experiments. The need to estimate $\mu$ has been subsequently removed by several recent works in both linear~\citep{jin2020pessimism,zanette2021provable} and non-linear~\citep{xie2021bellman, uehara2021representation} settings.

Of these, the work of~\citet{xie2021bellman} is the closest to this paper. Their approach optimizes a maximin objective where the maximization is over policies and minimization over all $f\in\Fcal$ which are Bellman-consistent for that policy under the data distribution. Intuitively, this identifies an $\Fcal$-induced lower bound for the value of each policy through the Bellman constraint and maximizes that lower bound. They also develop a regularized version more amenable to practical implementation, but provide no empirical validation of their approach. While the optimization of a pessimistic estimate of $J(\pi)$ results in a good policy with well-chosen hyperparameters, we argue  that maximizing an alternative lower bound on the relative performance difference $J(\pi) - J(\mu)$ is nearly as good in terms of the absolute quality of the returned policy with well-chosen hyperparameters, but additionally improves upon the behavior policy for all possible choices of certain hyperparameters.

On the empirical side, several recent  approaches~\citep{kumar2020conservative,yu2021combo,kostrikov2021offline} show promising empirical results for pessimistic methods. Many of these works consider policy iteration-style approaches where the policy class is implicitly defined in terms of a critic (e.g. through a softmax), whereas we allow explicit specification of both actor and critic classes. Somewhat related to our approach, the CQL algorithm~\citep{kumar2020conservative} trains a critic $Q$ by maximizing the combination of a lower bound on $J(\pi_Q) - J(\mu)$, where $\pi_Q$ is an implicit policy parameterized by $Q$, along with a Bellman error term for the current actor policy. The actor is trained with respect to the resulting critic. Lacking a clear objective like~\Eqref{eq:game_formuation}, this approach does not enjoy the robust policy improvement or other theoretical guarantees we establish in this paper. (We provide a detailed comparison with CQL in \cref{sec:comparison with CQL})
More generally, our experiments show that several elements of the theoretical design and practical implementation of our algorithm \algo allow us to robustly outperform most of these baselines in a comprehensive evaluation.



\section{Guarantees of Theoretical Algorithm}
\label{app:theo_proofs}
In this section, we provide the guarantees of theoretical algorithm including the the results provided in Section~\ref{sec:theoretical_algo}.



\subsection{Concentration Analysis}

This section provides the main results regarding $\Ecal_\Dcal(f,\pi)$ and its corresponding Bellman error. The results in this section are analogs of the results of~\citet[Appendix A]{xie2021bellman}, but we use covering numbers to provide finer characteristics of the concentration. We provide the background of covering number as follows.
\begin{definition}[$\varepsilon$-covering number]
An $\varepsilon$-cover of a set $\Fcal$ with respect to a metric $\rho$ is a set $\{g_1, \dotsc, g_n\} \subseteq \Fcal$, such that for each $g \in \Fcal$, there exists some $g_i \in \{g_1, \dotsc, g_n\}$ such that $\rho(g,g_i) \leq \varepsilon$.
We define the $\varepsilon$-covering number of a set $\Fcal$ under metric $\rho$, $\Ncal(\Fcal,\varepsilon,\rho)$, to be the the cardinality of the smallest $\varepsilon$-cover.
\end{definition}

Further properties of covering number can be found in standard textbooks~\citep[see, e.g.,][]{wainwright2019high}. In this paper, we will apply the $\varepsilon$-covering number on both function class $\Fcal$ and policy class $\Pi$.
For the function class, we use the following metric
\begin{align}
\label{eq:def_infnorm}
\rho_{\Fcal}(f_1,f_2) \coloneqq &~ \|f_1 - f_2\|_\infty = \sup_{(s,a) \in \Scal \times \Acal} |f_1(s,a) - f_2(s,a)|.
\end{align}
We use $\Ncal_\infty(\Fcal,\varepsilon)$ to denote the $\varepsilon$-covering number of $\Fcal$ w.r.t.~metric $\rho_{\Fcal}$ for simplicity.

Similarly, for the policy class, we define the metric as follows
\begin{align}
\label{eq:def_inf1norm}
\rho_\Pi(\pi_1,\pi_2) \coloneqq \|\pi_1 - \pi_2\|_{\infty,1} = \sup_{s \in \Scal} \|\pi_1(\cdot|s) - \pi_2(\cdot|s)\|_1,
\end{align}
and we use $\Ncal_{\infty,1}(\Pi,\varepsilon)$ to denote the $\varepsilon$-covering number of $\Pi$ w.r.t.~metric $\rho_\Pi$ for simplicity.

The following two theorems are the main results of this concentration analysis.

\begin{theorem}
\label{thm:version_space}
For any $\pi \in \Pi$, let $f_\pi$ be defined as follows,
\begin{align*}
f_\pi \coloneqq &~ \argmin_{f \in \Fcal}\sup_{\text{admissible }\nu} \left\|f - \Tcal^\pi f\right\|_{2,\nu}^2.
\end{align*}
Then, for $\Ecal_\Dcal(f_\pi,\pi)$ (defined in \Eqref{eq:defmsbope}), the following holds with probability at least $1-\delta$ for all $\pi \in \Pi$:
\begin{align*}
\Ecal_\Dcal(f_\pi,\pi) \leq \Ocal \left( \frac{\Vmax^2 \log \nicefrac{|\Ncal_\infty(\Fcal,\nicefrac{\Vmax}{N})||\Ncal_{\infty,1}(\Pi,\nicefrac{1}{N})|}{\delta}}{N} + \varepsilon_{\Fcal} \right) \eqqcolon \varepsilon_r.
\end{align*}
\end{theorem}
%

We now show that $\Ecal_\Dcal(f,\pi)$ could effectively estimate $\|f - \Tcal^\pi f \|_{2,\mu}^2$.
\begin{theorem}
\label{thm:mspo2be}
With probability at least $1-\delta$, for any $\pi \in \Pi$, $f \in \Fcal$,
\begin{align}
\label{eq:upbdmsbe_ori}
\| f - \Tcal^\pi f\|_{2,\mu} - \sqrt{\Ecal_\Dcal(f,\pi)} \leq \Ocal \left( \Vmax \sqrt{\frac{\log \nicefrac{|\Ncal_\infty(\Fcal,\nicefrac{\Vmax}{N})||\Ncal_{\infty,1}(\Pi,\nicefrac{1}{N})|}{\delta}}{N}} + \sqrt{\varepsilon_{\Fcal,\Fcal}}\right) .
\end{align}
\end{theorem}
When setting $\Ecal_\Dcal(f,\pi) = \varepsilon_r$, \Eqref{eq:upbdmsbe_ori} implies a bound on $\| f - \Tcal^\pi f\|_{2,\mu}$ which we denote as $\sqrt{\varepsilon_b}$ and will be useful later. That is,
\begin{align}
\label{eq:upbdmsbe}
\sqrt{\varepsilon_b} \coloneqq \sqrt{\varepsilon_r} + \Ocal \left( \Vmax \sqrt{\frac{ \log \nicefrac{|\Ncal_\infty(\Fcal,\nicefrac{\Vmax}{N})||\Ncal_{\infty,1}(\Pi,\nicefrac{1}{N})|}{\delta}}{N}} + \sqrt{\varepsilon_{\Fcal,\Fcal}} \right) .
\end{align}

We first provide some complementary lemmas used for proving Theorems~\ref{thm:version_space} and \ref{thm:mspo2be}. The first lemma, Lemma~\ref{lem:bernstein_general}, is the only place where we use concentration inequalities on $\Ecal_{\Dcal}$, and all high-probability statements regarding $\Ecal_{\Dcal}$ follow deterministically from Lemma~\ref{lem:bernstein_general}.


\begin{lemma}
\label{lem:bernstein_general}
With probability at least $1-\delta$, for any $f, g_1, g_2 \in \Fcal$ and $\pi \in \Pi$,
\begin{align*}
&~ \bigg| \left\| g_1 - \Tcal^\pi f\right\|_{2,\mu}^2 - \left\| g_2 - \Tcal^\pi f\right\|_{2,\mu}^2
\\
&~ - \frac{1}{N} \sum_{(s,a,r,s') \in \Dcal} \left(g_1(s,a) - r - \gamma f(s',\pi) \right)^2 + \frac{1}{N} \sum_{(s,a,r,s') \in \Dcal} \left(g_2(s,a)  - r - \gamma f(s',\pi) \right)^2 \bigg|
\\
\leq &~ \Ocal \left( \Vmax \|g_1 - g_2\|_{2,\mu}\sqrt{\frac{ \log \frac{|\Ncal_\infty(\Fcal,\frac{\Vmax}{N})||\Ncal_{\infty,1}(\Pi,\frac{1}{N})|}{\delta}}{N}} + \frac{\Vmax^2 \log \frac{|\Ncal_\infty(\Fcal,\frac{\Vmax}{N})||\Ncal_{\infty,1}(\Pi,\frac{1}{N})|}{\delta}}{N} \right).
\end{align*}
\end{lemma}
\begin{proof}[\bf\em Proof of Lemma~\ref{lem:bernstein_general}]

This proof follows a similar approach as the proof of~\citet[Lemma A.4]{xie2021bellman}, but ours is established based on a more refined concentration analysis via covering number. We provide the full detailed proof here for completeness.
By a standard calculation,
\begin{align}
&~ \frac{1}{N} \sum_{(s,a,r,s') \in \Dcal} \left(g_1(s,a) - r - \gamma f(s',\pi) \right)^2 - \frac{1}{N} \sum_{(s,a,r,s') \in \Dcal} \left(g_2(s,a)  - r - \gamma f(s',\pi) \right)^2
\nonumber
\\
= &~ \frac{1}{N} \sum_{(s,a,r,s') \in \Dcal} \left( \left(g_1(s,a) - r - \gamma f(s',\pi) \right)^2 - \left(g_2(s,a)  - r - \gamma f(s',\pi) \right)^2 \right)
\nonumber
\\
\label{eq:empmsbo}
= &~ \frac{1}{N} \sum_{(s,a,r,s') \in \Dcal} \left( \left(g_1(s,a) - g_2(s,a) \right) \left(g_1(s,a) + g_2(s,a)  - 2 r - 2 \gamma f(s',\pi) \right) \right).
\end{align}

Similarly, letting $\mu \times (\Pcal,R)$ denote the distribution $(s,a) \sim \mu, r = R(s,a), s' \sim \Pcal(\cdot|s,a)$, we have
\begin{align}
&~ \E_{\mu\times(\Pcal,R)} \left[\left(g_1(s,a) - r - \gamma f(s',\pi) \right)^2\right] - \E_{\mu\times(\Pcal,R)} \left[\left(g_2(s,a) - r - \gamma f(s',\pi) \right)^2 \right] \nonumber
\\
\overset{\text{(a)}}{=} &~ \E_{\mu\times(\Pcal,R)} \left[ \left(g_1(s,a) - g_2(s,a) \right) \left(g_1(s,a) + g_2(s,a)  - 2 r - 2 \gamma f(s',\pi) \right) \right] \nonumber
\\
= &~ \E_\mu \left[ \E \left[ \left(g_1(s,a) - g_2(s,a) \right) \left(g_1(s,a) + g_2(s,a)  - 2 r - 2 \gamma f(s',\pi) \right) \middle| s,a \right] \right] \nonumber
\\
\label{eq:popmsbo_middle}
= &~ \E_\mu \left[  \left(g_1(s,a) - g_2(s,a) \right) \left(g_1(s,a) + g_2(s,a)  - 2 \left(\Tcal^\pi f\right)(s,a) \right) \right]
\\
\label{eq:popmsbo}
\overset{\text{(b)}}{=} &~ \E_{\mu} \left[\left(g_1(s,a) - \left(\Tcal^\pi f\right)(s,a) \right)^2\right] - \E_{\mu} \left[\left(g_2(s,a) - \left(\Tcal^\pi f\right)(s,a) \right)^2 \right],
\end{align}
where (a) and (b) follow from the similar argument to \Eqref{eq:empmsbo}.

By using \Eqref{eq:empmsbo} and \Eqref{eq:popmsbo}, we know
\begin{align}
&~ \E_{\mu\times(\Pcal,R)} \left[ \frac{1}{N} \sum_{(s,a,r,s') \in \Dcal} \left(g_1(s,a) - r - \gamma f(s',\pi) \right)^2 - \frac{1}{N} \sum_{(s,a,r,s') \in \Dcal} \left(g_2(s,a)  - r - \gamma f(s',\pi) \right)^2 \right] \nonumber
\\
= &~\E_{\mu} \left[\left(g_1(s,a) - \left(\Tcal^\pi f\right)(s,a) \right)^2\right] - \E_{\mu} \left[\left(g_2(s,a) - \left(\Tcal^\pi f\right)(s,a) \right)^2 \right]. \nonumber
\end{align}

Now, let $\Fcal_{\varepsilon_1}$ be an ${\varepsilon_1}$-cover of $\Fcal$ and $\Pi_{\varepsilon_2}$ be an ${\varepsilon_2}$-cover of $\Pi$, so that we know: i) $|\Fcal_{\varepsilon_1}| = \Ncal_\infty(\Fcal,\varepsilon_1)$, $|\Pi_{\varepsilon_2}| = \Ncal_{\infty,1}(\Pi,\varepsilon_2)$; ii) there exist $\ftilde, \gtilde_1, \gtilde_2 \in \Fcal_{\varepsilon_1}$ and $\pitilde \in \Pi_{\varepsilon_2}$, such that $\| f - \ftilde \|_\infty, \| g_1 - \gtilde_1 \|_\infty, \| g_2 - \gtilde_2 \|_\infty \leq \varepsilon_1$ and $\| \pi - \pitilde \|_{\infty,1} \leq \varepsilon_2$, where $\| \cdot \|_\infty$ and $\| \cdot \|_{\infty,1}$ are defined in \Eqref{eq:def_infnorm} and \Eqref{eq:def_inf1norm}.

Then, with probability at least $1-\delta$, for all $f, g_1, g_2 \in \Fcal$, $\pi \in \Pi$, and the corresponding $\ftilde, \gtilde_1, \gtilde_2, \pitilde$,
\begin{align}
&~ \bigg| \E_{\mu} \left[\left(\gtilde_1(s,a) - \left(\Tcal^\pitilde \ftilde\right)(s,a) \right)^2\right] - \E_{\mu} \left[\left(\gtilde_2(s,a) - \left(\Tcal^\pitilde \ftilde\right)(s,a) \right)^2 \right]
\nonumber
\\
&~ - \frac{1}{N} \sum_{(s,a,r,s') \in \Dcal} \left(\gtilde_1(s,a) - r - \gamma \ftilde(s',\pitilde) \right)^2 + \frac{1}{N} \sum_{(s,a,r,s') \in \Dcal} \left(\gtilde_2(s,a)  - r - \gamma \ftilde(s',\pitilde) \right)^2 \bigg|
\nonumber
\\
= &~ \bigg| \E_{\mu} \left[\left(\gtilde_1(s,a) - \left(\Tcal^\pitilde \ftilde\right)(s,a) \right)^2\right] - \E_{\mu} \left[\left(\gtilde_2(s,a) - \left(\Tcal^\pitilde \ftilde\right)(s,a) \right)^2 \right]
\nonumber
\\
&~ - \frac{1}{N} \sum_{(s,a,r,s') \in \Dcal} \left( \left(\gtilde_1(s,a) - \gtilde_2(s,a) \right) \left(\gtilde_1(s,a) + \gtilde_2(s,a)  - 2 r - 2 \gamma \ftilde(s',\pitilde) \right) \right) \bigg|
\nonumber
\\
\leq &~ \sqrt{\frac{4\V_{\mu\times(\Pcal,R)} \left[ \left(\gtilde_1(s,a) - \gtilde_2(s,a) \right) \left(\gtilde_1(s,a) + \gtilde_2(s,a)  - 2 r - 2 \gamma \ftilde(s',\pitilde) \right) \right] \log \frac{|\Ncal_\infty(\Fcal,\varepsilon_1)||\Ncal_{\infty,1}(\Pi,\varepsilon_2)|}{\delta}}{N}}
\nonumber
\\ &~ + \frac{2 \Vmax^2 \log \frac{|\Ncal_\infty(\Fcal,\varepsilon_1)||\Ncal_{\infty,1}(\Pi,\varepsilon_2)|}{\delta}}{3N}, \nonumber
\end{align}
where the first equation follows from \Eqref{eq:popmsbo_middle} and the last inequality follows from the Bernstein's inequality and union bounding over $\Fcal_{\varepsilon_1}$ and $\Pi_{\varepsilon_2}$.

We now upper bound the variance term inside the squareroot of the above expression:
\begin{align*}
&~ \V_{\mu\times(\Pcal,R)} \left[ \left(\gtilde_1(s,a) - \gtilde_2(s,a) \right) \left(\gtilde_1(s,a) + \gtilde_2(s,a)  - 2 r - 2 \gamma \ftilde(s',\pitilde) \right) \right]
\\
\leq &~ \E_{\mu\times(\Pcal,R)} \left[ \left(\gtilde_1(s,a) - \gtilde_2(s,a) \right)^2 \left(\gtilde_1(s,a) + \gtilde_2(s,a)  - 2 r - 2 \gamma \ftilde(s',\pitilde) \right)^2 \right]
\\
\label{eq:msbobernsteinvar}
\leq &~ 4 \Vmax^2\E_{\mu} \left[ \left(\gtilde_1(s,a) - \gtilde_2(s,a) \right)^2 \right].
\end{align*}
where the last inequality follows from the fact of $|\gtilde_1(s,a) + \gtilde_2(s,a)  - 2 r - 2 \gamma \ftilde(s',\pitilde) | \leq 2 \Vmax$.
Therefore, w.p.~$1-\delta$,
\begin{align*}
&~ \bigg| \left\| \gtilde_1 - \Tcal^\pitilde \ftilde\right\|_{2,\mu}^2 - \left\| \gtilde_2 - \Tcal^\pitilde \ftilde\right\|_{2,\mu}^2
\\
&~ - \frac{1}{N} \sum_{(s,a,r,s') \in \Dcal} \left(\gtilde_1(s,a) - r - \gamma \ftilde(s',\pitilde) \right)^2 + \frac{1}{N} \sum_{(s,a,r,s') \in \Dcal} \left(\gtilde_2(s,a)  - r - \gamma \ftilde(s',\pitilde) \right)^2 \bigg|
\\
\leq &~ 4 \Vmax \|\gtilde_1 - \gtilde_2\|_{2,\mu}\sqrt{\frac{ \log \frac{|\Ncal_\infty(\Fcal,\varepsilon_1)||\Ncal_{\infty,1}(\Pi,\varepsilon_2)|}{\delta}}{N}} + \frac{2 \Vmax^2 \log \frac{|\Ncal_\infty(\Fcal,\varepsilon_1)||\Ncal_{\infty,1}(\Pi,\varepsilon_2)|}{\delta}}{3N}.
\end{align*}

By definitions of $\ftilde, \gtilde_1, \gtilde_2$ and $\pitilde$, we know for any $(s,a,r,s')$ tuple,
\begin{align*}
&~ \bigg| \left(g_1(s,a) - r - \gamma f(s',\pi) \right)^2 + \left(g_2(s,a)  - r - \gamma f(s',\pi) \right)^2
\\
&~ - \left(\gtilde_1(s,a) - r - \gamma \ftilde(s',\pitilde) \right)^2 + \left(\gtilde_2(s,a)  - r - \gamma \ftilde(s',\pitilde) \right)^2 \bigg|
= \Ocal(\Vmax\varepsilon_1 + \Vmax^2 \varepsilon_2),
\end{align*}
and
\begin{align*}
\|g_1 - g_2\|_{2,\mu} = &~ \|\gtilde_1 - \gtilde_2 + (g_1 - \gtilde_1) - (g_2 - \gtilde_2)\|_{2,\mu}
\\
\leq &~ \|\gtilde_1 - \gtilde_2\|_{2,\mu} + \|g_1 - \gtilde_1\|_{2,\mu} + \|g_2 - \gtilde_2\|_{2,\mu}
\\
\leq &~ \|\gtilde_1 - \gtilde_2\|_{2,\mu} + 2 \varepsilon_1.
\end{align*}
These implies
\begin{align*}
&~ \bigg| \left\| g_1 - \Tcal^\pi f\right\|_{2,\mu}^2 - \left\| g_2 - \Tcal^\pi f\right\|_{2,\mu}^2
\\
&~ - \frac{1}{N} \sum_{(s,a,r,s') \in \Dcal} \left(g_1(s,a) - r - \gamma f(s',\pi) \right)^2 + \frac{1}{N} \sum_{(s,a,r,s') \in \Dcal} \left(g_2(s,a)  - r - \gamma f(s',\pi) \right)^2 \bigg|
\\
\lesssim &~ \Vmax \|g_1 - g_2\|_{2,\mu}\sqrt{\frac{ \log \frac{|\Ncal_\infty(\Fcal,\varepsilon_1)||\Ncal_{\infty,1}(\Pi,\varepsilon_2)|}{\delta}}{N}} + \frac{\Vmax^2 \log \frac{|\Ncal_\infty(\Fcal,\varepsilon_1)||\Ncal_{\infty,1}(\Pi,\varepsilon_2)|}{\delta}}{N}
\\
&~ + \Vmax \varepsilon_1 \sqrt{\frac{ \log \frac{|\Ncal_\infty(\Fcal,\varepsilon_1)||\Ncal_{\infty,1}(\Pi,\varepsilon_2)|}{\delta}}{N}} + \Vmax\varepsilon_1 + \Vmax^2 \varepsilon_2.
\\
\tag{$x \lesssim y$ means $x \leq C \cdot y$ for some absolute constant $C$}
\end{align*}

Choosing $\varepsilon_1 = \Ocal(\frac{\Vmax}{N})$ and $\varepsilon_2 = \Ocal(\frac{1}{N})$ completes the proof.
\end{proof}


\begin{lemma}
\label{lem:msbosol}
For any $\pi \in \Pi$, let $f_\pi$ and $g$ be defined as follows,
\begin{align*}
f_\pi \coloneqq &~ \argmin_{f \in \Fcal}\sup_{\text{admissible }\nu} \left\|f - \Tcal^\pi f\right\|_{2,\nu}^2
\\
g \coloneqq &~ \argmin_{g' \in \Fcal} \frac{1}{N} \sum_{(s,a,r,s') \in \Dcal} \left(g'(s,a) - r - \gamma f_\pi(s',\pi) \right)^2.
\end{align*}
Then, with high probability,
\begin{align*}
\|f_\pi - g\|_{2,\mu} \leq \Ocal \left(\Vmax \sqrt{\frac{ \log \frac{|\Ncal_\infty(\Fcal,\frac{\Vmax}{N})||\Ncal_{\infty,1}(\Pi,\frac{1}{N})|}{\delta}}{N}} + \sqrt{\varepsilon_{\Fcal}}\right).
\end{align*}
\end{lemma}

\begin{proof}[\bf\em Proof of Lemma~\ref{lem:msbosol}]
The proof of this lemma is obtained exactly the same as~\citet[Proof of Lemma A.5]{xie2021bellman}, we we only need to change the use of~\citet[Lemma A.4]{xie2021bellman} to \cref{lem:bernstein_general}. This completes the proof.
\end{proof}


We now ready to prove Theorem \ref{thm:version_space} and Theorem \ref{thm:mspo2be}. Note that the proofs of Theorem \ref{thm:version_space} and Theorem \ref{thm:mspo2be} follow similar approaches as the proof of~\citet[Theorem A.1, Theorem A.2]{xie2021bellman}, and we provide the full detailed proof here for completeness.

\begin{proof}[\bf\em Proof of Theorem~\ref{thm:version_space}]
This proof is obtained by  exactly the same strategy of~\citet[Proof of Theorem A.1]{xie2021bellman}, but we we change to change the corresponding lemmas to the new ones provided above. The correspondence of those lemmas are as follows:~\citep[Lemma A.4]{xie2021bellman} $\to$ \cref{lem:bernstein_general};~\citep[Lemma A.5]{xie2021bellman} $\to$ \cref{lem:msbosol}. This completes the proof.
\end{proof}


\begin{proof}[\bf\em Proof of Theorem~\ref{thm:mspo2be}]
This proof is obtained by the same exactly same strategy of~\citet[Proof of Theorem A.2]{xie2021bellman}, but we we change to change the corresponding lemmas to the new ones provided above. The correspondence of those lemmas are as follows:~\citep[Lemma A.4]{xie2021bellman} $\to$ \cref{lem:bernstein_general};~\citep[Lemma A.5]{xie2021bellman} $\to$ \cref{lem:msbosol}. This completes the proof.
\end{proof}


\subsection{Decomposition of Performance Difference}
\label{sec:decomp_pd}
This section proves \Eqref{th:performance_difference_main}.
We provide a more general version of \Eqref{th:performance_difference_main} with its proof as follows.
\begin{lemma}
\label{th:performance_difference}
Let $\picomp$ be an arbitrary competitor policy, $\pihat \in \Pi$ be some learned policy, and $f$ be an arbitrary function over $\Scal\times\Acal$.  Then we have,
\begin{align*}
&~ J(\picomp) - J(\pihat)
\\
= &~ \frac{1}{1 - \gamma} \left( \E_{\mu} \left[ \left(f - \Tcal^\pihat f\right)(s,a) \right] + \E_{\picomp} \left[ \left(\Tcal^\pihat f - f\right)(s,a) \right] + \E_{\picomp} \left[ f(s,\picomp) - f(s,\pihat) \right] + \Lcal_\mu(\pihat, f) - \Lcal_\mu(\pihat, Q^\pihat)\right).
\end{align*}
\end{lemma}
\begin{proof}[\bf\em Proof of \cref{th:performance_difference}]

Let $R^{f,\pihat}(s,a) \coloneqq f(s,a) - \gamma \E_{s'|s,a}[f(s',\pihat)]$ be a fake reward function given $f$ and $\pihat$. We use the subscript ``$(\cdot)_{R^{f,\pihat}}$'' to denote functions or operators under the true dynamics but the fake reward $R^{f,\pihat}$. 
Since $f(s,a) = (\Tcal^{\pi}_{R^{f,\pihat}} f)(s,a)$, we know $f \equiv Q^{\pi}_{R^{f,\pihat}}$.

We perform a performance decomposition:
\begin{align*}
    J(\picomp) - J(\pihat)
    = &~  \left( J(\picomp) - J(\mu) \right) - \left( J(\pihat)  - J(\mu) \right)
\end{align*}
and rewrite the second term as
\begin{align*}
    (1-\gamma) \left( J(\pihat)  - J(\mu) \right)
    = &~ \Lcal_\mu(\pihat,Q^\pihat)  \\
    = &~ \Delta(\pihat)  + \Lcal_\mu(\pihat, f)
    \tag{$\Delta(\pihat)  \coloneqq \Lcal_\mu(\pihat,Q^\pihat)  - \Lcal_\mu(\pihat, f)$}
    \\
    = &~ \Delta(\pihat)  + \E_{\mu} [ f(s,\pihat)  - f(s,a) ]
    \\
    = &~ \Delta(\pihat)  + (1 - \gamma) (J_{R^{f,\pihat}}(\pihat) - J_{R^{f,\pihat}}(\mu))
    \tag{by performance difference lemma~\citep{kakade2002approximately}}
    \\
    = &~ \Delta(\pihat)  + (1-\gamma)  Q^\pihat_{R^{f,\pihat}}(s_0,\pihat)  - \E_{\mu}[ R^{\pihat,f}(s,a)]
    \\
    = &~ \Delta(\pihat)  + (1-\gamma)  f(s_0,\pihat)  - \E_{\mu}[ R^{\pihat,f}(s,a)].
    \tag{by $f(\cdot,\cdot) \equiv Q^{\pi}_{R^{f,\pihat}}(\cdot,\cdot)$}
\end{align*}


Therefore,
\begin{align*}
(1 - \gamma)(J(\picomp) - J(\pihat) ) = \underbrace{(1-\gamma) \left( J(\picomp) - f(d_0,\pihat)  \right)}_{\text{(I)}}
  + \underbrace{\left( \E_{\mu}[ R^{\pihat,f}(s,a)] - (1-\gamma) J(\mu) \right)}_{\text{(II)}} - \Delta(\pihat).
\end{align*}

We first analyze (II). We can expand it by the definition of $R^{\pihat,f}$ as follows
\begin{align*}
\text{(II)} = &~ \E_{\mu}[ R^{\pihat,f}(s,a)] - (1-\gamma) J(\mu)
\\
= &~ \E_{\mu}[ R^{\pihat,f}(s,a)- R(s,a)]
\\
= &~ \E_{\mu}[ (f - \Tcal^\pihat f)(s,a)].
\end{align*}

We now write (I) as
\begin{align*}
\text{(I)} = &~ (1-\gamma) \left( J(\picomp) - f(s_0,\pihat)  \right)\\
     = &~ \underbrace{ (1-\gamma) J(\picomp) -  \E_{d^{\picomp}}[ R^{\pihat,f}(s,a) ]}_{\text{(Ia)}} + \underbrace{\E_{d^{\picomp}}[ R^{\pihat,f}(s,a) ] - (1-\gamma)f(s_0,\pihat) }_{\text{(Ib)}}.
\end{align*}

We analyze each term above in the following.
\begin{align*}
\text{(Ib)} = &~ \E_{d^{\picomp}}[ R^{\pihat,f}(s,a) ] - (1-\gamma)f(s_0,\pihat)
\\
= &~ \E_{d^{\picomp}}[ f(s,\picomp) - f(s,\pihat) ].
\end{align*}
On the other hand, we can write
\begin{align*}
\text{(Ia)} = &~ (1-\gamma) J(\picomp) -  \E_{d^{\picomp}}[ R^{\pihat,f} (s,a) ]
\\
= &~ \E_{d^{\picomp}}[  R(s,a) -  R^{\pihat,f}(s,a)]
\\
= &~ \E_{d^{\picomp}}[ (\Tcal^\pihat f - f)(s,a)].
\end{align*}

Combine them all, we have
\begin{align*}
&~ J(\picomp) - J(\pihat)
\\
= &~ \frac{1}{1 - \gamma} \left( \text{(Ia)} + \text{(Ib)} + \text{(II)} - \Delta(\pihat) \right)
\\
= &~ \frac{1}{1 - \gamma} \left( \E_{\mu} \left[ \left(f - \Tcal^\pihat f\right)(s,a) \right] + \E_{\picomp} \left[ \left(\Tcal^\pihat f - f\right)(s,a) \right] + \E_{\picomp} \left[ f(s,\picomp) - f(s,\pihat)  \right] + \Lcal_\mu(\pihat, f) - \Lcal_\mu(\pihat, Q^\pihat) \right).
\end{align*}
This completes the proof. 
\end{proof}

We now prove a general version of \Eqref{th:performance_difference_main} using \cref{th:performance_difference}, which takes into account the approximation errors in the realizability and completeness assumptions (\cref{asm:relz2} and \cref{asm:completeness}).

\begin{lemma}[General Version of \Eqref{th:performance_difference_main}]
\label{th:performance_difference_app}
Let $\picomp$ be an arbitrary competitor policy. Also let $\pi_k$ and $f_k$ be obtained by Algorithm~\ref{alg:atac (theory)} for $k \in [K]$. Then with high probability, for any $k\in[K]$,
\begin{align*}
&~ (1 - \gamma) \left(J(\picomp) - J(\pi_k)\right) \leq \E_{\mu} \left[ f_k - \Tcal^{\pi_k} f_k \right] + \E_{\picomp} \left[ \Tcal^{\pi_k} f_k - f_k \right] +  \E_{\picomp} \left[ f_k(s,\picomp) - f_k(s,\pi_k) \right]
\\
&~ \qquad\quad + \Ocal \left( \Vmax \sqrt{\frac{ \log \nicefrac{|\Ncal_\infty(\Fcal,\nicefrac{\Vmax}{N})||\Ncal_{\infty,1}(\Pi,\nicefrac{1}{N})|}{\delta}}{N}} + \sqrt{\varepsilon_\Fcal}\right) + \beta \cdot \Ocal\left(\frac{\Vmax^2 \log \nicefrac{|\Ncal_\infty(\Fcal,\nicefrac{\Vmax}{N})||\Ncal_{\infty,1}(\Pi,\nicefrac{1}{N})|}{\delta}}{N} + \varepsilon_{\Fcal}\right).
\end{align*}
\end{lemma}
\begin{proof}[\bf\em Proof of \cref{th:performance_difference_app}]
By \cref{th:performance_difference}, we have
\begin{align*}
J(\picomp) - J(\pi_k) = \frac{\E_{\mu} \left[ f_k - \Tcal^{\pi_k} f_k \right]}{1 - \gamma} + \frac{\E_{\picomp} \left[ \Tcal^{\pi_k} f_k - f_k \right]}{1 - \gamma} +  \frac{\E_{\picomp} \left[ f_k(s,\picomp) - f_k(s,\pi_k) \right]}{1 - \gamma} + \frac{\Lcal_\mu(\pi_k, f_k) - \Lcal_\mu(\pi_k, Q^{\pi_k})}{1 - \gamma}.
\end{align*}
We now bound the term of $\Lcal_\mu(\pi_k, f_k) - \Lcal_\mu(\pi_k, Q^{\pi_k})$.
\begin{align*}
f_\pi \coloneqq &~ \argmin_{f \in \Fcal}\sup_{\text{admissible }\nu} \left\|f - \Tcal^\pi f\right\|_{2,\nu}^2, ~\forall \pi \in \Pi
\\
\varepsilon_\stat \coloneqq &~ \Ocal \left( \frac{\Vmax^2 \log \nicefrac{|\Ncal_\infty(\Fcal,\nicefrac{\Vmax}{N})||\Ncal_{\infty,1}(\Pi,\nicefrac{1}{N})|}{\delta}}{N} \right),
\\
\varepsilon_r \coloneqq &~ \varepsilon_\stat + \Ocal \left( \varepsilon_{\Fcal} \right).
\end{align*}
Then, by \cref{thm:version_space}, we know that with high probability, for any $k \in [K]$,
\begin{align}
\label{eq:bound_fpit}
\Ecal_\Dcal(\pi_k, f_{\pi_k}) \leq \varepsilon_r.
\end{align}
For $| \Lcal_\mu(\pi_k, Q^{\pi_k}) - \Lcal_\mu(\pi_k, f_{\pi_k}) |$, we have,
\begin{align}
\Lcal_\mu(\pi_k, Q^{\pi_k}) = &~ \E_\mu \left[Q^{\pi_k}(s,\pi_k) - Q^{\pi_k}(s,a) \right]
\nonumber
\\
= &~ (1 - \gamma) \left( J(\pi_k) - J(\mu) \right)
\nonumber
\\
= &~ (1 - \gamma) \left( f_{\pi_k}(s_0,\pi_k) - J(\mu) \right) + (1 - \gamma) \left(J(\pi_k) - f_{\pi_k}(s_0,\pi_k) \right)
\nonumber
\\
= &~ \E_\mu \left[ f_{\pi_k}(s,\pi_k) - (\Tcal^{\pi_k} f_{\pi_k})(s,a) \right] + \E_{d^{\pi_k}} \left[ (\Tcal^{\pi_k} f_{\pi_k})(s,a) - f_{\pi_k}(s,a) \right]
\tag{by the extension of performance difference lemma~\citep[see, e.g.,][Lemma 1]{cheng2020policy}}
\\
= &~ \Lcal_\mu(\pi_k,f_{\pi_k}) + \E_\mu \left[ f_{\pi_k}(s,a) - (\Tcal^{\pi_k} f_{\pi_k})(s,a) \right] + \E_{d^{\pi_k}} \left[ (\Tcal^{\pi_k} f_{\pi_k})(s,a) - f_{\pi_k}(s,a) \right]
\nonumber
\\
\Longrightarrow | \Lcal_\mu(\pi_k, Q^{\pi_k}) - \Lcal_\mu(\pi_k, f_{\pi_k}) | \leq &~ \| f_{\pi_k} - \Tcal^{\pi_k} f_{\pi_k}\|_{2,\mu} + \| \Tcal^{\pi_k} f_{\pi_k} - f_{\pi_k}\|_{2,d^{\pi_k}}
\nonumber
\\
\leq &~ \Ocal(\sqrt{\varepsilon_\Fcal}),
\label{eq:bd_lmu_fpit}
\end{align}
where the last step is by \cref{asm:relz2}.
%
Also, by applying standard concentration inequalities on $\Lcal_\Dcal$ (the failure probability will be split evenly with that on $\Ecal_{\Dcal}$ from Lemma~\ref{lem:bernstein_general}):
\begin{align}
\label{eq:bd_Lmu_ft}
\left|\Lcal_\mu(\pi_k, f_k) - \Lcal_\Dcal(\pi_k, f_k)\right| + \left|\Lcal_\mu(\pi_k, f_{\pi_k}) - \Lcal_\Dcal(\pi_k, f_{\pi_k})\right| \leq \sqrt{\varepsilon_\stat}, ~\forall k \in [K].
\end{align}
Therefore, 
\begin{align*}
&~ \Lcal_\mu(\pi_k, f_k) - \Lcal_\mu(\pi_k, Q^{\pi_k})
\\
\leq &~ \Lcal_\mu(\pi_k, f_k) + \beta \Ecal_\Dcal(\pi_k, f_k) - \Lcal_\mu(\pi_k, Q^{\pi_k}) \tag{$\Ecal_\Dcal(\cdot)\ge 0$}
\\
\leq &~ \Lcal_\mu(\pi_k, f_k) + \beta \Ecal_\Dcal(\pi_k, f_k) - \Lcal_\mu(\pi_k, f_{\pi_k}) - \beta \Ecal_\Dcal(\pi_k, f_{\pi_k}) + \Ocal(\sqrt{\epsilon_\Fcal}) + \beta \varepsilon_r
\tag{by \Eqref{eq:bound_fpit} and \Eqref{eq:bd_lmu_fpit}}
\\
\leq &~ \Lcal_\Dcal(\pi_k, f_k) + \beta \Ecal_\Dcal(\pi_k, f_k) - \Lcal_\Dcal(\pi_k, f_{\pi_k}) - \beta \Ecal_\Dcal(\pi_k, f_{\pi_k})
\\
&~ + \Ocal(\sqrt{\epsilon_\Fcal}) + \sqrt{\varepsilon_\stat} + \beta \cdot \Ocal\left(\varepsilon_\stat + \varepsilon_{\Fcal}\right)
\tag{by \Eqref{eq:bd_Lmu_ft}}
\\
\leq &~ \Ocal\left( \sqrt{\varepsilon_\Fcal} \right) + \sqrt{\varepsilon_\stat} + \beta \cdot \Ocal\left(\varepsilon_\stat + \varepsilon_{\Fcal}\right)
\tag{by the optimality of $f_k$}
\\
\leq &~ \Ocal \left(\sqrt{\varepsilon_\Fcal} + \Vmax \sqrt{\frac{ \log \nicefrac{|\Ncal_\infty(\Fcal,\nicefrac{\Vmax}{N})||\Ncal_{\infty,1}(\Pi,\nicefrac{1}{N})|}{\delta}}{N}} \right) + \beta \cdot \Ocal\left(\frac{\Vmax^2 \log \nicefrac{|\Ncal_\infty(\Fcal,\nicefrac{\Vmax}{N})||\Ncal_{\infty,1}(\Pi,\nicefrac{1}{N})|}{\delta}}{N} + \varepsilon_{\Fcal}\right).
\end{align*}
This completes the proof.
\end{proof}

\subsection{Performance Guarantee of the Theoretical Algorithm}
\label{app:main_thm_proof}

This section proves a general version of \cref{thm:maintext_thm} using \cref{th:performance_difference}, which relies on the approximate realizability and completeness assumptions (\cref{asm:relz2} and \cref{asm:completeness}).

\begin{theorem}[General Version of \cref{thm:maintext_thm}]
\label{thm:main_thm}
Under the same condition as \cref{thm:maintext_thm}, let $C > 0$ be any constant, $\nu$ be an arbitrarily distribution that satisfies $\Cscr(\nu;\mu,\Fcal,\pi_k) \leq C$, $\varepsilon_\stat \coloneqq \Ocal \left( \frac{\Vmax^2 \log \nicefrac{|\Ncal_\infty(\Fcal,\nicefrac{\Vmax}{N})||\Ncal_{\infty,1}(\Pi,\nicefrac{1}{N})|}{\delta}}{N} \right)$, and $\picomp$ be an arbitrary competitor policy. Then, we choose $\beta = \Ocal\left(\frac{\Vmax^{1/3}}{(\varepsilon_\Fcal + \varepsilon_\stat)^{2/3}}\right)$ and with probability at least $1-\delta$,
\begin{align*}
&~ J(\picomp) - J(\pibar)
\\
\leq &~ \Ocal\left( \frac{\sqrt{C}\left(\sqrt{\varepsilon_\Fcal} + \sqrt{\varepsilon_{\Fcal,\Fcal}} + \sqrt{\varepsilon_\stat} + (\Vmax\varepsilon_\Fcal + \Vmax\varepsilon_\stat)^{\nicefrac{1}{3}}\right) }{1 - \gamma}
\right) + \frac{\left\langle d^{\pi}\setminus\nu, ~ f_k - \Tcal^{\pi_k} f_k \right\rangle}{1 - \gamma}.
\end{align*}
\end{theorem}
\begin{proof}[\pfname{Theorem}{\ref{thm:main_thm}}]
Over this proof, let
\begin{align*}
\varepsilon_\stat \coloneqq &~ \Ocal \left( \frac{\Vmax^2 \log \nicefrac{|\Ncal_\infty(\Fcal,\nicefrac{\Vmax}{N})||\Ncal_{\infty,1}(\Pi,\nicefrac{1}{N})|}{\delta}}{N} \right).
\end{align*}
By the definition of $\pibar$, we have
\begin{align*}
&~ J(\picomp) - J(\pibar)
\\
= &~ \frac{1}{K} \sum_{k = 1}^{K} \left( J(\picomp) - J(\pi_k) \right)
\\
\leq &~ \frac{1}{K} \sum_{k = 1}^{K} \Bigg( \underbrace{\frac{\E_{\mu} \left[ f_k - \Tcal^{\pi_k} f_k \right]}{1 - \gamma}}_{\text{(I)}} + \underbrace{\frac{\E_{\picomp} \left[ \Tcal^{\pi_k} f_k - f_k \right]}{1 - \gamma}}_{\text{(II)}} +  \underbrace{\frac{\E_{\picomp} \left[ f_k(s,\picomp) - f_k(s,\pi_k) \right]}{1 - \gamma}}_{\text{(III)}} + \sqrt{\varepsilon_\Fcal} + \sqrt{\varepsilon_\stat} + \beta \cdot \Ocal(\varepsilon_\Fcal + \varepsilon_\stat) \Bigg).
\tag{by \cref{th:performance_difference_app}}
\end{align*}

By the same argument of~\citet[Proof of Theorem 4.1]{xie2021bellman}, we know for any $k \in [K]$,
\begin{align*}
\text{(I)} \leq \frac{\sqrt{\varepsilon_b} + \sqrt{V_{\max}/\beta}}{1 - \gamma}
\tag{$\varepsilon_b$ is defined in \cref{eq:upbdmsbe}}
\end{align*}
and
\begin{align*}
\text{(II)} \leq \frac{2 \sqrt{C}(\sqrt{\varepsilon_b} + \sqrt{V_{\max}/\beta}) }{1 - \gamma} + \frac{\left\langle d^{\pi}\setminus\nu, ~ f_k - \Tcal^{\pi_k} f_k \right\rangle}{1 - \gamma},
\end{align*}
where $C\geq 1$ can be selected arbitrarily and $\nu$ is an arbitrarily distribution that satisfies $\Cscr(\nu;\mu,\Fcal,\pi_k) \leq C$.

Also, using the property of the no-regret oracle, we have
\begin{align*}
\frac{1}{K} \sum_{k = 1}^{K} \text{(III)} = o(1) 
\end{align*}

Note that $\sqrt{\varepsilon_b} = \Ocal(\sqrt{\varepsilon_\Fcal} + \sqrt{\varepsilon_{\Fcal,\Fcal}} + \sqrt{\varepsilon_\stat} )$. Then, combine them all, we obtain,
\begin{align*}
&~ J(\picomp) - J(\pibar)
\\
\leq &~ \Ocal\left( \frac{\sqrt{C}\left(\sqrt{\varepsilon_\Fcal} + \sqrt{\varepsilon_{\Fcal,\Fcal}} + \sqrt{\varepsilon_\stat} + \sqrt{V_{\max}/\beta}\right) }{1 - \gamma} 
+ \beta(\varepsilon_\Fcal + \varepsilon_\stat)\right) + \frac{1}{K}\sum_{k = 1}^{K} \frac{\left\langle d^{\pi}\setminus\nu, ~ f_k - \Tcal^{\pi_k} f_k \right\rangle}{1 - \gamma}.
\end{align*}
Therefore, we choose $\beta = \Theta\left(\frac{\Vmax^{1/3}}{(\varepsilon_\Fcal + \varepsilon_\stat)^{2/3}}\right)$, and obtain
\begin{align*}
&~ J(\picomp) - J(\pibar)
\\
\leq &~ \Ocal\left( \frac{\sqrt{C}\left(\sqrt{\varepsilon_\Fcal} + \sqrt{\varepsilon_{\Fcal,\Fcal}} + \sqrt{\varepsilon_\stat} + (\Vmax\varepsilon_\Fcal + \Vmax\varepsilon_\stat)^{\nicefrac{1}{3}}\right) }{1 - \gamma}
\right) +
\frac{1}{K}\sum_{k = 1}^{K} \frac{\left\langle d^{\picomp}\setminus\nu, ~ f_k - \Tcal^{\pi_k} f_k \right\rangle}{1 - \gamma}.
\end{align*}
This completes the proof.
\end{proof}

\section{Experiment Details} \label{sec:exp details}


\subsection{Implementation Details}

\begin{algorithm}[t]
\caption{ATAC (Detailed Practical Version)}
\label{alg:atac (detailed practice) }
{\bfseries Input:} Batch data $\Dcal$, policy $\pi$, critics $f_1, f_2$, constants $\beta\geq 0$, $\tau\in[0,1]$, $w\in[0,1]$, entropy lower bound $ \textrm{Entropy}_{\min}$
\begin{algorithmic}[1]
\State Initialize target networks $\bar{f}_1\gets f_1$, $\bar{f}_2\gets f_2$
\State Initialize Lagrange multiplier $\alpha \gets 1$
\For{$k = 1,2,\dotsc,K$}
\State Sample minibatch $\Dcal_{\textrm{mini}}$ from dataset $\Dcal$.

\State For {$f\in\{f_1, f_2\}$}, update critic networks 
\Statex\qquad$l_{\textrm{critic}}(f) \coloneqq
    \Lcal_{\Dcal_{\textrm{mini}}}(f,\pi) + \beta \Ecal_{\Dcal_{\textrm{mini}}}^{w}(f, \pi)$
\Statex \qquad  \textcolor{ForestGreen}{  \# $ f \gets \text{Proj}_\Fcal(f - \eta_{\textrm{fast}}\nabla l_{\textrm{critic}})$}
\Statex \qquad $ f \gets \text{ADAM}(f , \nabla l_{\textrm{critic}}, \eta_{\textrm{fast}})$
\Statex \qquad $ f \gets \text{ClipWeightL2}(f)$ 

\State Update actor network 

\Statex \qquad \textcolor{ForestGreen}{ \# $l_{\textrm{actor}}(\pi)   \coloneqq - \Lcal_{\Dcal_{\textrm{mini}}}(f_1,\pi)$}
\Statex \qquad   \textcolor{ForestGreen}{\# $\pi \gets \text{Proj}_{\Pi}(\pi - \eta_{\textrm{slow}}\nabla l_{\textrm{actor}})$}

\Statex \qquad  $ \tilde{l}_{\textrm{actor}}(\pi, \alpha) = - \Lcal_{\Dcal_{\textrm{mini}}}(f_1,\pi) - \alpha ( \E_{\Dcal_{\textrm{mini}}}[ \pi \log \pi] + \textrm{Entropy}_{\min} )$

\Statex \qquad $\pi \gets \text{ADAM}(\pi , \nabla_\pi \tilde{l}_{\textrm{actor}}, \eta_{\textrm{slow}})$
\Statex \qquad $\alpha \gets \text{ADAM}(\alpha , - \nabla_\alpha  \tilde{l}_{\textrm{actor}}, \eta_{\textrm{fast}})$
\Statex \qquad $\alpha \gets \max\{0, \alpha\}$

\State For $(f,\bar{f})\in\{(f_i,\bar{f}_i)\}_{i=1,2}$, update target networks 
\Statex \qquad\qquad$\bar{f} \gets (1-\tau) \bar{f} + \tau{f}  $.
\EndFor
\end{algorithmic}
\end{algorithm}

We provide a more detailed version of our practical algorithm \cref{alg:atac (practice)} in \cref{alg:atac (detailed practice) }, which shows how the actor and critic updates are done with ADAM. As mentioned in \cref{sec:practical algo}, the projection in $\pi \gets \text{Proj}_{\Pi}(\pi - \eta_{\textrm{slow}}\nabla l_{\textrm{actor}})$ of the pseudo-code \cref{alg:atac (practice)} is done by a further Lagrange relaxation through introducing a Lagrange multiplier $\alpha\geq 0$. We update $\alpha$ in the fast timescale $\eta_{\textrm{fast}}$, so the policy entropy $\E_\Dcal[-\pi\log\pi]$ can be maintained above a threshold $\textrm{Entropy}_{\min}$, roughly following the path of the projected update in \cref{ln:update actor} in the pseudo code in \cref{alg:atac (practice)}.
$\textrm{Entropy}_{\min}$ is set based on the heuristic used in SAC~\citep{haarnoja2018soft}.

In implementation, we use separate 3-layer fully connected neural networks to realize the policy and the critics, where each hidden layer has 256 neurons and ReLU activation and the output layer is linear. The policy is Gaussian, with the mean and the standard deviation predicted by the neural network. We impose an $l_2$ norm constraint of 100 for the weight (not the bias) in each layer of the critic networks.

The first-order optimization is implemented by ADAM~\citep{kingma2014adam} with a minibatch size $|\Dcal_{\textrm{mini}}|=256$, and the two-timescale stepsizes are set as $\eta_{\textrm{fast}} = 0.0005$ and $\eta_{\textrm{slow}} = 10^{-3} \eta_{\textrm{fast}}$. These stepsizes $\eta_{\textrm{fast}}$ and $\eta_{\textrm{slow}}$ were selected \emph{offline} with a heuristic:  Since \algo with $\beta=0$ is IPM-IL, we did a grid search (over $\eta_{\textrm{fast}} \in \{5e-4 , 5e-5, 5e-6\}$ and $\eta_{\textrm{slow}} = \{5e-5, 5e-6 , 5e-7\}$, on the \textit{hopper-medium} and \textit{hopper-expert} datasets) and selected the combination that attains the lowest $\ell_2$ IL error after 100 epochs.

We set $w=0.5$ in \Eqref{eq:DQRA loss}, as we show in the ablation (\cref{fig:extra dqra ablation}) that either $w=0$ and $w=1$ leads to bad numerical stability and/or policy performance. We use $\tau=0.005$ for target network update from the work of~\citet{haarnoja2018soft}. The discount is set to the common $\gamma=0.99$.

The regularization coefficient $\beta$ is our only hyperparameter that varies across datasets based on an online selection. We consider $\beta$ in $ \in \{0, 4^{-4}, 4^{-3}, 4^{-2}, 4^{-1}, 1, 4, 4^{2}, 4^{3}, 4^{4}\}$.
For each $\beta$, we perform \algo training with 10 different seeds: for each seed, we run 100 epochs of BC for warm start and 900 epochs of \algo, where 1 epoch denotes 2K gradient updates. During the warmstart, the critics are optimized to minimize the Bellman surrogate $\Ecal_{\Dcal_{\textrm{mini}}}^{w}(f, \pi)$ except for $\beta=0$.

Since \algo does not have guarantees on last-iterate convergence, we report also the results of both the last iterate (denoted as \algo and ATAC$_0$) and the best checkpoint (denoted as ATAC$^*$ and ATAC$_0^*$) selected among $9$ checkpoints (each was made every 100 epochs).

We argue that the online selection of $\beta$ and few checkpoints are reasonable for \algo, as \algo theory provides robust policy improvement guarantees. While the assumptions made in the theoretical analysis does not necessarily apply to the practical version of \algo, empirically we found that \algo does demonstrate robust policy improvement properties in the D4RL benchmarks that we experimented with, which we will discuss more below.

\subsection{Detailed Experimental Results}

\begin{table}[hbtp]
\scriptsize
    \centering
    \begin{adjustbox}{angle=90}
        \begin{tabular}{|c |c |ccc |cc c| cc c| ccc|}
         \hline
 & Behavior & ATAC$^*$ & CI & $\beta$ & ATAC & CI & $\beta$ & ATAC$_0^*$ & CI & $\beta$ & ATAC$_0$ & CI & $\beta$\\
\hline
halfcheetah-rand & -0.1 & 4.8 & [-0.5, 0.5] & 16.0 & 3.9 & [-1.2, 0.5] & 4.0 & 2.3 & [-0.0, 0.3] & 64.0 & 2.3 & [-0.1, 0.3] & 64.0\\
walker2d-rand & 0.0 & 8.0 & [-0.9, 0.4] & 64.0 & 6.8 & [-0.4, 1.1] & 4.0 & 7.6 & [-0.2, 0.3] & 16.0 & 5.7 & [-0.1, 0.1] & 16.0\\
hopper-rand & 1.2 & 31.8 & [-7.2, 0.1] & 64.0 & 17.5 & [-11.2, 13.2] & 16.0 & 31.6 & [-0.9, 0.6] & 64.0 & 18.2 & [-14.0, 12.7] & 16.0\\
halfcheetah-med & 40.6 & 54.3 & [-0.8, 0.2] & 4.0 & 53.3 & [-0.4, 0.1] & 4.0 & 43.9 & [-4.9, 0.6] & 16.0 & 36.8 & [-1.4, 1.2] & 0.0\\
walker2d-med & 62.0 & 91.0 & [-0.5, 0.3] & 64.0 & 89.6 & [-0.2, 0.2] & 64.0 & 90.5 & [-0.4, 0.8] & 64.0 & 89.6 & [-1.2, 0.3] & 64.0\\
hopper-med & 44.2 & 102.8 & [-0.5, 0.9] & 64.0 & 85.6 & [-7.2, 6.6] & 16.0 & 103.5 & [-0.8, 0.2] & 16.0 & 94.8 & [-11.4, 4.9] & 4.0\\
halfcheetah-med-replay & 27.1 & 49.5 & [-0.3, 0.1] & 16.0 & 48.0 & [-0.6, 0.2] & 64.0 & 49.2 & [-0.3, 0.4] & 64.0 & 47.2 & [-0.2, 0.4] & 64.0\\
walker2d-med-replay & 14.8 & 94.1 & [-0.2, 0.4] & 64.0 & 92.5 & [-4.5, 1.0] & 16.0 & 94.2 & [-1.2, 1.1] & 64.0 & 89.8 & [-4.1, 2.0] & 16.0\\
hopper-med-replay & 14.9 & 102.8 & [-0.3, 0.3] & 16.0 & 102.5 & [-0.4, 0.2] & 16.0 & 102.7 & [-1.0, 1.0] & 1.0 & 102.1 & [-0.3, 0.5] & 16.0\\
halfcheetah-med-exp & 64.3 & 95.5 & [-0.2, 0.1] & 0.062 & 94.8 & [-0.4, 0.1] & 0.062 & 41.6 & [-0.1, 2.4] & 0.0 & 39.7 & [-1.3, 1.7] & 0.0\\
walker2d-med-exp & 82.6 & 116.3 & [-0.7, 0.5] & 64.0 & 114.2 & [-7.4, 0.7] & 16.0 & 114.5 & [-1.5, 0.8] & 64.0 & 104.9 & [-8.1, 8.0] & 64.0\\
hopper-med-exp & 64.7 & 112.6 & [-0.3, 0.2] & 1.0 & 111.9 & [-0.3, 0.3] & 1.0 & 83.0 & [-18.0, 12.8] & 1.0 & 46.5 & [-16.6, 15.2] & 0.0\\
\hline
pen-human & 207.8 & 79.3 & [-14.2, 16.5] & 0.004 & 53.1 & [-37.2, 21.0] & 0.004 & 106.1 & [-32.2, 7.3] & 0.004 & 61.7 & [-7.0, 27.6] & 0.004\\
hammer-human & 25.4 & 6.7 & [-3.6, 8.2] & 0.016 & 1.5 & [-0.2, 0.1] & 64.0 & 3.8 & [-1.9, 0.9] & 4.0 & 1.2 & [-0.3, 0.7] & 64.0\\
door-human & 28.6 & 8.7 & [-2.1, 0.1] & 0.0 & 2.5 & [-1.9, 1.5] & 0.0 & 12.2 & [-3.6, 6.7] & 16.0 & 7.4 & [-7.2, 1.3] & 16.0\\
relocate-human & 86.1 & 0.3 & [-0.2, 0.7] & 0.25 & 0.1 & [-0.1, 0.1] & 64.0 & 0.5 & [-0.2, 1.2] & 4.0 & 0.1 & [-0.0, 0.0] & 1.0\\
pen-cloned & 107.7 & 73.9 & [-5.2, 5.7] & 0.0 & 43.7 & [-28.2, 24.4] & 0.0 & 104.9 & [-13.5, 13.0] & 0.016 & 68.9 & [-49.0, 16.4] & 0.062\\
hammer-cloned & 8.1 & 2.3 & [-0.2, 3.3] & 16.0 & 1.1 & [-0.4, 0.2] & 0.016 & 3.2 & [-1.6, 0.6] & 4.0 & 0.4 & [-0.2, 0.2] & 0.25\\
door-cloned & 12.1 & 8.2 & [-4.5, 5.1] & 0.062 & 3.7 & [-2.8, 1.0] & 0.016 & 6.0 & [-5.5, 1.1] & 4.0 & 0.0 & [-0.0, 0.0] & 0.0\\
relocate-cloned & 28.7 & 0.8 & [-0.6, 0.8] & 0.062 & 0.2 & [-0.1, 0.1] & 0.004 & 0.3 & [-0.1, 0.8] & 16.0 & 0.0 & [-0.0, 0.0] & 4.0\\
pen-exp & 105.7 & 159.5 & [-8.4, 1.8] & 1.0 & 136.2 & [-5.4, 18.7] & 0.062 & 154.4 & [-4.1, 4.0] & 4.0 & 97.7 & [-66.7, 8.8] & 0.062\\
hammer-exp & 96.3 & 128.4 & [-0.5, 0.2] & 0.016 & 126.9 & [-0.5, 0.3] & 0.016 & 118.3 & [-20.3, 9.9] & 0.062 & 99.2 & [-41.0, 5.6] & 4.0\\
door-exp & 100.5 & 105.5 & [-1.6, 0.3] & 0.0 & 99.3 & [-24.1, 4.6] & 0.25 & 103.6 & [-6.5, 0.8] & 64.0 & 48.3 & [-39.3, 31.2] & 64.0\\
relocate-exp & 101.6 & 106.5 & [-1.5, 1.0] & 0.016 & 99.4 & [-12.2, 1.9] & 0.004 & 104.0 & [-2.7, 3.1] & 64.0 & 74.3 & [-1.4, 7.1] & 16.0\\
         \hline
        \end{tabular}
        \end{adjustbox}

    \caption{\small Experimental results of \algo and ATAC$_0$ on the D4RL dataset and its confidence interval. We report the median score and the $25^{th}$ and $75^{th}$ percentiles, over $10$ random seeds.}
    \label{tab:ATAC with CI}

\end{table}

\begin{figure}[t!]
	\centering
	\begin{subfigure}{0.23\textwidth}
		\includegraphics[width=\textwidth]{figures/hopper-medium-replay-v2_rg_ablation.png}
	\end{subfigure}
	\begin{subfigure}{0.23\textwidth}
		\includegraphics[width=\textwidth]{figures/hopper-medium-replay-v2_rg_ablation_td_error.png}
	\end{subfigure}
	\\
	\begin{subfigure}{0.23\textwidth}
	\includegraphics[width=\textwidth]{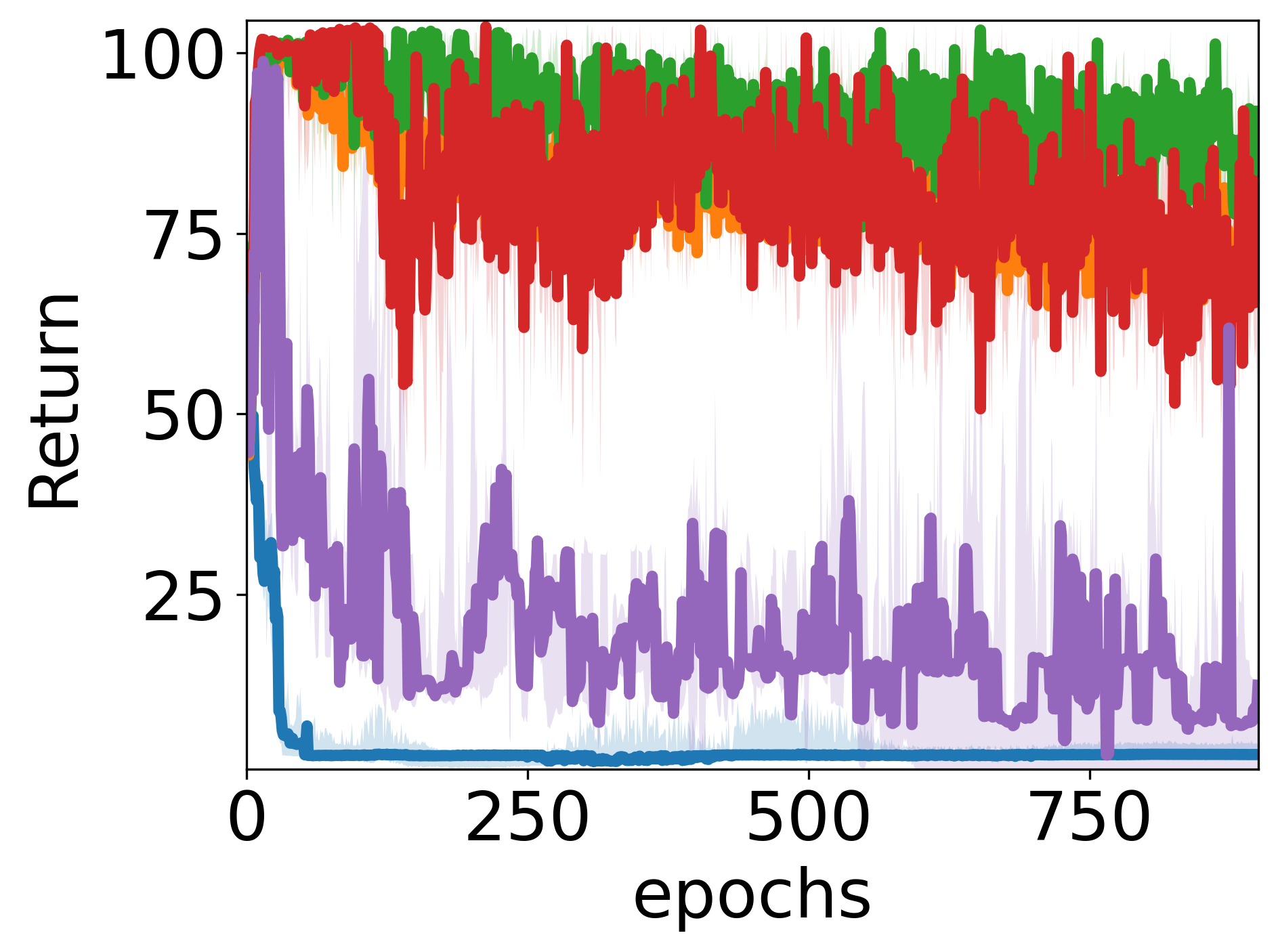}
    \end{subfigure}
    \begin{subfigure}{0.23\textwidth}
    	\includegraphics[width=\textwidth]{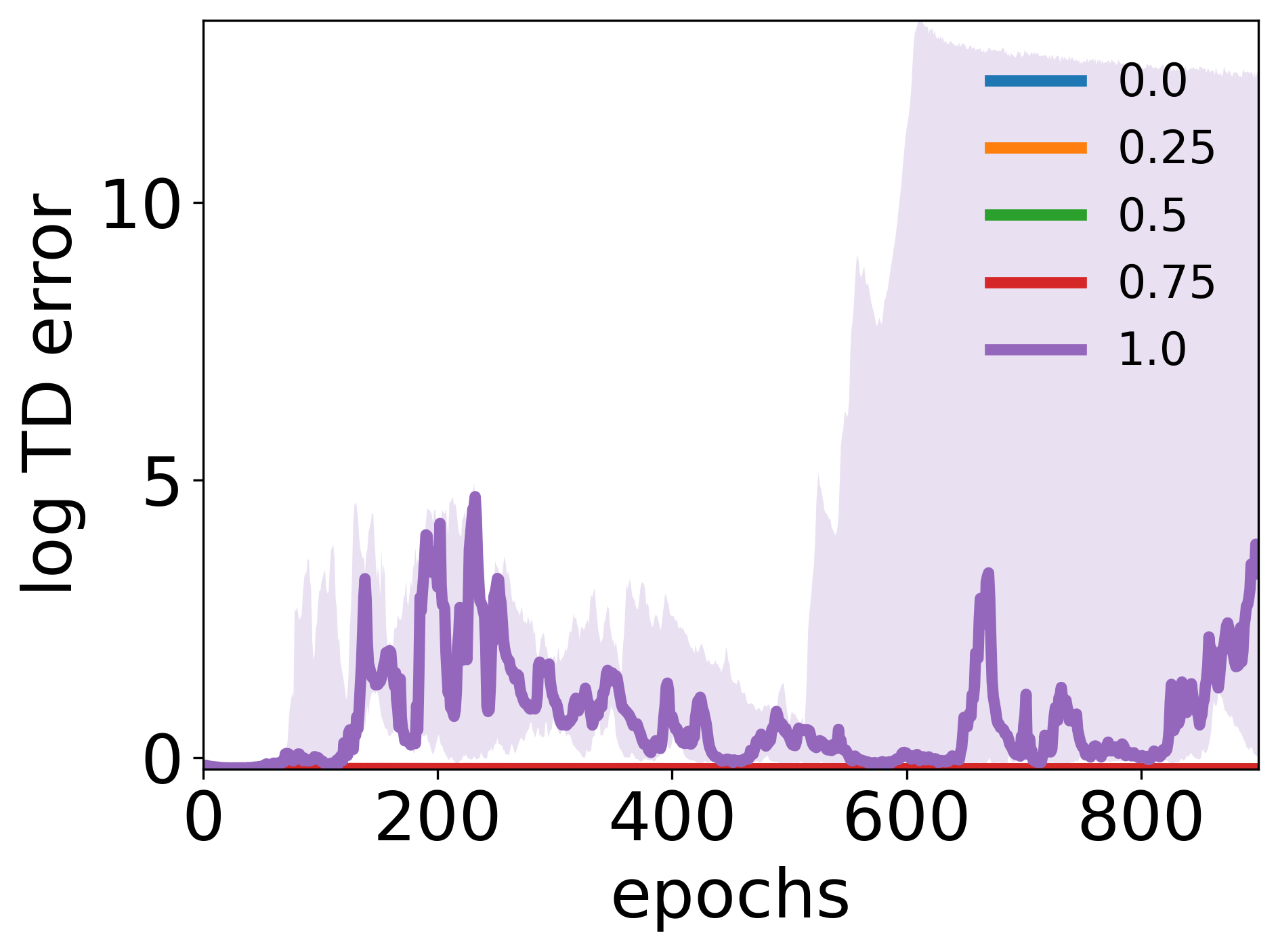}
    \end{subfigure}
    \\
    \begin{subfigure}{0.23\textwidth}
		\includegraphics[width=\textwidth]{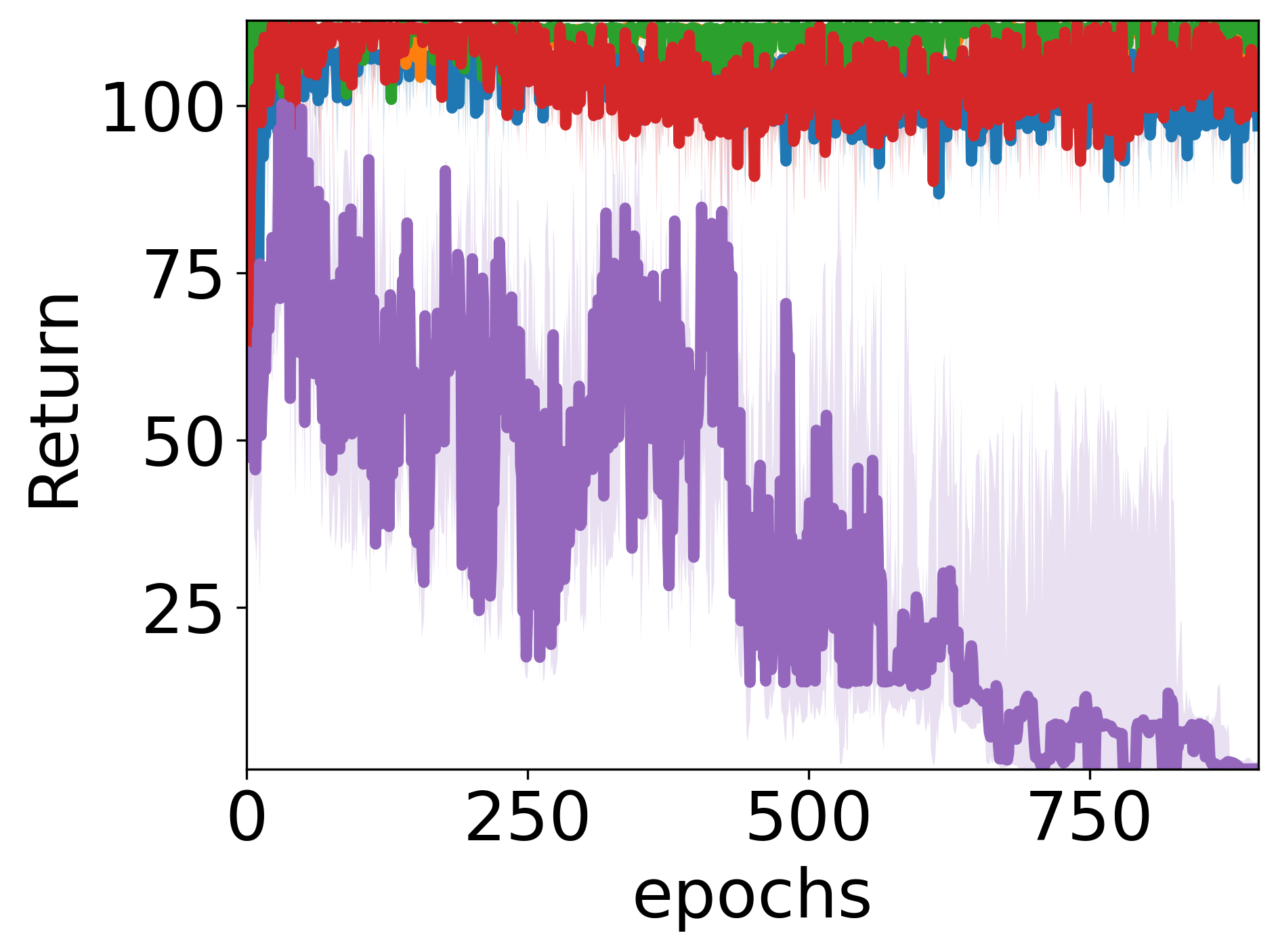}
	\end{subfigure}
	\begin{subfigure}{0.23\textwidth}
		\includegraphics[width=\textwidth]{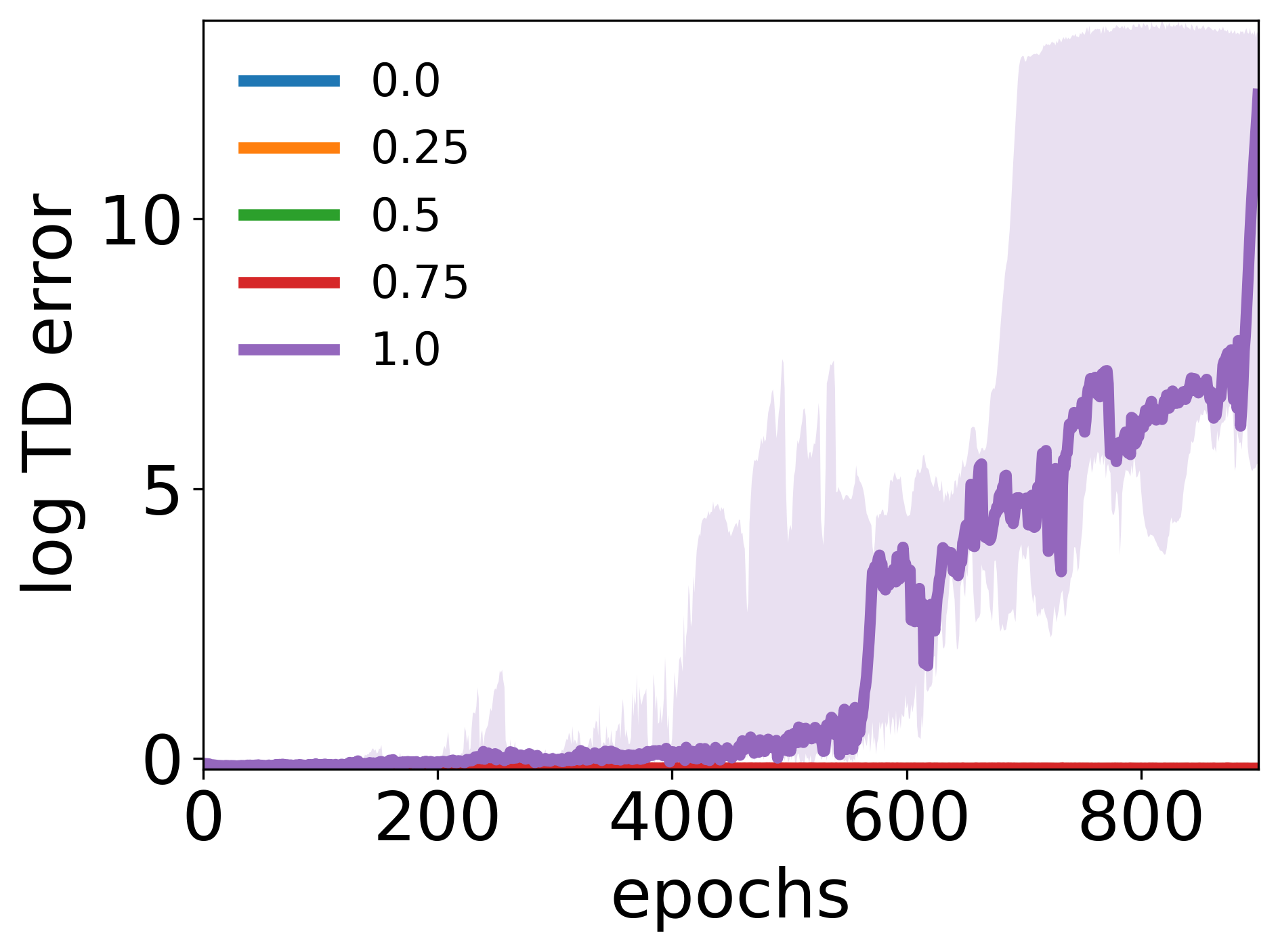}
	\end{subfigure}
	\caption{\small{Ablation of the DQRA loss with different mixing weights $w$ in \Eqref{eq:DQRA loss}. The plots show the policy performance and TD error across optimization epochs of \algo with the \textit{hopper-medium-replay}, \textit{hopper-medium}, and \textit{hopper-medium-expert} datasets from top to buttom. The stability and performance are greatly improved when $w\in(0,1)$.
For each $w$, the plot shows the $25^{th}$, $50^{th}$, $75^{th}$ percentiles over 10 random seeds.}}
\label{fig:extra dqra ablation}
\vspace{-2mm}
\end{figure}

\begin{table}[t!]
\footnotesize
    \centering
        \begin{tabular}{|c |  c c | c c |}
         \hline
   & ATAC$^*$ & ATAC & CQL & TD3+BC \\ 
\hline
halfcheetah-rand  & 2.3 & 2.3 & \textbf{35.4} & 10.2 \\ 
walker2d-rand  & \textbf{8.2} & 6.5 & 7.0 & 1.4 \\ 
\color{lightgray}hopper-rand & \color{lightgray}\textbf{12.1} & \color{lightgray}\textbf{12.0} &\color{lightgray} 10.8 & \color{lightgray}11.0 \\ 
 halfcheetah-med  & \textbf{42.9} &  \textbf{42.6} &  \textbf{44.4} &\textbf{42.8}\\ 
walker2d-med  & \textbf{84.0} & \textbf{83.0} & 74.5 & 79.7 \\ 
\color{lightgray} hopper-med  & \color{lightgray} 53.3 & \color{lightgray} 33.5 & \color{lightgray} 86.6 & \color{lightgray} \textbf{99.5} \\ 
halfcheetah-med-replay & 43.3 & 41.7 & \textbf{46.2} & 43.3 \\ 
walker2d-med-replay & \textbf{33.7} & 21.8 & 32.6 & 25.2 \\ 
\color{lightgray} hopper-med-replay  & \color{lightgray} 39.2 & \color{lightgray}  29.5 & \color{lightgray} \textbf{48.6} & \color{lightgray} 31.4 \\ 
halfcheetah-med-exp & \textbf{108.4} & \textbf{107.5} & 62.4 & 97.9 \\ 
walker2d-med-exp  & \textbf{111.8} & 109.1 & 98.7 & 101.1 \\ 
\color{lightgray} hopper-med-exp  & \color{lightgray} \textbf{112.8} & \color{lightgray}  \textbf{112.5} &\color{lightgray}  \textbf{111.0} & \color{lightgray} \textbf{112.2}\\ 
     \hline
        \end{tabular}
    \caption{\small Results of mujoco-v0 dataset. We grayed out the results of hopper-v0 because these datasets have bug (see D4RL github).}
    \label{tb:mujocov0}
\end{table}

We used a selection of the Mujoco datasets (v2) and Adroit datasets (v1) from D4RL as our benchmark environments. For each evaluation, we roll out the mean part of the Gaussian policy for 5 rollouts and compute the Monte Carlo return. For each dataset, we report the statistical results over 10 random seeds in \cref{tab:ATAC with CI}.

Compared with the summary we provided in the main text (\cref{tab:main exp results}), \cref{tab:ATAC with CI} includes also the confidence interval which shows how much the $25^{th}$ and the $75^{th}$ percentiles of performance deviate from the median (i.e. the $50^{th}$ percentile). In addition, \cref{tab:ATAC with CI} also provides the selected hyperparamter $\beta$ for each method.

Overall, we see that confidence intervals are small for \algo, except for larger variations happening in \textit{hopper-rand}, \textit{pen-human}, and \textit{hammer-human}. Therefore, the performance improvement of \algo from other offline RL baselines and behavior policies is significant.
We also see that \algo most of the time picks $\beta=64$ for the Mujoco datasets, except for the halfcheetah domain, and has a tendency of picking smaller $\beta$ as the dataset starts to contain expert trajectories (i.e. in *-exp datasets). This is reasonable, since when the behavior policy has higher performance, an agent requires less information from the reward to perform well; in the extreme of learning with trajectories of the optimal policy, the learner can be optimal just by IL.

We also include extra ablation results on the effectiveness of DQRA loss in stabilizing learning in \cref{fig:extra dqra ablation},  which includes two extra hopper datasets compared with \cref{fig:dqra ablation}. Similar to the results in the main paper, we see that $w=1$ is unstable, $w=0$ is stable but under-performing, while using $w\in(0,1)$ strikes a balance between the two. Our choice $w=0.5$ has the best performance in these three datasets and is numerically stable.
We also experimented with the max-aggregation version recently proposed by~\citet{wang2021convergent}. It does address the instability issue seen in the typical bootstrapped version $w=1$, but its results tend to be noisier compared with $w=0.5$ as it makes the optimization landscape more non-smooth.

Lastly, we include experimental results of \algo on D4RL mujoco-v0 datasets in \cref{tb:mujocov0}. We used v2 instead of v0 in the main results, because 1) hopper-v0 has a bug (see D4RL github; for this reason they are grayed out in \cref{tb:mujocov0}), and 2) some baselines we compare \algo with also used v2 (or they didn't specify and we suspect so). Here we include these results for completeness.

\subsection{Robust Policy Improvement}

We study empirically the robust policy improvement property of \algo.
First we provide an extensive validation on how ATAC$^*$ performs with different $\beta$ on all datasets in \cref{fig:robust PI (mujoco)} and \cref{fig:robust PI (adroit)}, which are the complete version of \cref{fig:robust PI}. In these figures, we plot the results of ATAC$^*$ (relative pessimism) and ATAC$_0^*$ (absolute pessimism) (which is a deep learning implementation of PSPI~\citep{xie2021bellman}) in view of the behavior policy's performance. These results show similar trends as we have observed in \cref{fig:robust PI}. \algo can robustly improve from the behavior policy over a wide range of $\beta$ values. In particular, we see the performance degrades below the behavior policy only for large $\beta$s, because of the following reasons. When $\beta\to0$ \algo converges to the IL mode, which can recover the behavior policy performance if the realizability assumption is satisfied. On the other hand, when $\beta$ is too large, \cref{prop:safe_pi} shows that the statistical error will start to dominate and therefore lead to substandard performance. This robust policy improvement property means that practitioners of \algo can online tune its performance by starting with $\beta=0$ and the gradually increasing $\beta$ until the performance drop,  without ever dropping below the performance of behavior policy much.

\cref{fig:robust PI (mujoco)} and \cref{fig:robust PI (adroit)} show the robustness of ATAC$^*$ which uses the best checkpoint. Below in \cref{tab:robust PI scores} we validate further whether safe policy improvement holds across iterates.
To this end, we define a robust policy improvement score
\begin{align}
    \textrm{score}_{\textrm{RPI}}(\pi) \coloneqq \frac{J(\pi)-J(\mu)}{|J(\mu)|}
\end{align}
which captures how a policy $\pi$ performs relatively to the behavior policy $\mu$.
\cref{tab:robust PI scores} shows the percentiles of the robust policy improvement score for each dataset, over \textit{all} the $\beta$ choices, random seeds, and iterates from the 100$^{th}$ epoch to the 900$^{th}$ epoch of \algo training.
Overall, we see that in most datasets (excluding *-human and *-clone datasets which do not satisfy our theoretical realizability assumption), more than 50\% of iterates generated by \algo across all the experiments are better than the behavior policy. For others, more than 60\% of iterates are within 80 \% of the behavior policy's performance. This robustness result is quite remarkable as it includes iterates where \algo has not fully converged as well as bad choices of $\beta$.
%

\begin{table*}[t!]
\footnotesize
    \centering
        \begin{tabular}{|c | c  c c c c c  c c c c  |}
\hline
 & $10^{th}$ & $20^{th}$ & $30^{th}$ & $40^{th}$ & $50^{th}$ & $60^{th}$ & $70^{th}$ & $80^{th}$ & $90^{th}$ & $100^{th}$\\
 \hline
halfcheetah-rand & \textbf{0.9} & \textbf{1.0} & \textbf{1.0} & \textbf{1.0} & \textbf{1.0} & \textbf{1.1} & \textbf{1.4} & \textbf{1.5} & \textbf{1.8} & \textbf{5.5}\\
walker2d-rand & \textbf{1.9} & \textbf{2.3} & \textbf{3.5} & \textbf{5.6} & \textbf{16.5} & \textbf{64.6} & \textbf{134.0} & \textbf{139.1} & \textbf{159.2} & \textbf{519.1}\\
hopper-rand & \textbf{0.1} & \textbf{0.2} & \textbf{1.3} & \textbf{3.5} & \textbf{6.7} & \textbf{10.6} & \textbf{11.4} & \textbf{12.6} & \textbf{23.2} & \textbf{63.3}\\
halfcheetah-med & -0.1 & \textbf{0.0} & \textbf{0.1} & \textbf{0.1} & \textbf{0.1} & \textbf{0.1} & \textbf{0.2} & \textbf{0.3} & \textbf{0.3} & \textbf{0.4}\\
walker2d-med & \textbf{0.1} & \textbf{0.2} & \textbf{0.3} & \textbf{0.3} & \textbf{0.3} & \textbf{0.4} & \textbf{0.4} & \textbf{0.4} & \textbf{0.4} & \textbf{0.5}\\
hopper-med & \textbf{0.2} & \textbf{0.2} & \textbf{0.3} & \textbf{0.4} & \textbf{0.4} & \textbf{0.6} & \textbf{0.7} & \textbf{0.9} & \textbf{1.1} & \textbf{1.4}\\
halfcheetah-med-replay & \textbf{0.6} & \textbf{0.6} & \textbf{0.6} & \textbf{0.7} & \textbf{0.7} & \textbf{0.8} & \textbf{0.8} & \textbf{0.9} & \textbf{0.9} & \textbf{1.0}\\
walker2d-med-replay & -1.0 & -1.0 & -1.0 & -0.2 & \textbf{3.5} & \textbf{4.4} & \textbf{4.8} & \textbf{5.0} & \textbf{5.2} & \textbf{5.5}\\
hopper-med-replay & -1.0 & -0.9 & -0.9 & \textbf{1.1} & \textbf{5.4} & \textbf{6.0} & \textbf{6.0} & \textbf{6.1} & \textbf{6.1} & \textbf{6.2}\\
halfcheetah-med-exp & -0.6 & -0.5 & -0.4 & -0.3 & -0.2 & \textbf{-0.0} & \textbf{0.2} & \textbf{0.4} & \textbf{0.5} & \textbf{0.5}\\
walker2d-med-exp & -0.1 & -0.1 & \textbf{0.0} & \textbf{0.3} & \textbf{0.3} & \textbf{0.3} & \textbf{0.3} & \textbf{0.4} & \textbf{0.4} & \textbf{0.4}\\
hopper-med-exp & -0.6 & -0.4 & -0.2 & -0.2 & -0.1 & -0.1 & \textbf{0.0} & \textbf{0.6} & \textbf{0.7} & \textbf{0.8}\\

\hline
pen-human & -1.0 & -1.0 & -1.0 & -0.9 & -0.9 & -0.9 & -0.8 & -0.8 & -0.7 & -0.2\\
hammer-human & -1.1 & -1.1 & -1.1 & -1.1 & -1.0 & -1.0 & -1.0 & -1.0 & -1.0 & \textbf{0.9}\\
door-human & -1.1 & -1.1 & -1.1 & -1.1 & -1.1 & -1.1 & -1.1 & -1.0 & -0.9 & -0.1\\
relocate-human & -1.0 & -1.0 & -1.0 & -1.0 & -1.0 & -1.0 & -1.0 & -1.0 & -1.0 & -0.9\\
pen-cloned & -1.0 & -1.0 & -1.0 & -0.9 & -0.9 & -0.8 & -0.7 & -0.6 & -0.5 & \textbf{0.5}\\
hammer-cloned & -1.3 & -1.3 & -1.3 & -1.3 & -1.3 & -1.2 & -1.2 & -1.1 & -1.0 & \textbf{9.2}\\
door-cloned & -1.2 & -1.2 & -1.2 & -1.2 & -1.2 & -1.2 & -1.1 & -1.0 & -0.8 & \textbf{1.4}\\
relocate-cloned & -1.0 & -1.0 & -1.0 & -1.0 & -1.0 & -1.0 & -1.0 & -1.0 & -1.0 & -0.7\\
pen-exp & -0.4 & -0.2 & -0.1 & \textbf{0.0} & \textbf{0.1} & \textbf{0.2} & \textbf{0.2} & \textbf{0.3} & \textbf{0.4} & \textbf{0.6}\\
hammer-exp & -1.0 & -1.0 & -1.0 & -0.8 & -0.6 & -0.2 & \textbf{0.1} & \textbf{0.3} & \textbf{0.3} & \textbf{0.4}\\
door-exp & -1.0 & -0.8 & -0.6 & -0.4 & -0.2 & -0.2 & \textbf{-0.0} & \textbf{0.0} & \textbf{0.0} & \textbf{0.1}\\
relocate-exp & -0.4 & -0.3 & -0.3 & -0.2 & -0.2 & -0.1 & -0.1 & \textbf{-0.0} & \textbf{0.0} & \textbf{0.1}\\
\hline
        \end{tabular}
    \caption{\small The robust policy improvement scores of \algo. We report for each dataset, the percentiles of iterates over all 9 choices of $\beta$, 10 seeds, and 800 epochs (from the 100$^{th}$ to the 900$^{th}$ epochs). In most datasets (excluding *-human and *-clone datasets which likely do not satisfy our theoretical realizability assumption), more than 50\% of iterates generated by \algo across all seeds and $\beta$s are better than the behavior policy. For others, more than 60\% of iterates are within 80\% of the behavior policy's performance.
    }
    \label{tab:robust PI scores}
\end{table*}

\begin{figure*}[t!]
	\centering
	\begin{subfigure}{0.24\textwidth}
		\includegraphics[width=\textwidth]{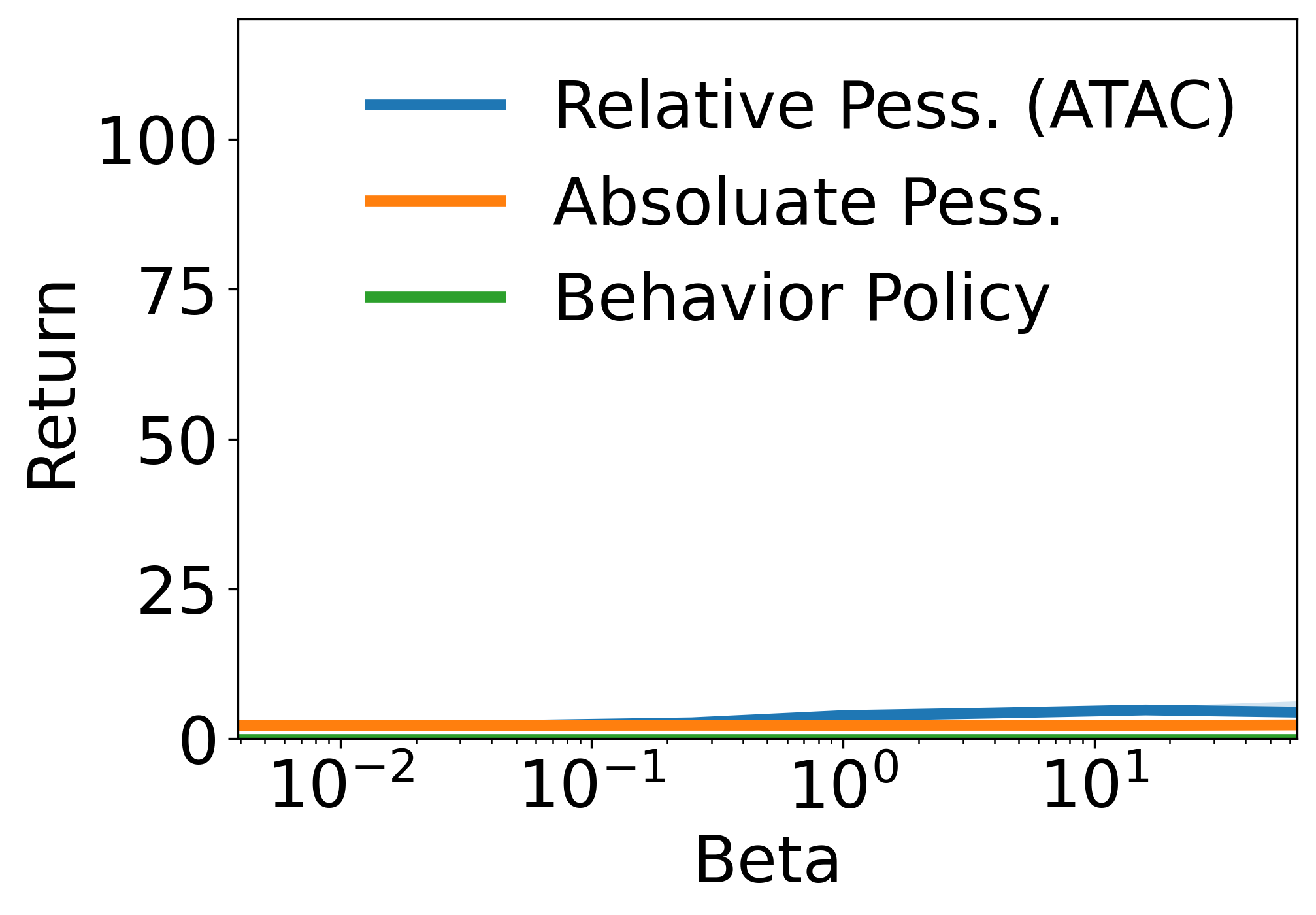}
		\caption{halfcheetah-random}
		\label{fig:halfcheetah-random robust PI}
	\end{subfigure}
	\begin{subfigure}{0.24\textwidth}
		\includegraphics[width=\textwidth]{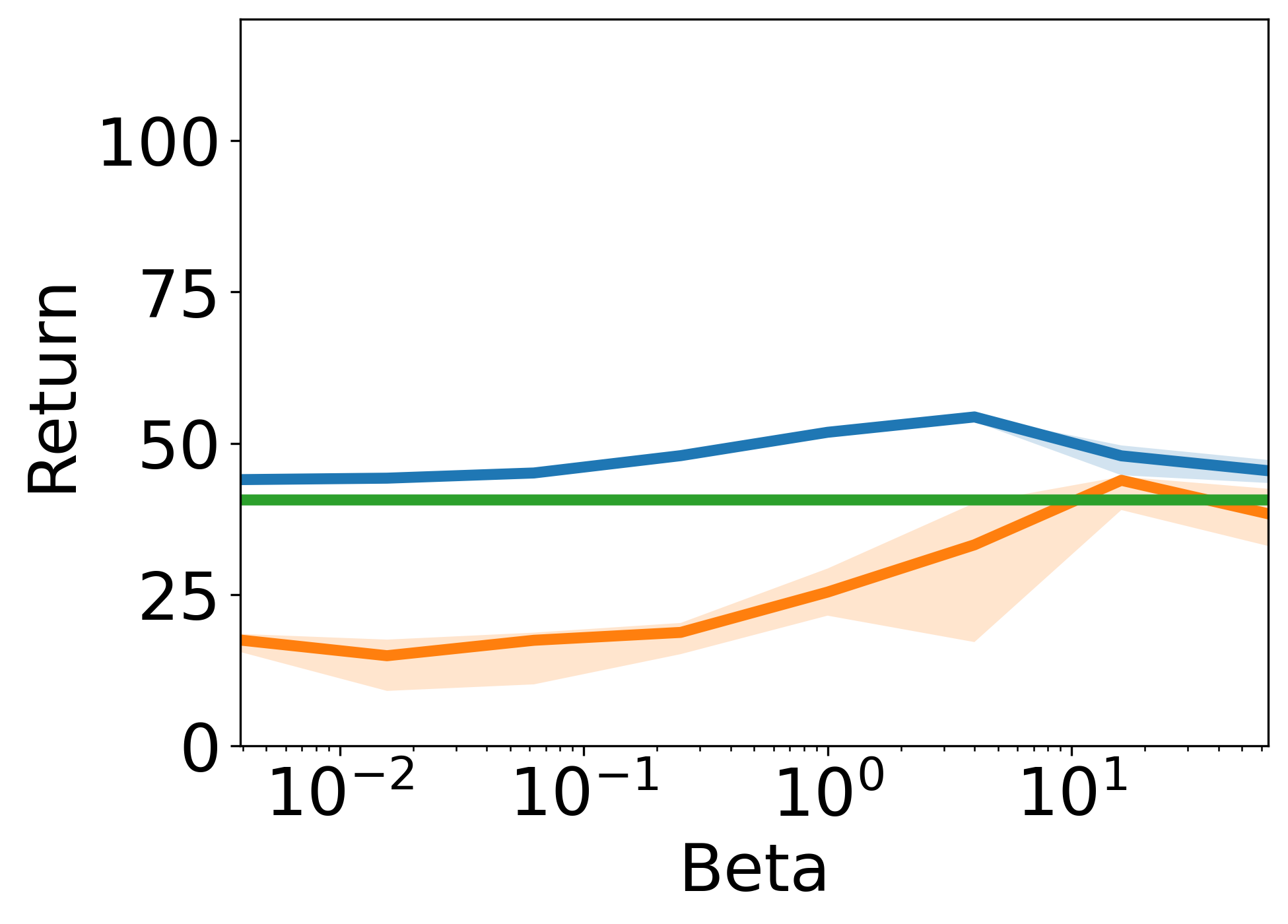}
		\caption{halfcheetah-medium}
		\label{fig:halfcheetah-medium robust PI}
	\end{subfigure}
	\begin{subfigure}{0.24\textwidth}
		\includegraphics[width=\textwidth]{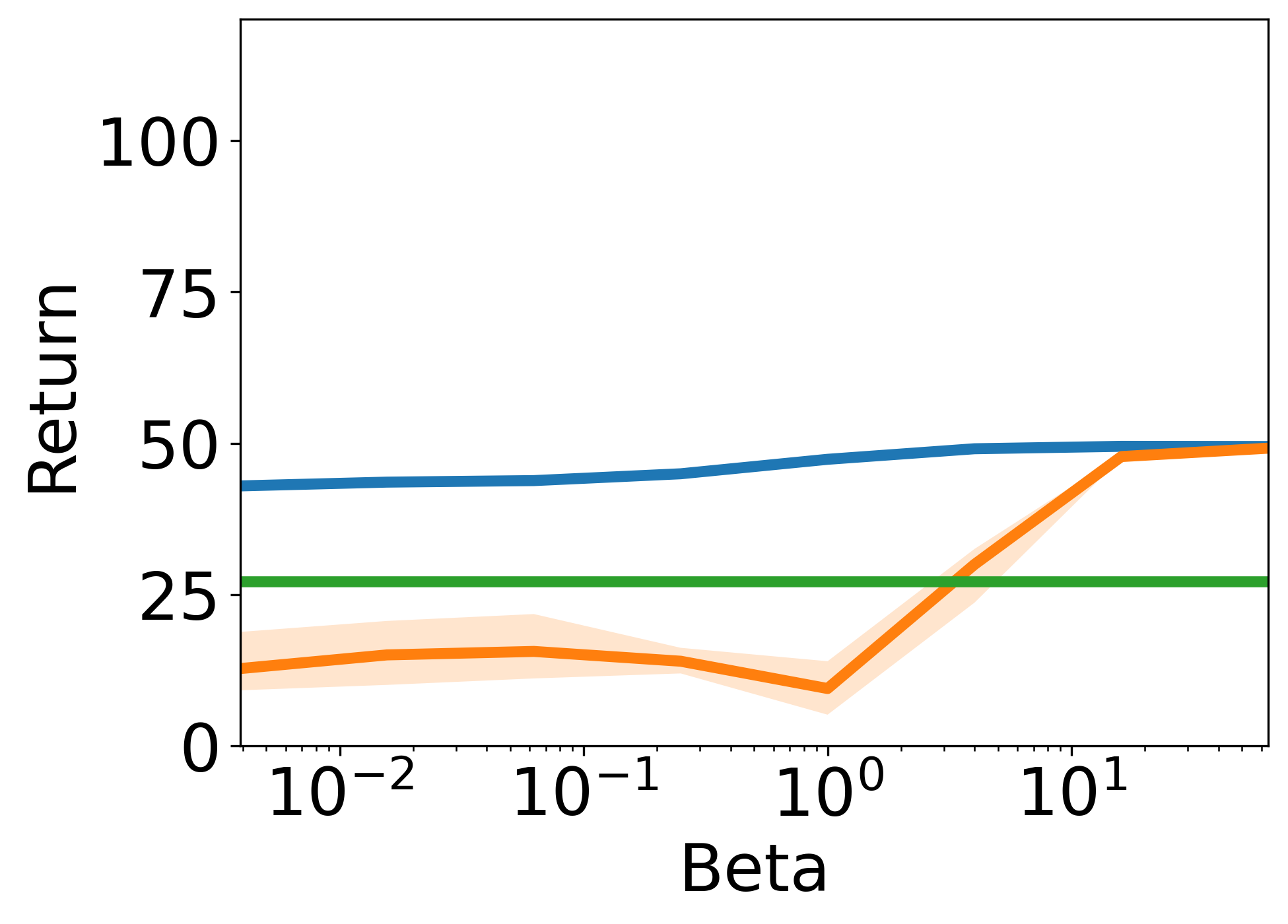}
		\caption{halfcheetah-medium-replay}
		\label{fig:halfcheetah-medium-replay robust PI}
	\end{subfigure}
	\begin{subfigure}{0.24\textwidth}
		\includegraphics[width=\textwidth]{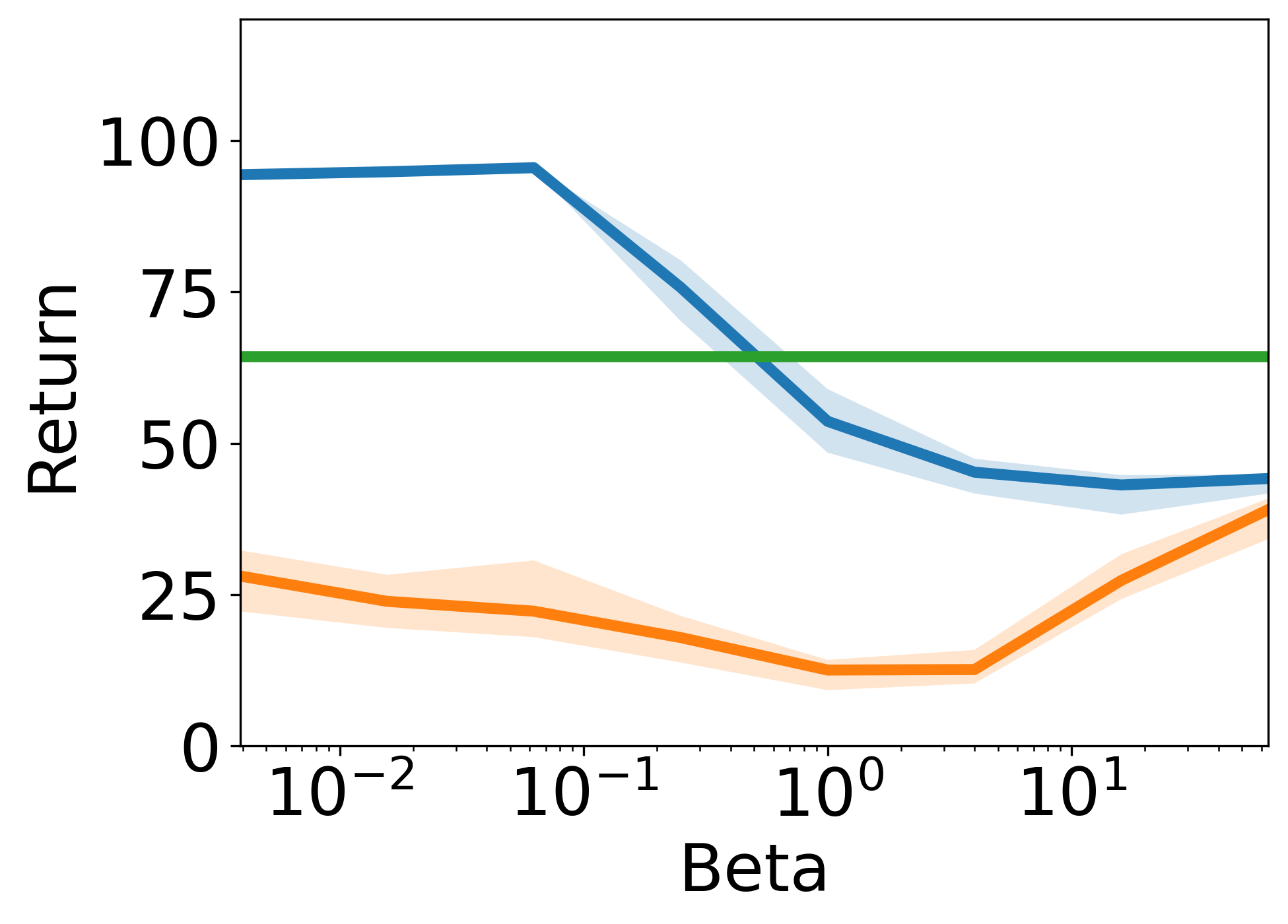}
		\caption{halfcheetah-medium-expert}
		\label{fig:halfcheetah-medium-expert robust PI}
	\end{subfigure}

	\begin{subfigure}{0.24\textwidth}
		\includegraphics[width=\textwidth]{figures/hopper-random-v2.png}
		\caption{hopper-random}
	\end{subfigure}
	\begin{subfigure}{0.24\textwidth}
		\includegraphics[width=\textwidth]{figures/hopper-medium-v2.png}
		\caption{hopper-medium}
	\end{subfigure}
	\begin{subfigure}{0.24\textwidth}
		\includegraphics[width=\textwidth]{figures/hopper-medium-replay-v2.png}
		\caption{hopper-medium-replay}
	\end{subfigure}
	\begin{subfigure}{0.24\textwidth}
		\includegraphics[width=\textwidth]{figures/hopper-medium-expert-v2.png}
		\caption{hopper-medium-expert}
	\end{subfigure}

	\begin{subfigure}{0.24\textwidth}
		\includegraphics[width=\textwidth]{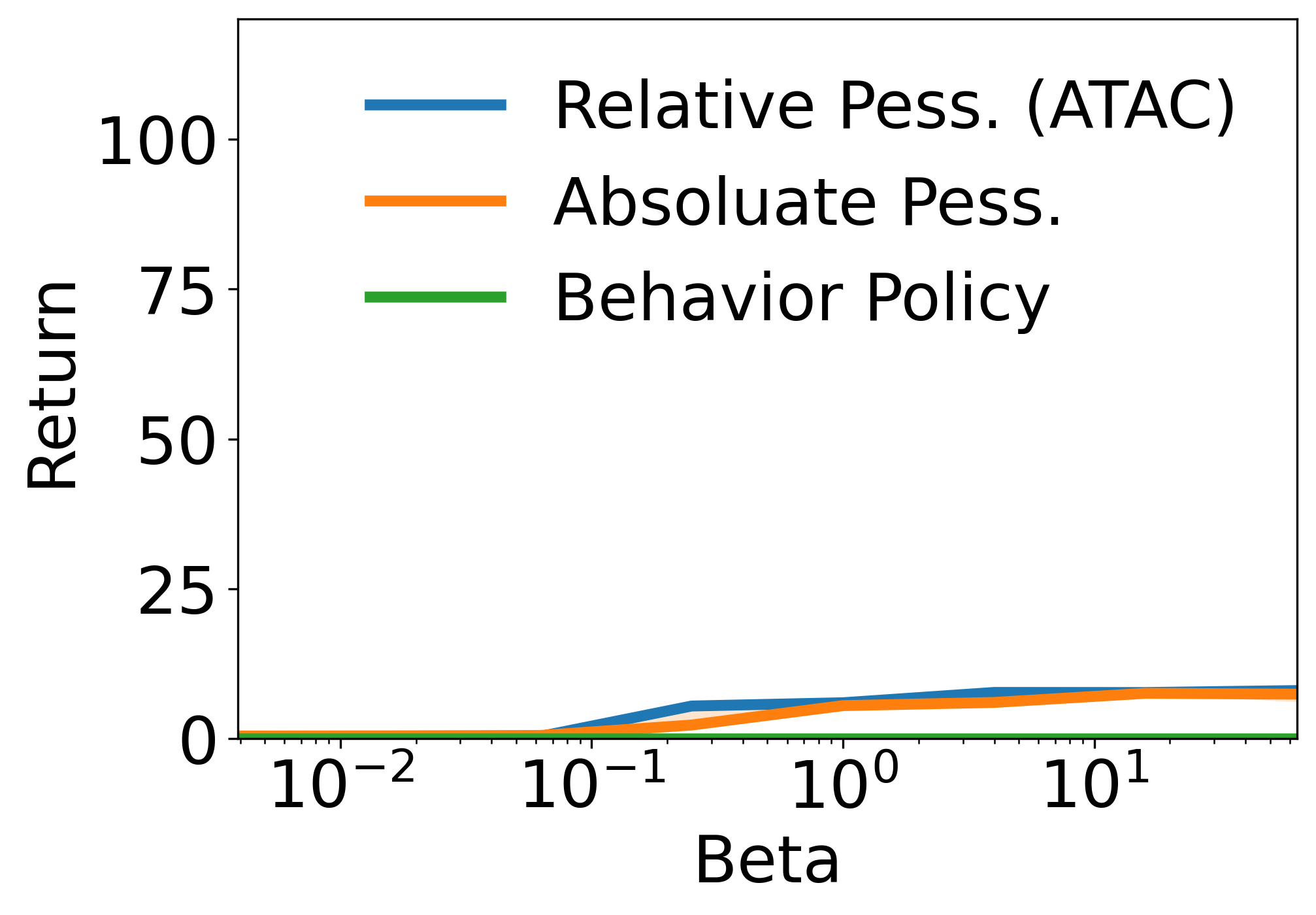}
		\caption{walker2d-random}
		\label{fig:walker2d-random robust PI}
	\end{subfigure}
	\begin{subfigure}{0.24\textwidth}
		\includegraphics[width=\textwidth]{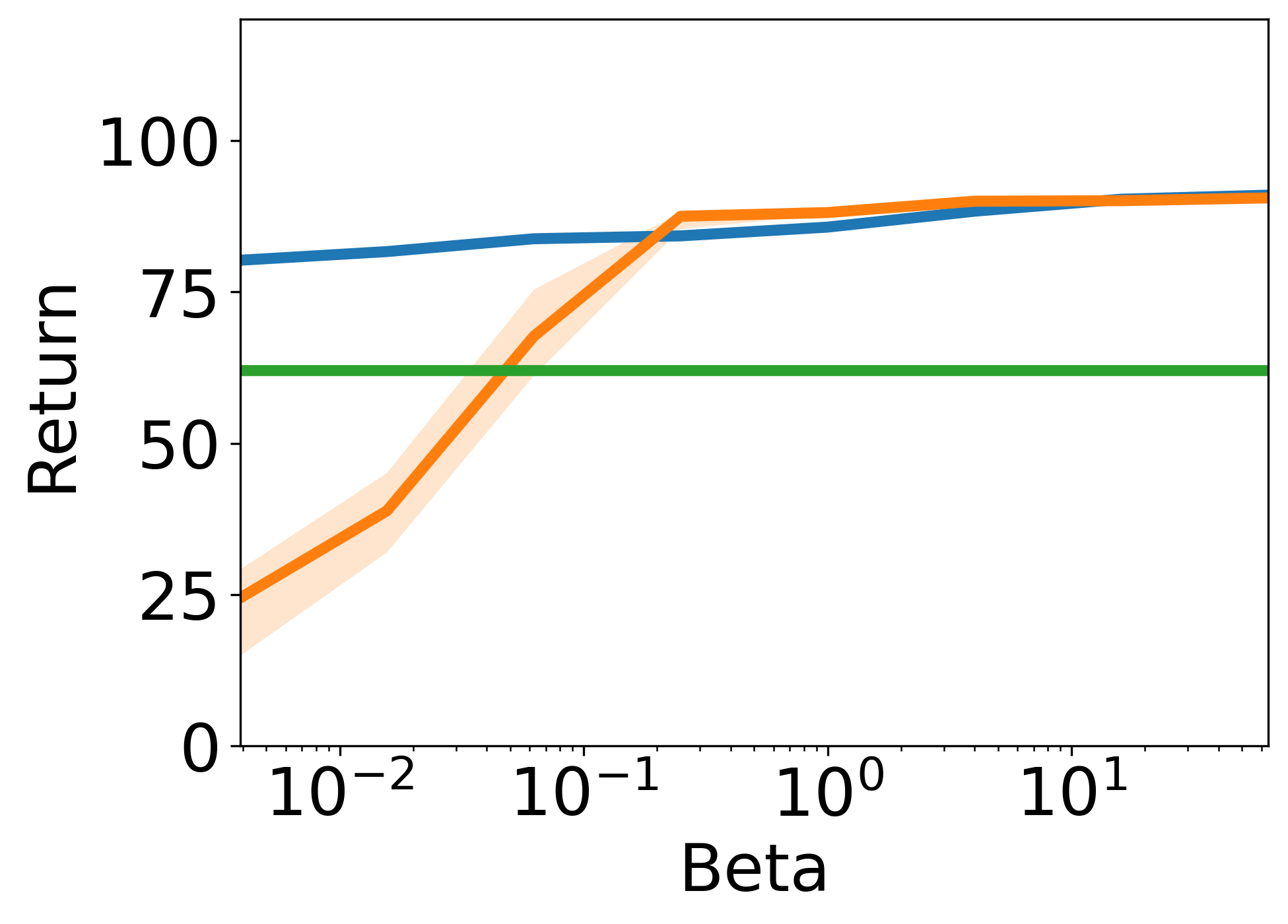}
		\caption{walker2d-medium}
		\label{fig:walker2d-medium robust PI}
	\end{subfigure}
	\begin{subfigure}{0.24\textwidth}
		\includegraphics[width=\textwidth]{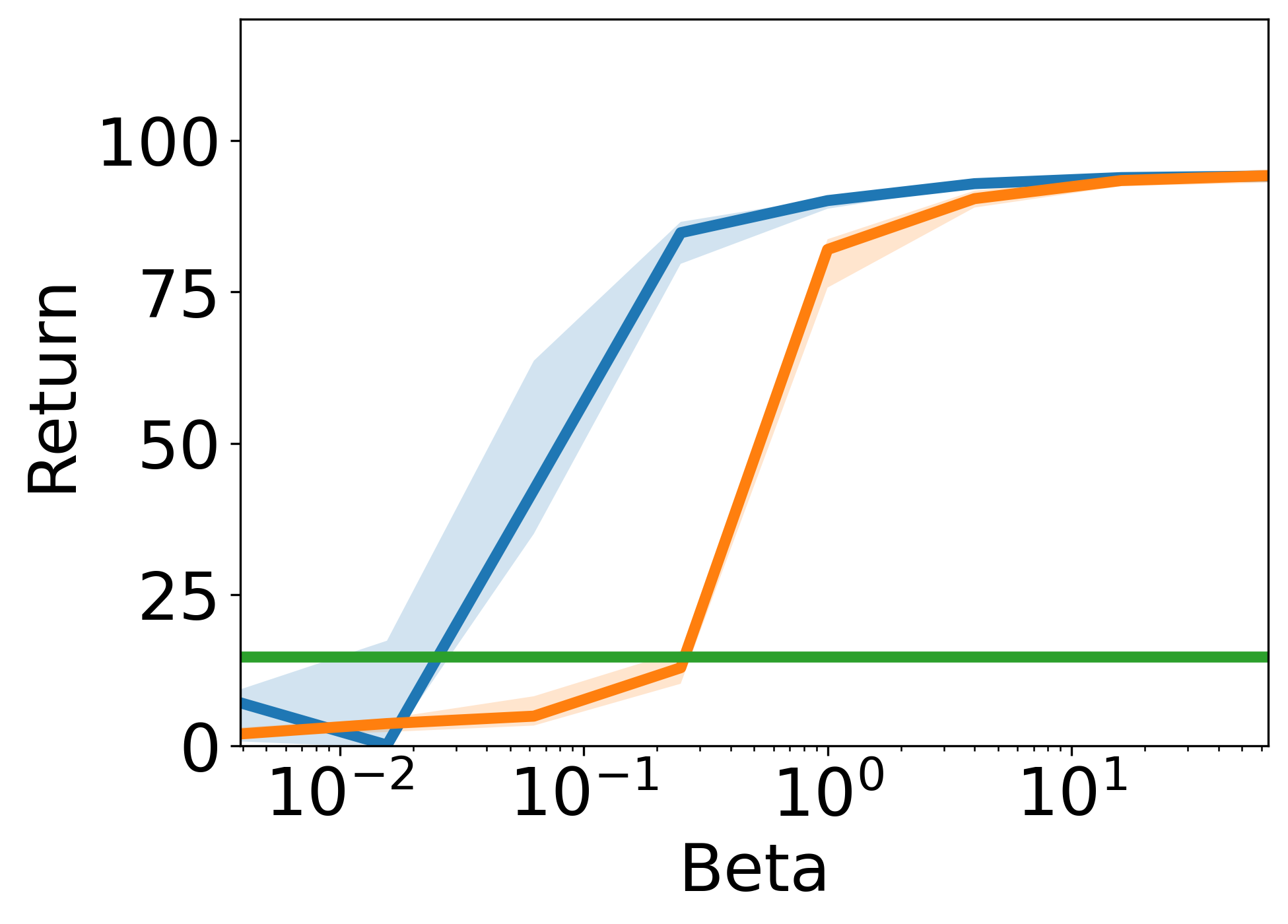}
		\caption{walker2d-medium-replay}
		\label{fig:walker2d-medium-replay robust PI}
	\end{subfigure}
	\begin{subfigure}{0.24\textwidth}
		\includegraphics[width=\textwidth]{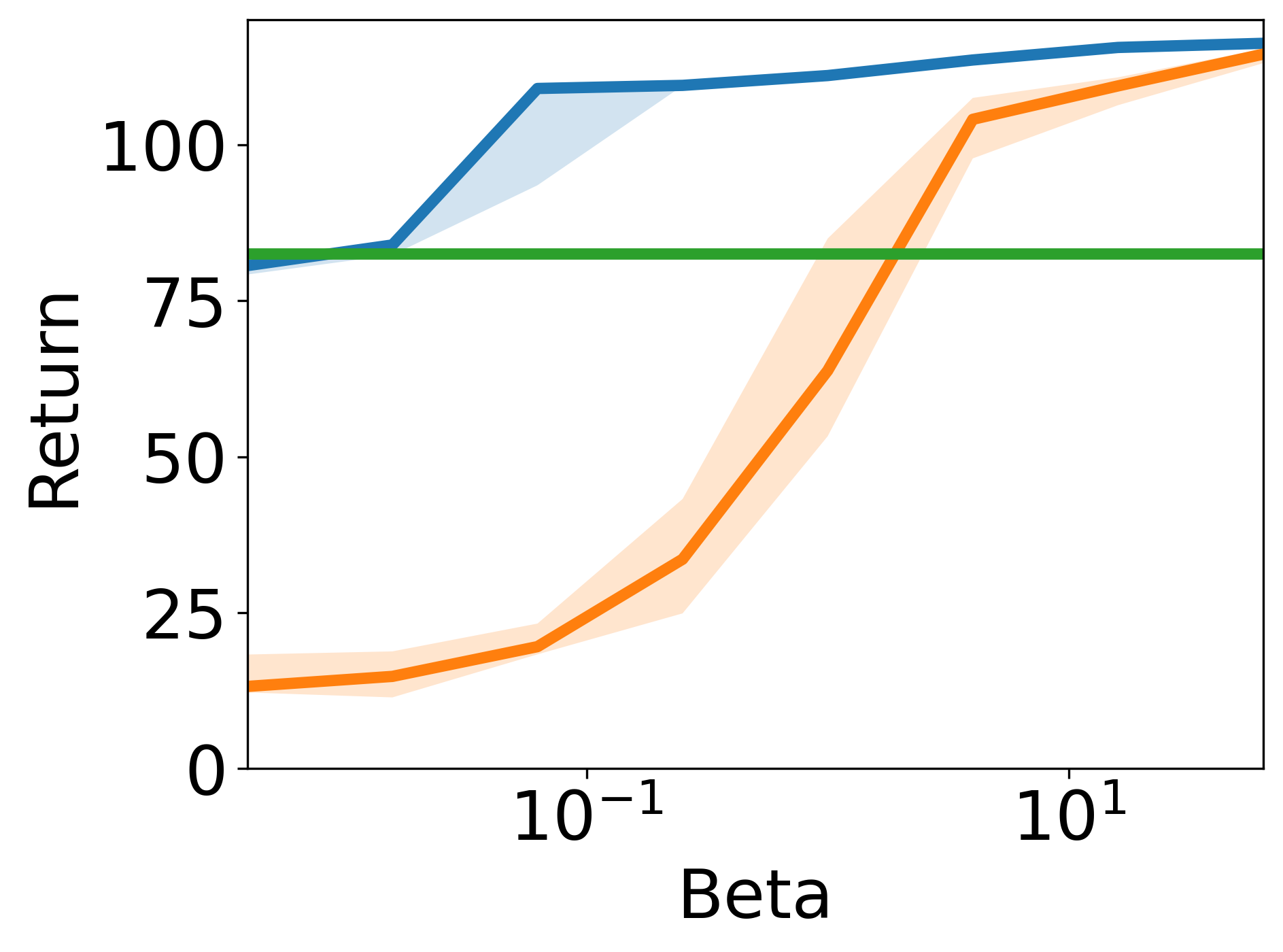}
		\caption{walker2d-medium-expert}
		\label{fig:walker2d-medium-expert robust PI}
	\end{subfigure}

\caption{\small{Robust Policy Improvement of \algo in the Mujoco domains. \algo based on \emph{relative} pessimism improves from behavior policies over a wide range of hyperparameters that controls the degree of pessimism. On the contrary, \emph{absolute} pessimism does not have this property and needs well-tuned hyperparameters to ensure safe policy improvement.
The plots show the $25^{th}$, $50^{th}$, $75^{th}$ percentiles over 10 random seeds.
}}
\label{fig:robust PI (mujoco)}
\end{figure*}

\begin{figure*}[t!]
	\centering
	\begin{subfigure}{0.24\textwidth}
		\includegraphics[width=\textwidth]{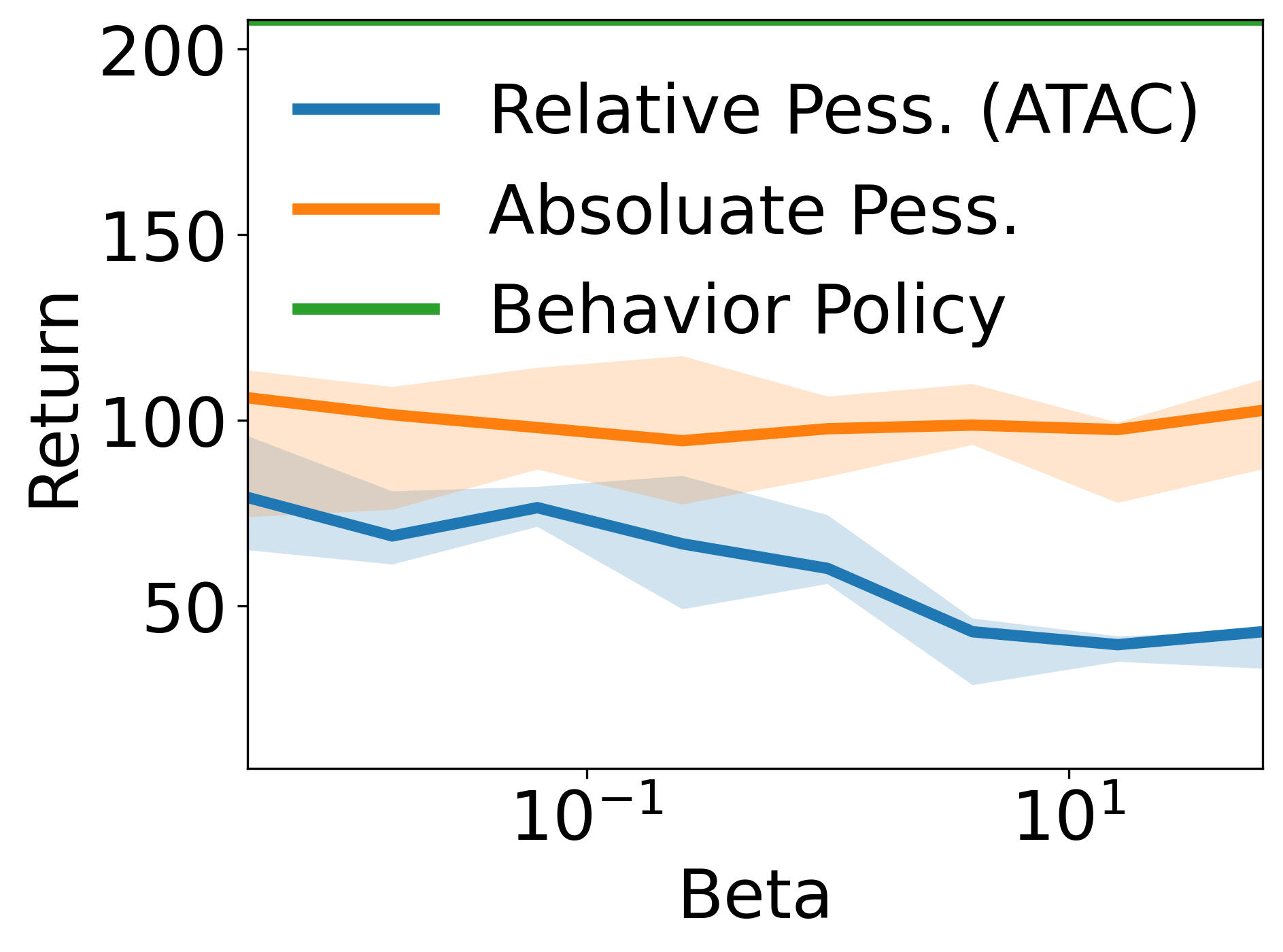}
		\caption{pen-human}
		\label{fig:pen-random robust PI}
	\end{subfigure}
	\begin{subfigure}{0.24\textwidth}
		\includegraphics[width=\textwidth]{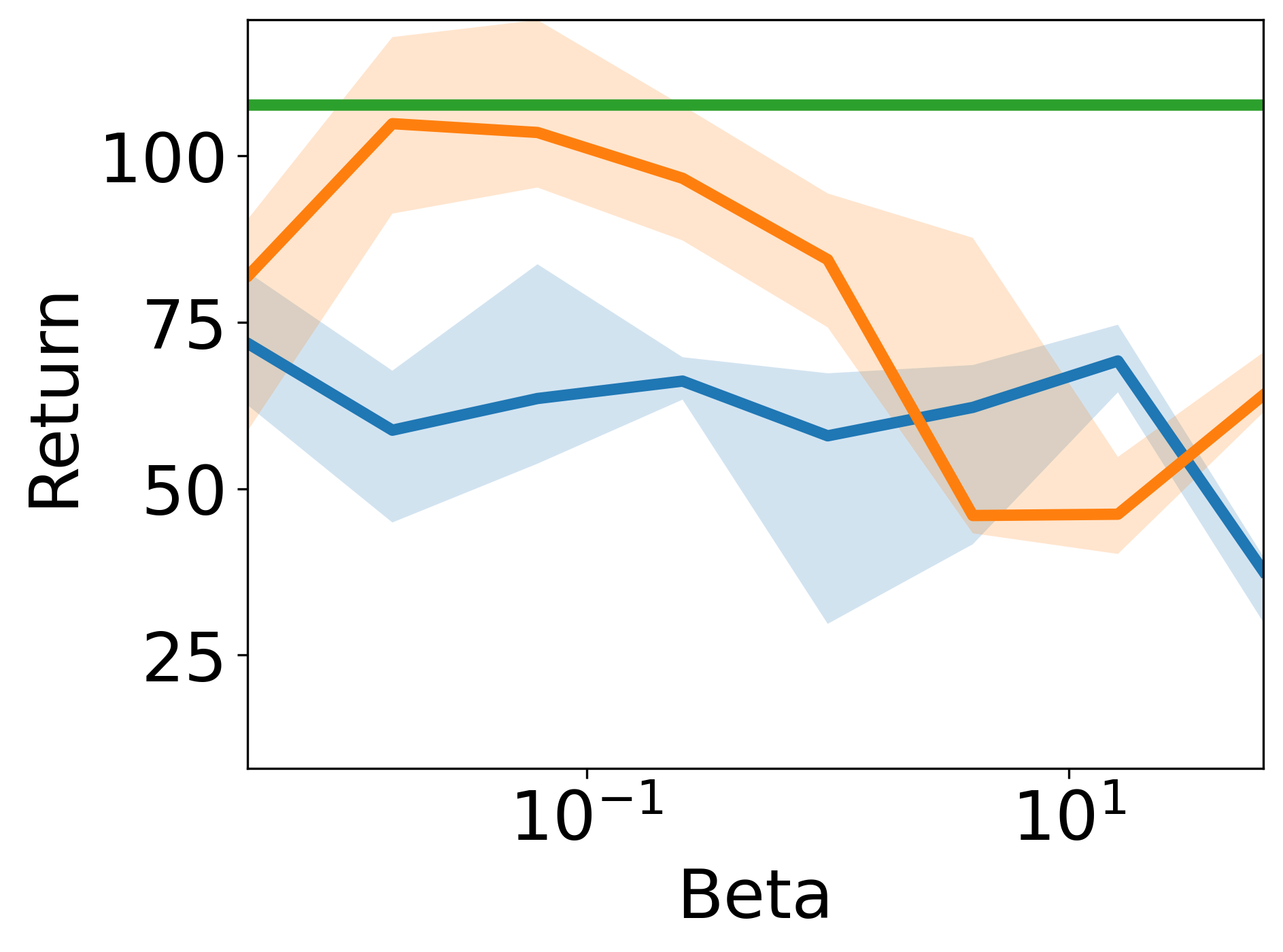}
		\caption{pen-cloned}
		\label{fig:pen-medium robust PI}
	\end{subfigure}
	\begin{subfigure}{0.24\textwidth}
		\includegraphics[width=\textwidth]{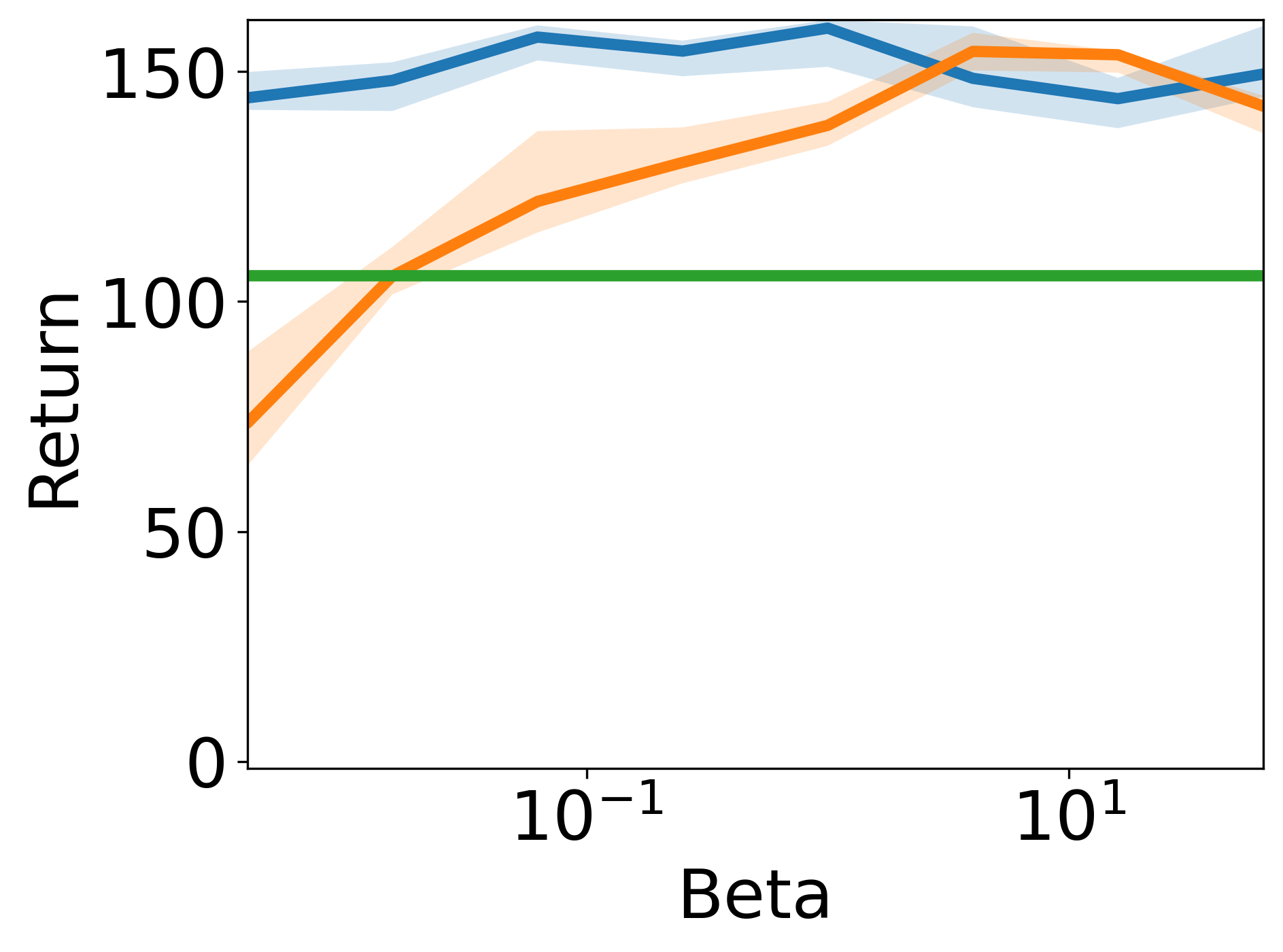}
		\caption{pen-exp}
		\label{fig:pen-medium-replay robust PI}
	\end{subfigure}
	\\
	\begin{subfigure}{0.24\textwidth}
		\includegraphics[width=\textwidth]{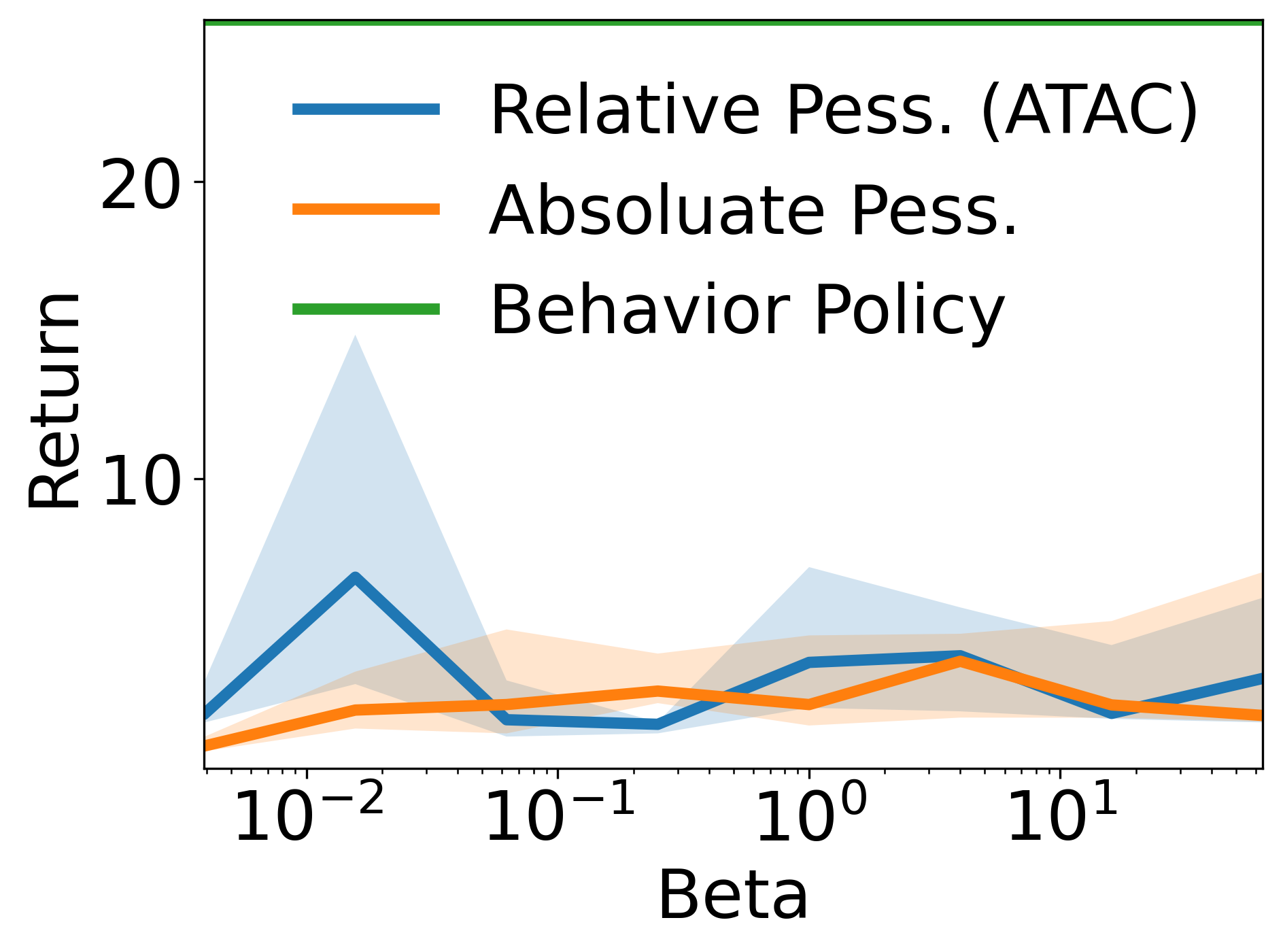}
		\caption{hammer-human}
		\label{fig:hammer-random robust PI}
	\end{subfigure}
	\begin{subfigure}{0.24\textwidth}
		\includegraphics[width=\textwidth]{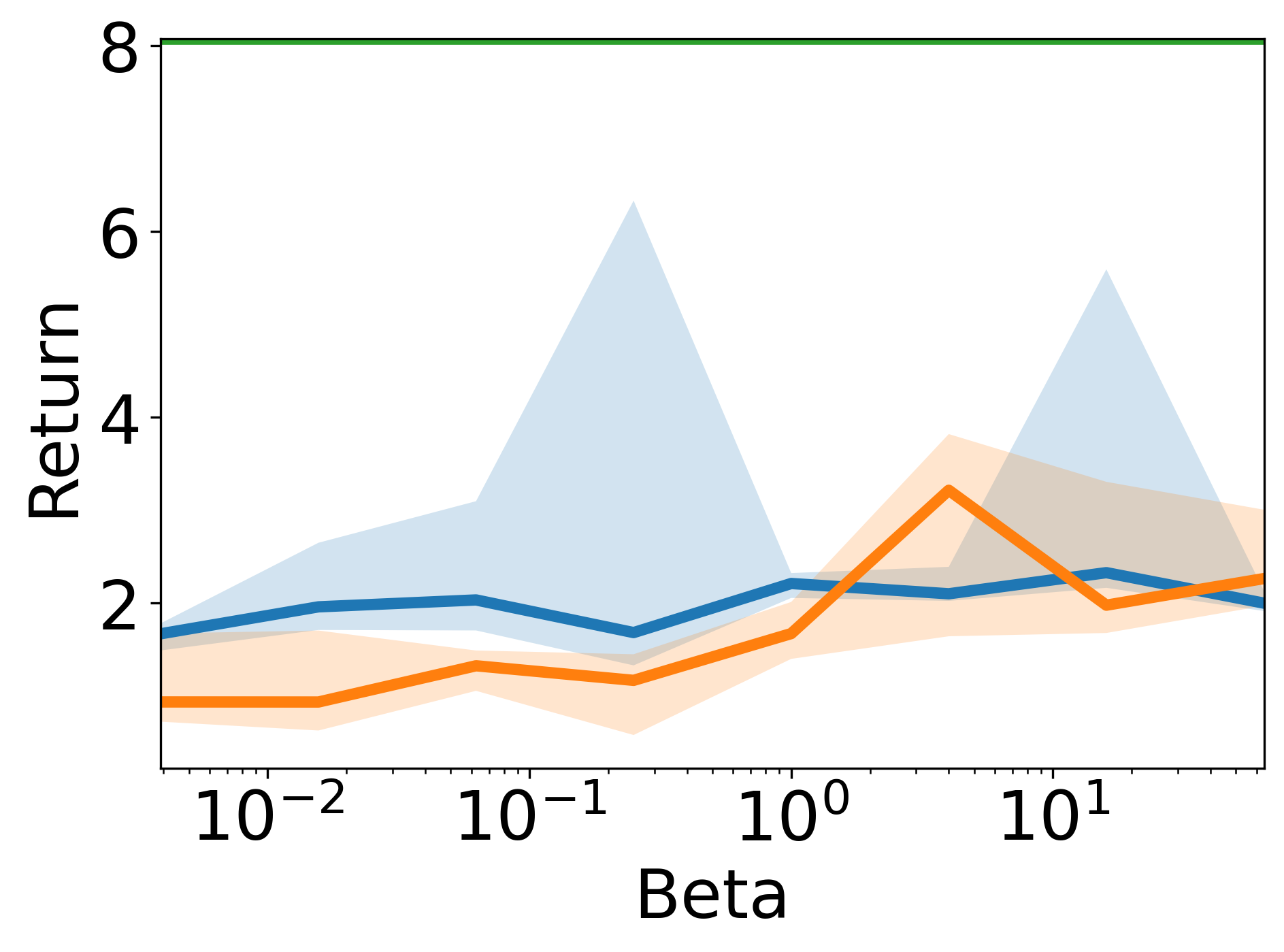}
		\caption{hammer-cloned}
		\label{fig:hammer-medium robust PI}
	\end{subfigure}
	\begin{subfigure}{0.24\textwidth}
		\includegraphics[width=\textwidth]{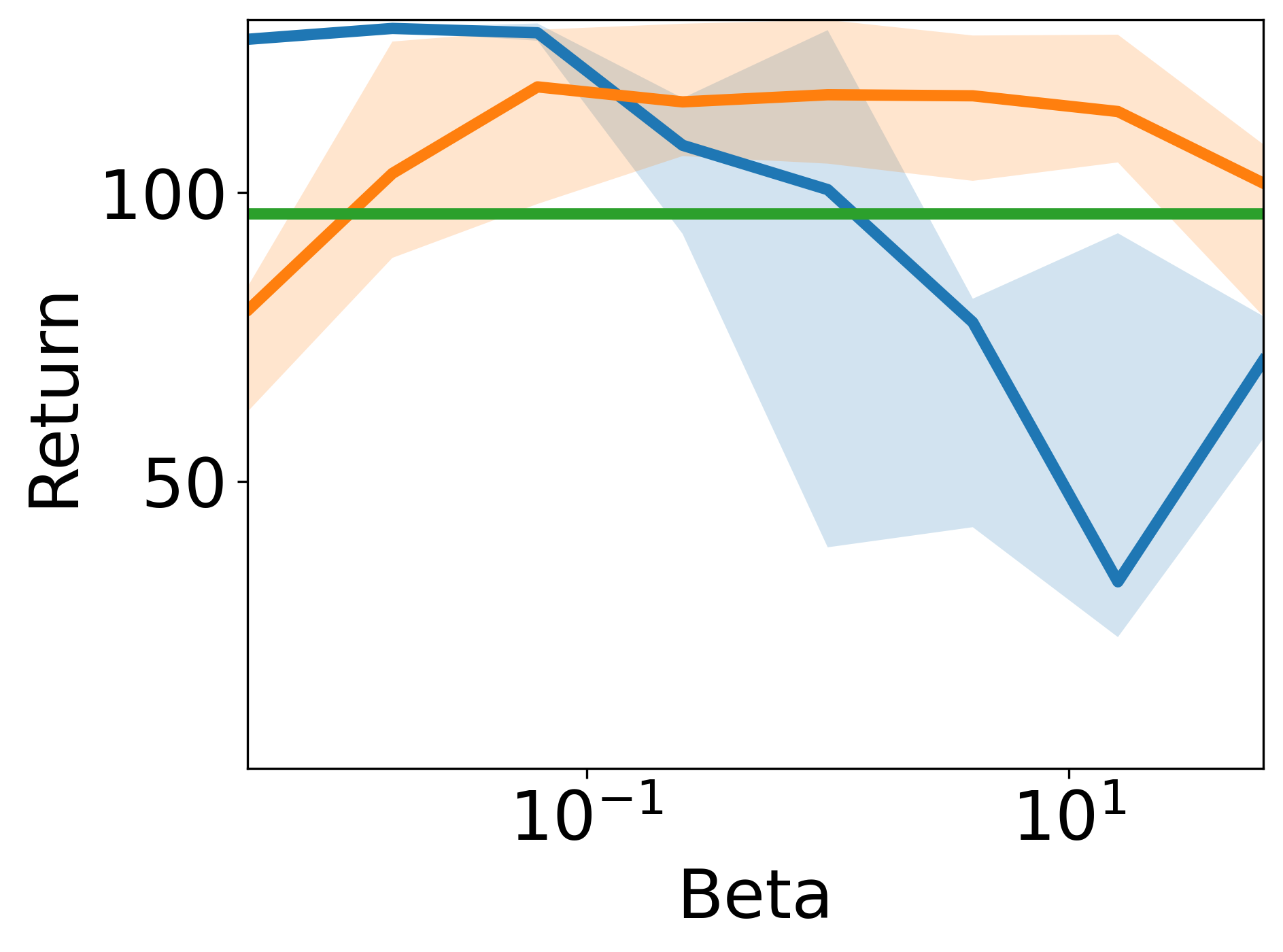}
		\caption{hammer-exp}
		\label{fig:hammer-medium-replay robust PI}
	\end{subfigure}
	\\
	\begin{subfigure}{0.24\textwidth}
		\includegraphics[width=\textwidth]{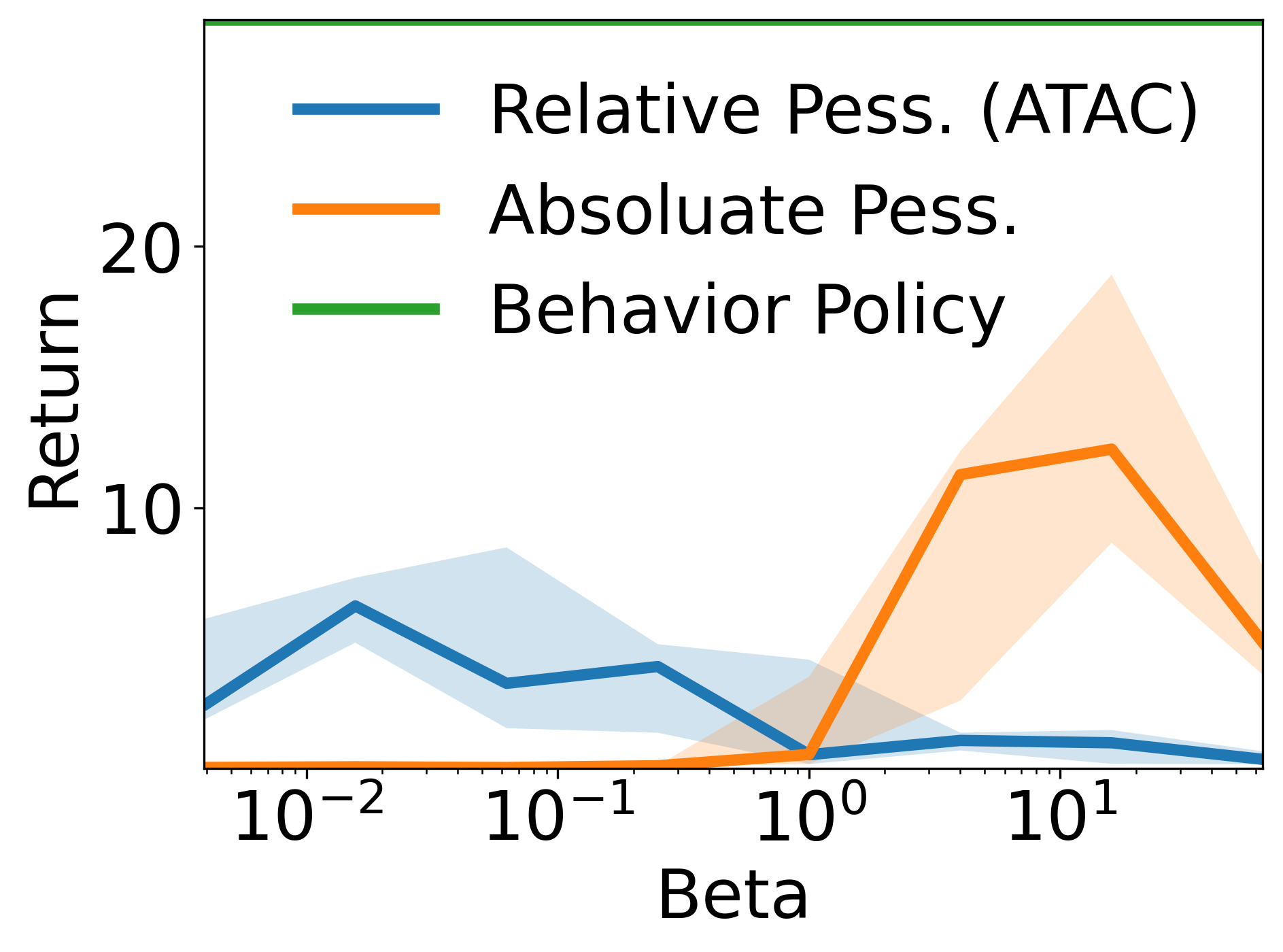}
		\caption{door-human}
		\label{fig:door-random robust PI}
	\end{subfigure}
	\begin{subfigure}{0.24\textwidth}
		\includegraphics[width=\textwidth]{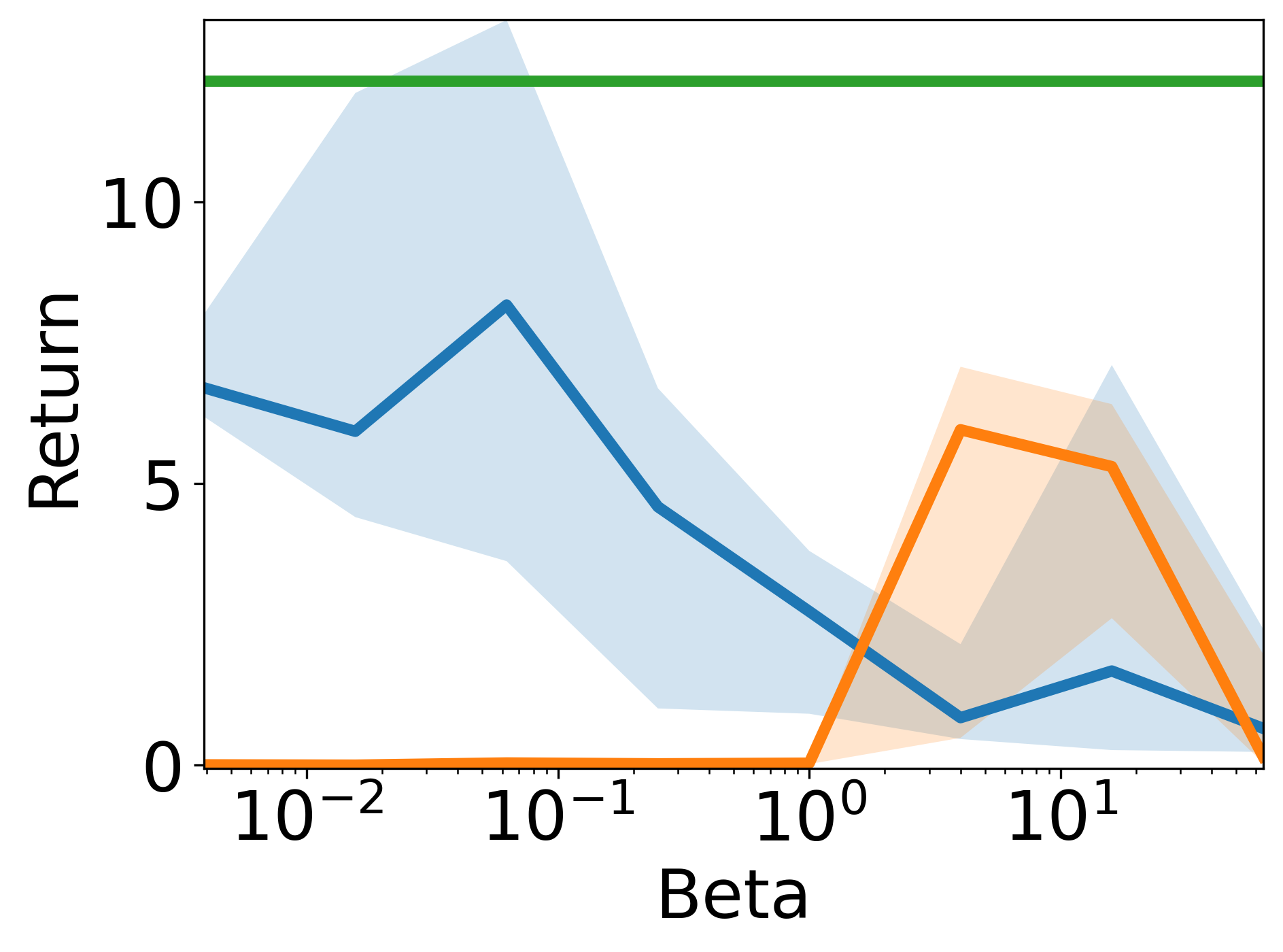}
		\caption{door-cloned}
		\label{fig:door-medium robust PI}
	\end{subfigure}
	\begin{subfigure}{0.24\textwidth}
		\includegraphics[width=\textwidth]{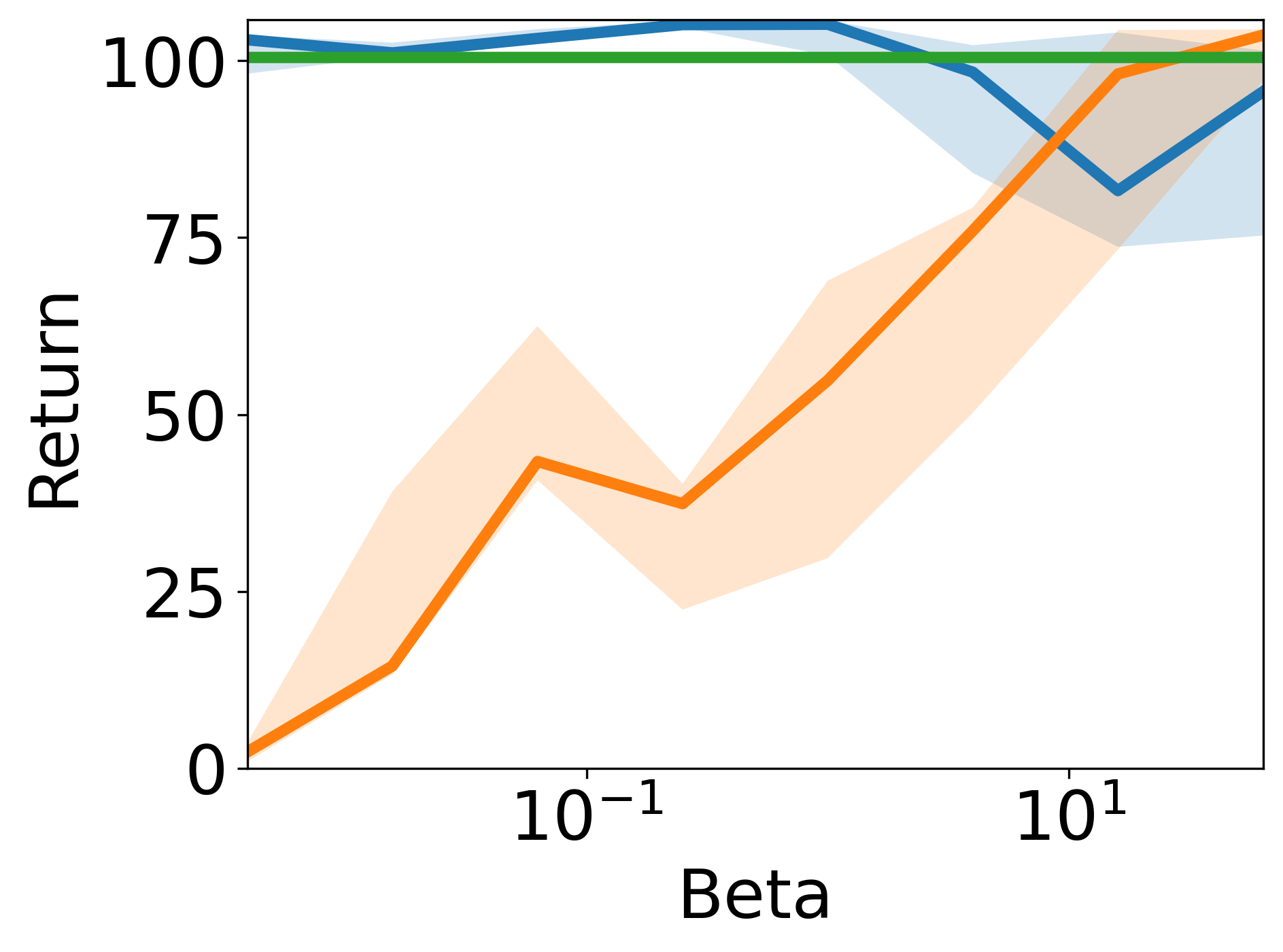}
		\caption{door-exp}
		\label{fig:door-medium-replay robust PI}
	\end{subfigure}
	\\
	\begin{subfigure}{0.24\textwidth}
		\includegraphics[width=\textwidth]{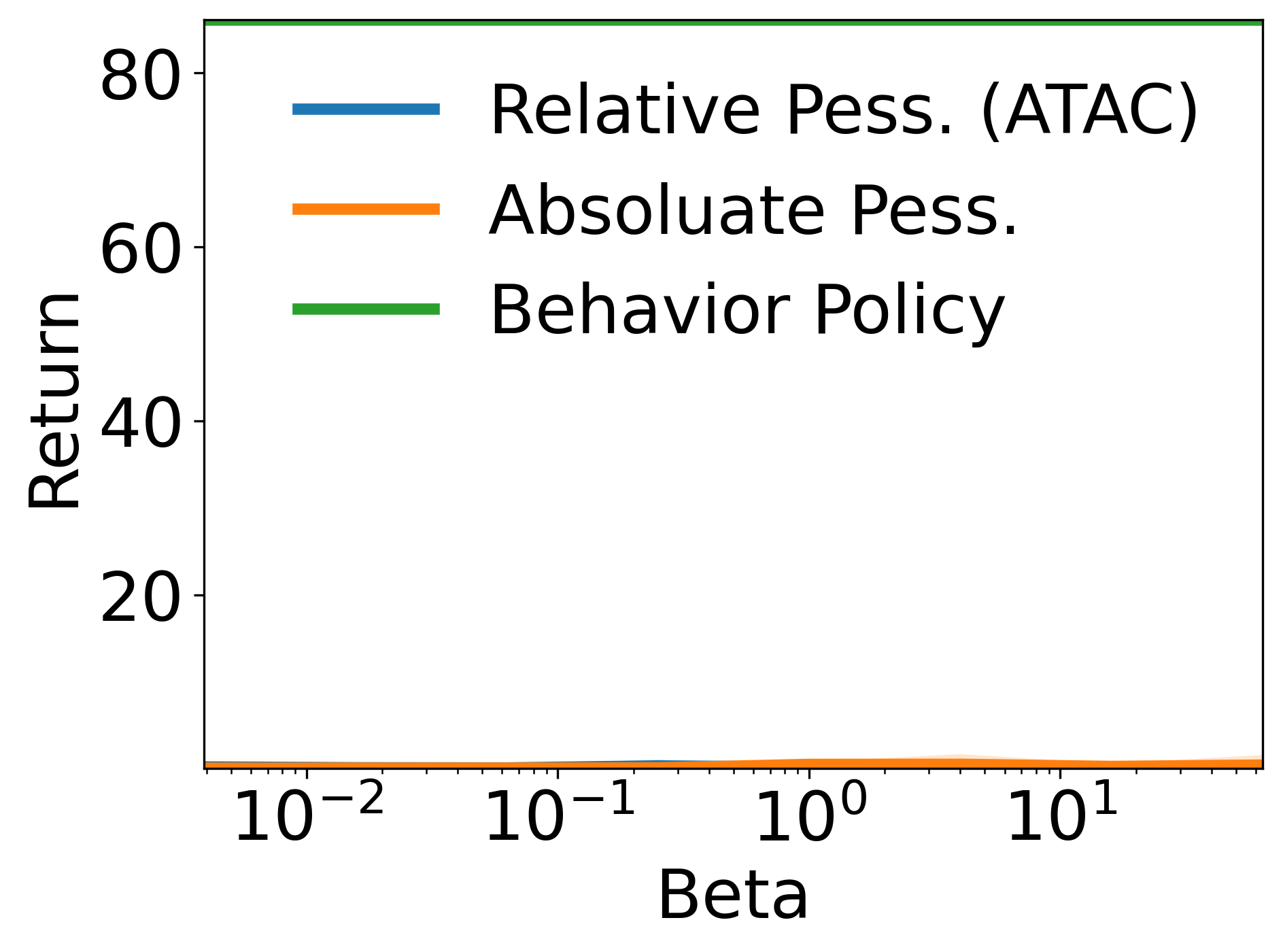}
		\caption{relocate-human}
		\label{fig:relocate-random robust PI}
	\end{subfigure}
	\begin{subfigure}{0.24\textwidth}
		\includegraphics[width=\textwidth]{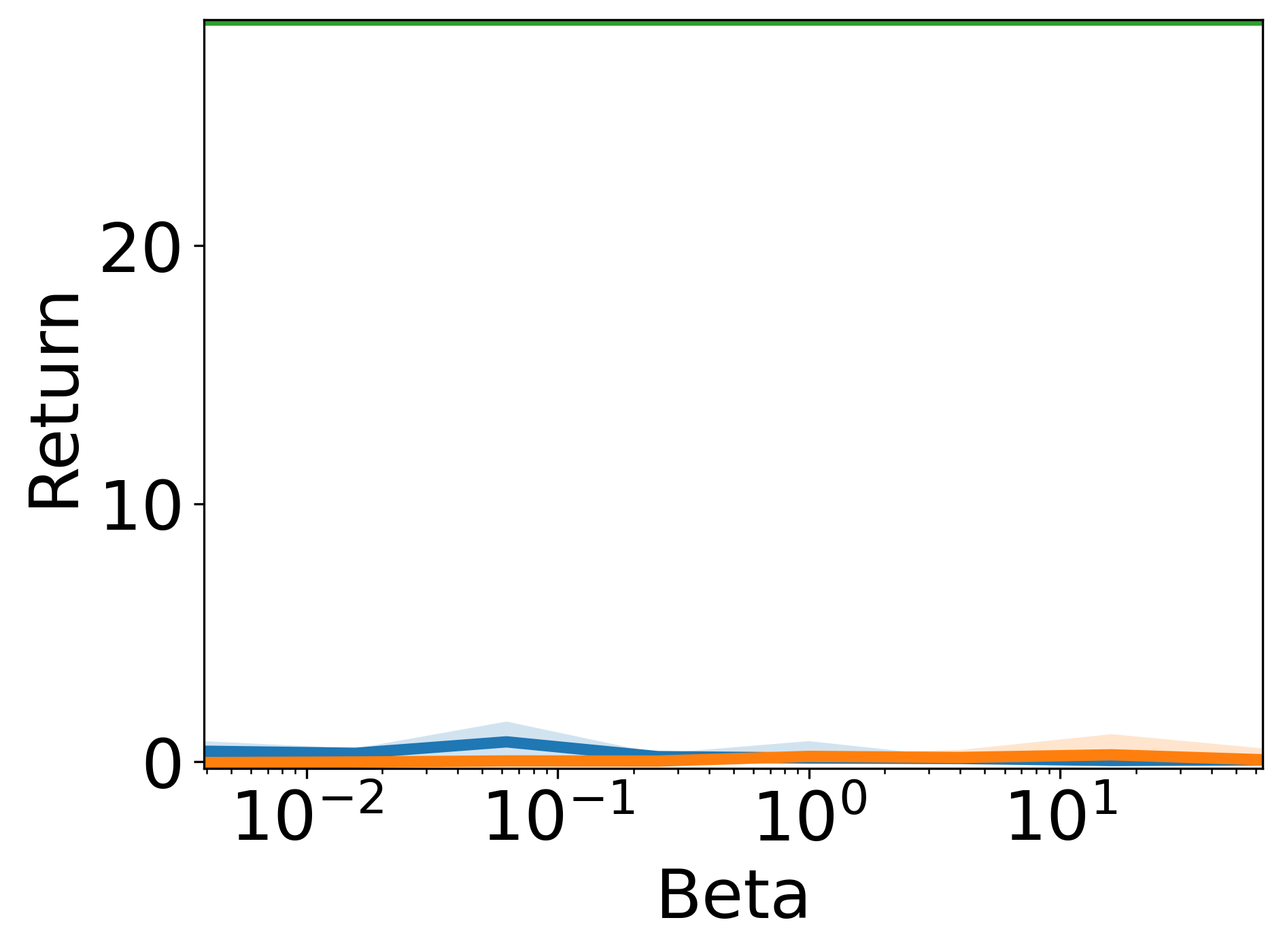}
		\caption{relocate-cloned}
		\label{fig:relocate-medium robust PI}
	\end{subfigure}
	\begin{subfigure}{0.24\textwidth}
		\includegraphics[width=\textwidth]{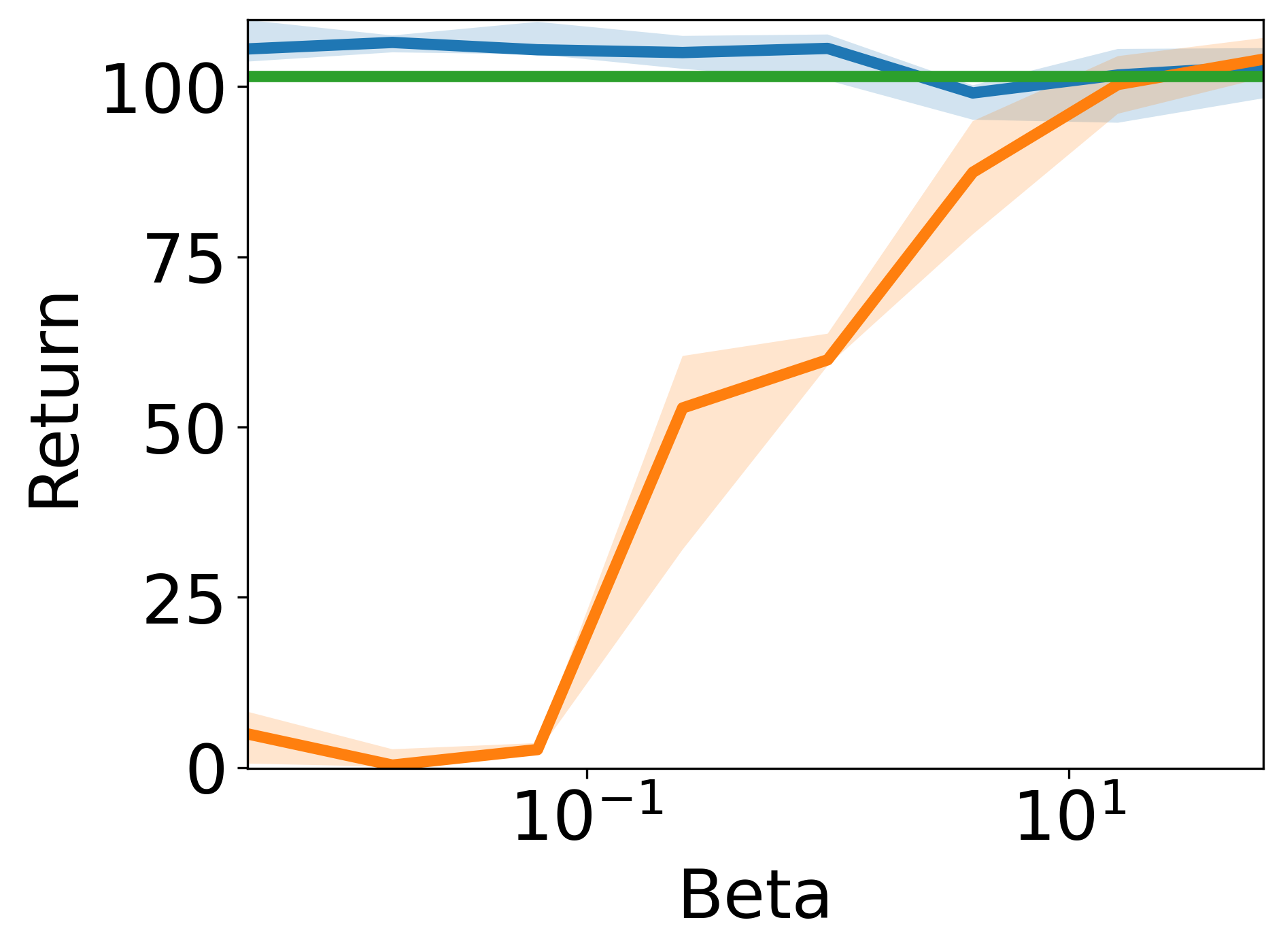}
		\caption{relocate-exp}
		\label{fig:relocate-medium-replay robust PI}
	\end{subfigure}

\caption{\small{Robust Policy Improvement of \algo in the Adroit domains. \algo based on \emph{relative} pessimism improves from behavior policies over a wide range of hyperparameters that controls the degree of pessimism for the *-exp datasets. On the contrary, \emph{absolute} pessimism does not have this property and needs well-tuned hyperparameters to ensure safe policy improvement.
For *-human and *-cloned datasets, robust policy improvement is not observed empirically, likely because human demonstrators cannot be modeled by Markovian Gaussian policies (i.e. $\mu \notin \Pi$).
The plots show the $25^{th}$, $50^{th}$, $75^{th}$ percentiles over 10 random seeds.
}}
\label{fig:robust PI (adroit)}
\end{figure*}

\clearpage
\input{cql_comparison}

\end{document}

%% file: cql_comparison.tex




\section{Comparison between ATAC and CQL} \label{sec:comparison with CQL}

We compare ATAC with CQL~\citep{kumar2020conservative} in details, since they share a similar pessimistic policy evaluation procedure. In a high level, there are several major differences at the conceptual level:

\begin{enumerate}
\item \textbf{(Conceptual Algorithm)} ATAC describes an explicit solution concept, whereas CQL does not have a clear objective but is described as an iterative procedure. Since the convergence property and fixed point of CQL is unclear for general setups, we cannot always compare ATAC and CQL.
\item \textbf{(Maximin vs Minimax)} ATAC decouples the policy and the critic, whereas CQL aims to derive the policy from a critic. Specifically, ATAC uses a maximin formulation that finds policies performing well even for the worst case critic,  whereas CQL uses a minimax formulation that finds the optimal policy for the worst case critic. In general, maximin and minimax leadto different policies.

\item \textbf{(Robust Policy Improvement)} Because of the difference between maximin and minimax in the second point, ATAC recovers behavior cloning when the Bellman term is turned off but CQL doesn’t. This property is crucial to establishing the robust policy improvement property of ATAC.

\end{enumerate}

ATAC and CQL also differ noticeably in the implementation design. ATAC uses the novel DQRA loss, projections, and two-timescale update; on the other hand, CQL adds an inner policy maximization, uses standard double-Q bootstrapping, and more similar step sizes for the critic and the actor.

Given such differences in both abstract theoretical reasoning and practical implementations, ATAC and CQL are two fundamentally different approaches to general offline RL, though it is likely there are special cases where the two produce the same policy (e.g. bandit problems with linear policies and critics). 

Below we discuss the core differences between the two algorithms in more details.

\subsection{Conceptual Algorithm}
First we compare the two algorithms at the conceptual level, ignoring the finite-sample error. 
\textbf{ATAC has a clear objective and an accompanying iterative algorithm to find approximate solutions, whereas CQL is described directly as an iterative algorithm whose fixed point property is not established in general.}

Specifically, recall that ATAC aims to find the solution to the Stackelberg
\begin{align}
  \widehat{\pi}^\star  & \in\argmax_{\pi\in\Pi}   \E_{\mu}[ f(s,\pi) - f(s,a)] \nonumber \\
  \label{eq:atac_pe}
 \textstyle    \textrm{s.t.} \quad f^\pi  & \in\argmin_{f\in\Fcal}  \E_{\mu}[ f(s,\pi) - f(s,a)]  + \beta \E_{\mu}[ ((f - \Tcal^\pi f) (s,a))^2 ]  
\end{align}
and we show that an approximate solution to the above can be found by a no-regret reduction in \cref{alg:atac (theory)}.

On the other hand, CQL (specifically CQL ($\Rcal$) in Eq.(3) of \citep{kumar2020conservative}) performs the update below\footnote{Assume the data is collected by the behavior policy $\mu$.}
\begin{align} \label{eq:CQL rule}
    f_{k+1} \gets \argmin_{f\in\Fcal} {\color{red}{\max_{\pi\in\Pi}} }\alpha \E_{\mu} [ f(s,\pi) - f(s,a)] - \Rcal(\pi) + \E_{\mu}[ ((f - \Tcal^{\color{red}{\pi_k}} f_k) (s,a))^2 ]
\end{align}
\citet{kumar2020conservative} propose this iterative procedure as an approximation of a pessimistic policy iteration scheme, which alternates between pessimistic policy evaluation and policy improvement with respect to the pessimistic critic:
\begin{quote}
\emph{We could alternate between performing full off-policy evaluation for each policy iterate, $\pi^k$, and one
step of policy improvement. However, this can be computationally expensive. 
Alternatively, since the policy $\pi^k$ is typically derived from the Q-function, we could instead choose $\mu(a|s)$ to approximate the policy that would maximize the current Q-function iterate, thus giving rise to an online algorithm.} \citep{kumar2020conservative}.
\end{quote}
Note $\mu$ in the quote above corresponds to $\pi$ in the inner maximization in \eqref{eq:CQL rule}.
When presenting this conceptual update rule, \citet{kumar2020conservative} however do not specify exactly how $\pi_k$ is updated but only provide properties on the policy $\exp(f_k(s,a)/Z(s))$. Thus, below we will suppose CQL aims to find policies of quality similar to  $\exp(f_k(s,a)/Z(s))$.

\subsection{Maximin vs. Minimax}

Although it is unclear what the fixed point of CQL is in  general, we still can see ATAC and CQL aim to find very different policies.
\textbf{ATAC decouples the policy and the critic to find a robust policy, whereas CQL aims to derive the policy from a critic function.}. This observation is reflected below.
\begin{enumerate}
    \item ATAC is based on a maximin formulation, whereas CQL is based on minimax formulation.
    \item ATAC updates policies by a no regret routine, where each policy is slow updated and determined by all the critics generated in the past iterations, whereas CQL is more akin to a policy iteration algorithm, where each policy is derived by a single critic.
\end{enumerate}
    
We can see this difference concretely, if we specialize the two algorithms to bandit problems. In this case, CQL is no longer iterative and has a clear objective. 
Specifically, if we let $\alpha = \frac{1}{\beta}$, the two special cases can be written as 
\begin{align*}
    &\widehat{\pi}^\star   \in\argmax_{\pi\in\Pi} \min_{f\in\Fcal}   \E_{\mu}[ f(s,\pi) - f(s,a)]
    + \beta \E_{\mu}[ ((f - r) (s,a))^2 ] & \text{(ATAC)} \\
  & \widehat{f}^\star \gets \argmin_{f\in\Fcal} \max_{\pi\in\Pi} \E_{\mu} [ f(s,\pi) - f(s,a)] + \beta \E_{\mu}[ ((f - r) (s,a))^2 ]  - \beta \Rcal(\pi)  &\text{(CQL)} 
\end{align*}
If we further ignore the extra regularization term  $\Rcal(\pi)$ (as that can often be absorbed into the policy class), then the main difference between the two approaches, in terms of solution concepts, is clearly the order of max and min. It is well known maximin and minimax gives different solutions in general, unless when the objective is convex-concave (with respect to the policy and critic parameterizations). For example, in this bandit special case, suppose the states and actions are tabular; the objective is convex-concave when $\Pi$ and $\Fcal$ contains \emph{all} tabular functions, but convex-concave objective is lost when $\Fcal$ contains a finite set of functions.
In the latter scenario, CQL and ATAC would give very different policies, and CQL would not enjoy the nice properties of ATAC.

\subsection{Robust Policy Improvement}

We now illustrate concretely how the difference between ATAC and CQL  affects the robust policy improvement property. For simplicity, we only discuss in population level.  

By \cref{prop:RPI} and \cref{prop:safe_pi}, we know $\widehat\pi^\star$, the learned policy from ATAC, provably improves behavior policy $\mu$ under a wide range of $\beta$ choice of \Eqref{eq:atac_pe}, including $\beta = 0$. In other word, as long as $\mu\in\Pi$, ATAC has $J(\widehat\pi^\star) \geq J(\mu)$ even if $\beta = 0$ in \Eqref{eq:atac_pe}.

However, the following argument shows that: In CQL, if $\pi_f \in \Pi, \forall f \in \Fcal$ and $ \Fcal$ contains constant functions, then setting $\beta = 0$ cannot guarantee policy improvement over $\mu$, even when $\mu \in \Pi$, where $\pi_f$ denotes the greedy policy with respect to $f$.

Based on what's shown before, the corresponding CQL update rule with $\beta = 0$ can be written as
\begin{align*}
    f_{k + 1} \gets \argmin_{f\in\Fcal} \max_{\pi\in\Pi} \E_{\mu} [ f(s,\pi) - f(s,a)].
\end{align*}
We now prove that $f_{k+1}$ is constant across actions in every state on the support of $\mu$ for any $k$:
\begin{enumerate}
\item $\min_{f\in\Fcal} \max_{\pi\in\Pi} \E_{\mu} [ f(s,\pi) - f(s,a)] = 0$, by
\begin{align*}
&~ 0 = \left. \min_{f\in\Fcal} \E_{\mu} [ f(s,\pi) - f(s,a)] \right|_{\pi = \mu} \leq \min_{f\in\Fcal} \max_{\pi\in\Pi} \E_{\mu} [ f(s,\pi) - f(s,a)] \leq \left. \max_{\pi\in\Pi} \E_{\mu} [ f(s,\pi) - f(s,a)] \right|_{f \equiv 0} = 0.
\end{align*}
\item For any $f' \in \Fcal$, if there exists $(s_1,a_1) \in \Scal \times \Acal$ such that $\mu(s_1)>0$ and $f'(s_1,a_1) > \max_{a \in \Acal \setminus {a_1}} f'(s_1,a)$, then $\max_{\pi \in \Pi} \E_\mu[ f'(s,\pi) - f'(s,a)] \geq \E_\mu[ f'(s,\pi_{f'}) - f'(s,a)] \geq \mu(s_1) ( f'(s_1,a_1) - f'(s_1,\mu) ) \geq \mu(s_1) (1 - \mu(a_1|s_1)) (f'(s_1,a_1) - \max_{a \in \Acal \setminus {a_1}} f'(s_1,a)) > 0$. 
\item Combining the two bullets above, we obtain that $f_{k+1}$ for all $k$ must have $f_{k+1}(s_1,a_1) = f_{k+1}(s_1,a_2)$ for all $(s_1,a_1,a_2) \in \Scal \times \Acal \times \Acal$ such that $\mu(s_1)>0$, i.e., $f_{k+1}$ is constant across actions in every $s \in \Scal$ in the support of $\mu$.
\end{enumerate}
Therefore, for CQL with $\beta=0$, the policies are updated with per-state constant functions, leading to  arbitrary learned policies and failing to provide the safe policy improvement guarantee over the behavior policy $\mu$.

